\documentclass[journal,romanappendices, 10pt]{IEEEtran} 

\usepackage{amsmath}
\usepackage{amssymb}
\usepackage{cases}
\usepackage{graphicx}
\usepackage[sort,compress]{cite}
\usepackage{booktabs}
\usepackage{kbordermatrix}

\usepackage{amsthm}

\newtheorem{assumption}{Assumption}
\newtheorem{theorem}{Theorem}
\newtheorem{problem}{Problem}
\newtheorem{lemma}{Lemma}

\newtheorem{corollary}{Corollary}

\theoremstyle{definition}
\newtheorem{remark}{Remark}

\usepackage[]{xcolor}
\usepackage{subcaption}
\usepackage{wrapfig}
\usepackage{multirow}
\usepackage{mathtools}
\usepackage{empheq}
\usepackage{color}
\usepackage{bbm}
\usepackage{xspace}

\usepackage[shortlabels]{enumitem}
\usepackage{tabularray}

\usepackage{tikz} 
\usepackage{tkz-graph}
\usepackage{tkz-berge}
\usetikzlibrary{backgrounds,fit,shapes,snakes,arrows,shapes.geometric,positioning}
\usetikzlibrary{intersections,patterns,shapes.misc}
\usetikzlibrary{decorations.pathmorphing}
\usepackage{pgfplots}

\tikzstyle{block} = [rectangle, rounded corners, minimum width=3cm, minimum height=1cm,text centered, draw=black, fill=red!30]
\tikzstyle{new} = [rectangle, rounded corners, minimum width=1cm, minimum
height=1cm,text centered, draw=black, fill=blue!10!white, dashed]
\tikzstyle{arrow} = [thick,->,>=stealth]
\usetikzlibrary{calc, quotes}

\tikzstyle{fblock} = [rectangle, draw, fill=gray!20, 
text width=8em, text centered, rounded corners, minimum height=4em, minimum width = 8em]
\tikzstyle{line} = [draw, -latex']

\usepackage{algorithmicx}
\usepackage{algorithm}
\usepackage[]{algpseudocode}

\usepackage{fontawesome}

\usepackage{soul}
\usepackage{comment}

\usepackage[]{url}
\usepackage[pdftex, colorlinks=true, linkcolor=black, citecolor = black, urlcolor = blue, filecolor=black , pagebackref=false, hypertexnames=false]{hyperref}
\usepackage{cleveref}
\crefname{problem}{Problem}{Problems}
\crefname{example}{Example}{Examples}
\crefname{section}{Sec.}{Secs.}
\Crefname{section}{Section}{Sections}
\Crefname{table}{Table}{Tables}
\crefname{table}{Table}{Tabs.}
\crefname{figure}{Fig.}{Figs.}
\crefname{algorithm}{Algorithm}{Algorithms}
\crefname{remark}{Remark}{Remarks}
\crefname{theorem}{Theorem}{Theorems}
\crefname{lemma}{Lemma}{Lemmas}
\crefname{corollary}{Corollary}{Corollaries}
\crefname{assumption}{Assumption}{Assumptions}

\usepackage[font=small]{caption}
\newcommand{\Real}{\mathbb{R}}
\newcommand{\Natural}{\mathbb{N}}

\DeclareMathOperator*{\minimize}{minimize}

\providecommand{\inner}[2]{\left \langle #1, #2 \right \rangle}
\providecommand{\norm}[1]{\left\|#1\right\|}

\DeclareMathOperator{\Sym}{\mathcal{S}}
\def \PSD{\mathcal{S}_{+}}

\DeclareMathOperator{\tr}{tr}

\DeclareMathOperator{\Diag}{Diag}

\def \inv{^{-1}}

\def \pinv {^\dagger}

\DeclareMathOperator{\image}{image}

\DeclareMathOperator{\Diff}{D}
\DeclareMathOperator{\rgrad}{grad}
\DeclareMathOperator{\Hess}{Hess}

\DeclareMathOperator{\Retr}{Retr}

\DeclareMathOperator{\dist}{\mathbf{d}}
\DeclareMathOperator{\distang}{\mathbf{d}_\angle}
\DeclareMathOperator{\distchr}{\mathbf{d}_\text{chr}}
\DeclareMathOperator{\inj}{inj}
\DeclareMathOperator{\Exp}{Exp}
\DeclareMathOperator{\Log}{Log}
\DeclareMathOperator{\Lift}{lift}

\DeclareMathOperator{\SOd}{SO}
\DeclareMathOperator{\sod}{so}
\providecommand{\hatop}[1]{\left[#1\right]_\times}
\DeclareMathOperator{\SE}{SE}

\DeclareMathOperator{\Sc}{Sc}

\newcommand{\indvector}{\Delta}

\newcommand{\bmat}{\begin{bmatrix}}
\newcommand{\emat}{\end{bmatrix}}

\newcommand{\etal}{\emph{et~al.}\xspace}

\newcommand{\eg}{\emph{e.g.}\xspace}
\newcommand{\ie}{\emph{i.e.}\xspace}

\newcommand{\cf}{\emph{c.f.}\xspace}

\providecommand{\dataset}[1]{{\small \textsf{#1}}\xspace}
\providecommand{\method}[1]{{\small \textsf{#1}}\xspace}

\newcommand{\Vcal}{\mathcal{V}}
\newcommand{\Ecal}{\mathcal{E}}
\newcommand{\Fcal}{\mathcal{F}}
\newcommand{\Ccal}{\mathcal{C}}
\newcommand{\Ucal}{\mathcal{U}}
\newcommand{\Mcal}{\mathcal{M}}
\newcommand{\Hcal}{\mathcal{H}}

\newcommand{\Ncal}{\mathcal{N}}

\newcommand{\Xcal}{\mathcal{X}}

\newcommand{\fhat}{\widehat{f}}
\newcommand{\hhat}{\widehat{h}}
\newcommand{\xhat}{\widehat{x}}
\newcommand{\xstarhat}{\xhat^\star}
\newcommand{\ghat}{\widehat{g}}
\newcommand{\Hhat}{\widehat{H}}
\newcommand{\mhat}{\widehat{m}}

\newcommand{\That}{\widehat{T}}
\newcommand{\Rhat}{\widehat{R}}
\newcommand{\that}{\widehat{t}}
\newcommand{\Shat}{\widehat{S}}

\newcommand{\Mhat}{\widehat{M}}

\newcommand{\etahat}{\widehat{\eta}}

\newcommand{\xtilde}{\widetilde{x}}

\newcommand{\Ttilde}{\widetilde{T}}
\newcommand{\Rtilde}{\widetilde{R}}
\newcommand{\Stilde}{\widetilde{S}}
\newcommand{\ttilde}{\widetilde{t}}
\newcommand{\Htilde}{\widetilde{H}}
\newcommand{\Ltilde}{\widetilde{L}}
\newcommand{\Xtilde}{\widetilde{X}}
\newcommand{\Gtilde}{\widetilde{G}}

\newcommand{\Runder}{\underline{R}}

\newcommand{\Rstar}{R^\star}

\newcommand{\xstar}{x^\star}
\newcommand{\Xstar}{X^\star}

\newcommand{\ebar}{\overline{e}}
\newcommand{\gbar}{\overline{g}}
\newcommand{\fbar}{\overline{f}}
\newcommand{\Hbar}{\overline{H}}
\newcommand{\Mbar}{\overline{M}}
\newcommand{\thetabar}{\bar{\theta}}
\newcommand{\Mcalbar}{\overline{\mathcal{M}}}
\newcommand{\Retrbar}{\overline{\Retr}}

\newcommand{\RMSE}{\text{RMSE}}

\newcommand{\Rref}{R^\text{ref}}

\newcommand{\TLS}{\rho^\text{\tiny TLS}}
\newcommand{\wgnc}{w^\text{\tiny GNC}}
\newcommand{\SigmaPGO}{\Sigma^\text{\tiny PGO}}
\newcommand{\SigmaInit}{\Sigma^\text{\tiny INIT}}
\newcommand{\SparsifiedSchurComplement}{\textsc{SparsifiedSchurComplement}\xspace}
\newcommand{\SparsifiedLaplacianSolver}{\textsc{SparsifiedLaplacianSolver}\xspace}

\newcommand{\kB}{kB}
\newcommand{\Gendarmenmarkt}{\dataset{Gendarmenmarkt}}

\newcommand{\myParagraph}[1]{{\bf #1.}\xspace}

\providecommand{\edit}[1]{{#1}}

\providecommand{\techreport}[1]{{#1}}  

\providecommand{\mainpaper}[1]{{}}  

\title{ 
Spectral Sparsification for Communication-Efficient 
Collaborative Rotation and Translation Estimation}

\author{Yulun Tian and Jonathan P. How
	\thanks{The authors are with the Department of Aeronautics and Astronautics, Massachusetts Institute of Technology, 77 Massachusetts Ave, Cambridge, MA 02139, USA. {\tt\small \{yulun, jhow\}@mit.edu}.
	{This work was supported in part by ARL DCIST under Cooperative
	Agreement Number W911NF-17-2-0181, and in part by ONR under BRC Award N000141712072.}

    {The authors gratefully acknowledge Dr. Kaveh Fathian, Parker Lusk, Dr. Kasra Khosoussi, Dr. David
    M. Rosen, and anonymous reviewers for helpful comments. The authors would also like to
    thank Prof. Pierre-Antoine Absil for insightful discussions on the linear convergence of approximate Newton
    methods on Riemannian manifolds.}
}%
}

\begin{document}

\maketitle

\begin{abstract}
	We propose fast and communication-efficient optimization algorithms for multi-robot
	rotation averaging and translation estimation problems
	that arise from collaborative simultaneous localization and mapping (SLAM), 
	{structure-from-motion (SfM)},
	and camera network localization applications.
	Our methods are based on theoretical relations between the \emph{Hessians} of the underlying Riemannian optimization problems 
	and the \emph{Laplacians} of suitably weighted graphs.
	{We leverage these results to design a collaborative solver in which robots coordinate with a central server to perform approximate second-order optimization, by solving a \emph{Laplacian system} at each iteration.}
	Crucially, our algorithms permit robots to employ \emph{spectral sparsification} to sparsify intermediate dense matrices before communication, 
	and hence provide a mechanism to trade off accuracy with communication efficiency with provable guarantees.
	We perform rigorous theoretical analysis of our methods and prove that they enjoy (local) \emph{linear} rate of convergence.
	{Furthermore, we show that our methods can be combined with {graduated non-convexity} to achieve \emph{outlier-robust} estimation.}
	{Extensive experiments on real-world SLAM and SfM scenarios demonstrate the superior convergence rate and communication efficiency of our methods.}
\end{abstract}


\section{Introduction}

{Collaborative spatial perception is a fundamental capability for multi-robot systems to operate in unknown, GPS-denied environments.
State-of-the-art systems (\eg, \cite{Schmuck21covins,Chang2022lamp2, Cramariuc2022maplab,Cieslewski18DataEfficient, Lajoie20DOOR, Tian21KimeraMulti}) rely on optimization-based back-ends to achieve accurate multi-robot simultaneous localization and mapping (SLAM).
Often, a central server receives data from all robots (\eg in the form of factor graphs \cite{Dellaert17FactorGraph}) and solves the underlying large-scale optimization for the entire team.
In comparison, \emph{collaborative optimization} frameworks leverage robots' local computation and iterative communication (either peer-to-peer or coordinated by a server),
and thus have the potential to scale to larger scenes and support more robots.}

{Recent works focus on developing} fully distributed algorithms in which robots carry out iterative optimization via peer-to-peer message passing \cite{Tron2014CameraNetwork,Choudhary17IJRR,Tian21DC2PGO,Tian20ASAPP,Fan2021MajorizationJournal,Murai2022RobotWeb}.
While these methods are flexible in terms of the required communication architecture, 
they often suffer from slow convergence due to their first-order nature and the inherent poor conditioning of typical SLAM problems.
To resolve the slow convergence issue, 
an alternative is to pursue a \emph{second-order} optimization framework.
{A prominent example is DDF-SAM~\cite{Cunningham10DDFSAM,Cunningham12DataAssociation,Cunningham13DDFSAM2}, in which robots marginalize out internal variables (\ie, those without inter-robot measurements) in their local factor graphs before communication.
From an optimization perspective, robots partially eliminate their local Hessians and communicate the resulting matrices.}
However, a shortcoming of this approach is that the transmitted matrices are usually \emph{dense} (even if the original problem is sparse), 
{and hence could result in long transmission times that prevent the team from obtaining a timely solution.}

{To address the aforementioned technical gaps, this work presents results towards collaborative optimization that achieves both \emph{fast convergence} and \emph{efficient communication}.}
{Specifically, we develop new algorithms for solving multi-robot rotation averaging and translation estimation.
	These problems are fundamental and have applications ranging from initialization for pose graph SLAM \cite{Carlone2015Initialization}, 
	structure-from-motion (SfM) \cite{Zhu18DistMotionAverage}, 
	and camera network localization \cite{Tron2014CameraNetwork}.}
{Our approach is based on a server-client architecture (\cref{fig:example:communication_architecture}), in which multiple robots (clients) coordinate with a server to collaboratively solve the optimization problem leveraging local computation.}
{The crux of our method lies in exploiting theoretical relations between the Hessians of the optimization problems and the Laplacians of the underlying graphs.} 
We leverage these theoretical insights to develop a fast collaborative optimization method
in which each iteration computes an \emph{approximate second-order} update 
\edit{by replacing the Hessian with a \emph{constant} Laplacian matrix, 
which improves efficiency in both computation and communication. 
Furthermore, during communication,} robots use \emph{spectral sparsification} \cite{Batson2013SpectralSurvey,Spielman2011Sampling} to sparsify {intermediate dense matrices resulted from elimination of its internal variables}. {\cref{fig:example:original_graph,fig:example:reduced_graph_dense,fig:example:reduced_graph_sparse} show a high-level illustration of our approach.}
By varying the degree of sparsification, our method thus provides a principled way for trading off accuracy with communication efficiency.
The theoretical properties of spectral sparsification allow us to perform rigorous convergence analysis, 
and establish \emph{linear} rates of convergence for our methods.
{Lastly, we also present an extension to \emph{outlier-robust} estimation by combining our approach with graduated non-convexity (GNC) \cite{Yang2020Gnc,Black1996Gnc}.}

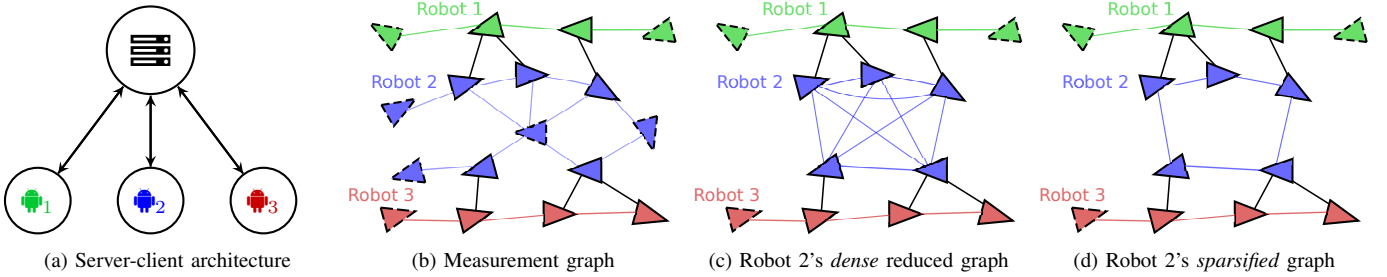
\begin{figure*}[t]
	\centering
	\definecolor{robotgreen}{RGB}{0,200,50}
	\definecolor{robotblue}{RGB}{0,0,255}
	\definecolor{robotred}{RGB}{200,0,0}
	\begin{subfigure}[t]{0.24\linewidth}
		\begin{tikzpicture}[scale=0.5]
			\tikzstyle{robot}=[circle,scale=1,draw,thick]
			\tikzstyle{server}=[circle,scale=1.5,draw,thick]
			\tikzstyle{upload}=[->,>=stealth,thick]
			\tikzstyle{download}=[<-,>=stealth,thick]
			
			\node[server] (server) at (3.0,5.5) [] {\faServer};
			\node[robot] (r1) at (0.0,1.5) [] {\textcolor{robotgreen}{\faAndroid$_{1}$}};
			\node[robot] (r2) at (3.0,1.5) [] {\textcolor{robotblue}{\faAndroid$_{2}$}};
			\node[robot] (r3) at (6.0,1.5) [] {\textcolor{robotred}{\faAndroid$_{3}$}};
			
			\draw[upload] (r1) -- (server);
			\draw[download] (r1) -- (server);
			\draw[upload] (r2) -- (server);
			\draw[download] (r2) -- (server);
			\draw[upload] (r3) -- (server);
			\draw[download] (r3) -- (server);
		\end{tikzpicture}
		\captionsetup{justification=centering}
		\caption{Server-client architecture}
		\label{fig:example:communication_architecture}
	\end{subfigure}
	\hfill
	\begin{subfigure}[t]{0.24\linewidth}
		\includegraphics[width=\textwidth]{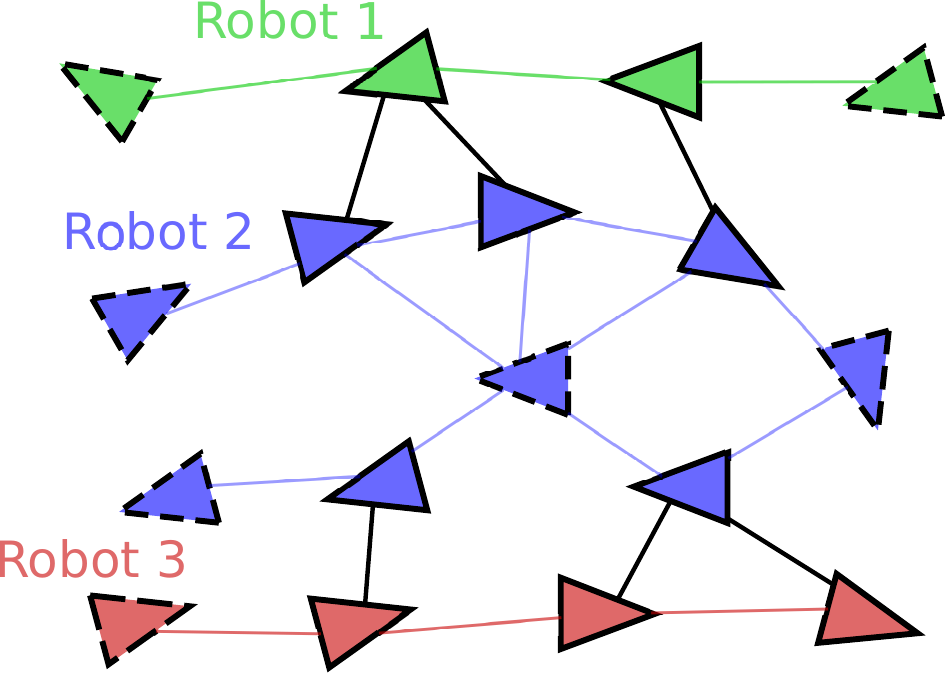}
		\caption{Measurement graph}
		\label{fig:example:original_graph}
	\end{subfigure}
	\hfill
	\begin{subfigure}[t]{0.24\linewidth}
		\includegraphics[width=\textwidth]{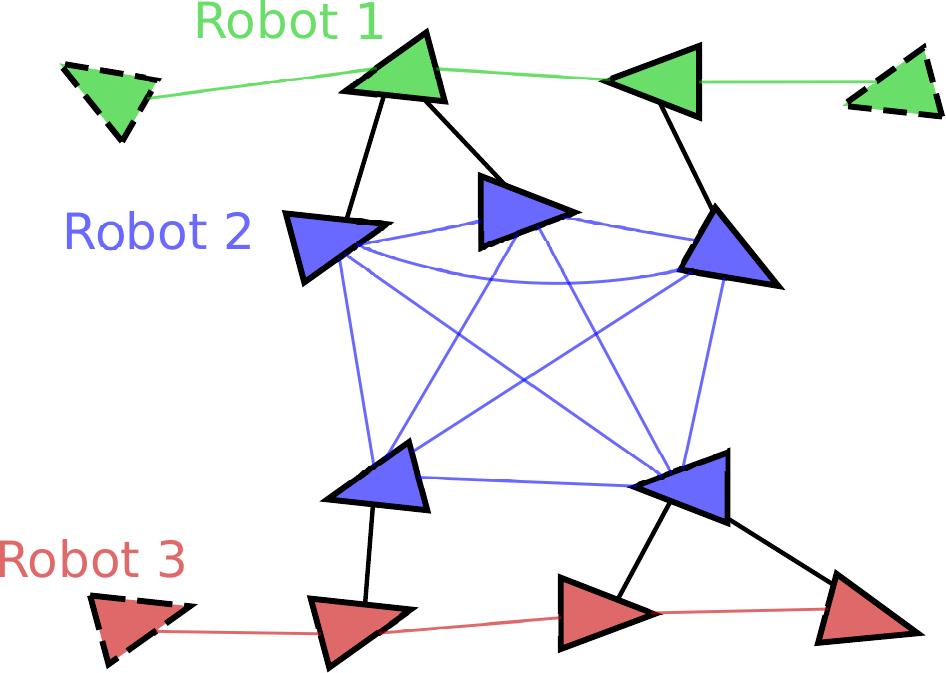}
		\caption{Robot 2's \emph{dense} reduced graph}
		\label{fig:example:reduced_graph_dense}
	\end{subfigure}
	\hfill
	\begin{subfigure}[t]{0.24\linewidth}
		\includegraphics[width=\textwidth]{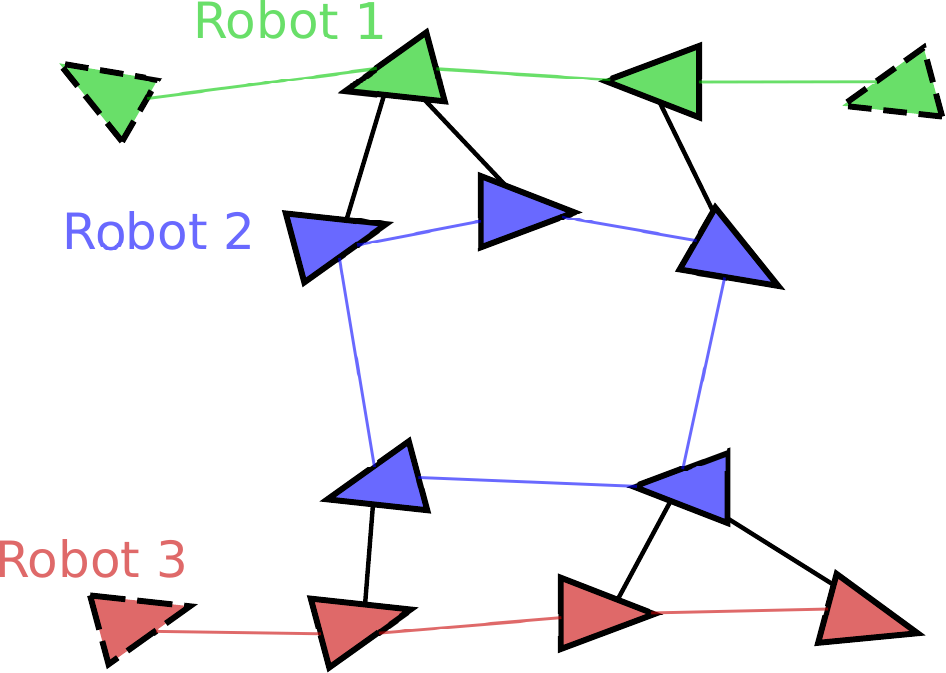}
		\caption{Robot 2's \emph{sparsified} graph}
		\label{fig:example:reduced_graph_sparse}
	\end{subfigure}
	\caption{
		{(a) Information flow in server-client architecture. Each communication round consists of an \emph{upload} stage (clients to server) and a \emph{download} stage (server to clients). }
		(b) Example 3-robot problem visualized as a graph. 
		For each robot $\alpha \in \{1,2,3\}$, its vertices (variables) $\Vcal_\alpha$ are shown in a distinct color.
		Each edge corresponds to a relative measurement between two variables.
		{Separators (solid line) correspond to variables with inter-robot measurements, and the remaining variables form the interior vertices (dashed line).}
		(c) For robot 2, elimination of its interior vertices creates a \emph{dense} matrix $S_2$, which corresponds to a {dense} graph over its separators.
		(d) In our approach, robot 2 achieves communication efficiency by transmitting {a sparse approximation $\Stilde_2$ of the original dense matrix $S_2$, which also corresponds to a sparsified graph over its separators. }
	} 
	\label{fig:example}
        \vspace{-0.3cm}
\end{figure*}

{\myParagraph{Contributions}}
{The key contributions of this work are summarized as follows:
\begin{itemize}[leftmargin=0.3cm]
	\item We present collaborative optimization algorithms for multi-robot rotation averaging and translation estimation under the server-client architecture, which enjoy {fast convergence} (in terms of the number of iterations) and {efficient communication} through the use of \emph{spectral sparsification}.
	\item In contrast to the typical sublinear convergence of prior methods, we prove (local) \emph{linear convergence} for our methods and show that the rate of convergence depends on the user-defined  sparsification parameter.
	\item We present an extension to \emph{outlier-robust} estimation by combining the proposed algorithms with GNC. 
	\item We perform extensive evaluations of our methods and demonstrate their values on real-world SLAM and SfM scenarios with  outlier measurements. 
\end{itemize}
Lastly, while our algorithms and theoretical guarantees cover separate rotation averaging and translation estimation, 
we demonstrate through our experiments that their combination can be used to achieve robust initialization for pose graph optimization (PGO), which is another fundamental problem commonly used in collaborative SLAM.

{\myParagraph{Paper Organization}}
The rest of this paper is organized as follows.
The remainder of this section introduces necessary notation and mathematical preliminaries, and
in \cref{sec:related_work}, we review related works.
{\cref{sec:problem_definition} formally introduces the problem formulation, communication architecture, and relevant applications.}
In \cref{sec:laplacian_systems}, we establish theoretical relations between the Hessians and the underlying graph Laplacians.
Then, in \cref{sec:algorithm}, we leverage these theoretical results to design fast and communication-efficient solvers for the problems of interest and establish convergence guarantees.
Finally, \cref{sec:experiments} presents numerical evaluations of the proposed algorithms. 
\mainpaper{{Proofs and additional details, discussions, and experiments are provided in our extended technical report \cite{Tian2022Sparsification}.}}

\subsection*{Notations and Preliminaries}
\label{sec:preliminary}

\techreport{{\cref{tab:notations} in the appendix summarizes detailed notations used in this work.}}
\mainpaper{{Table~V in our technical report \cite{Tian2022Sparsification} summarizes detailed notations used in this work.}}
Unless stated otherwise, lowercase and uppercase letters denote vectors and matrices, respectively.
We define $[n] \triangleq \{1, 2, \hdots, n\}$ as the set of positive integers from 1 to $n$.

\myParagraph{Linear Algebra and Spectral Approximation}
$\Sym^n$ and $\PSD^n$ denote the set of $n \times n$ symmetric and symmetric positive semidefinite matrices, respectively. 
{We use $\otimes$ to denote the Kronecker product. 
For a positive integer $n$, $1_n \in \Real^n$ and $I_n \in \Real^{n\times n}$ denote the vector of all ones and the Identity matrix.}
{For any matrix $A$, $\ker(A)$ and $\image(A)$ denote the kernel (nullspace) and image (span of column vectors) of $A$, respectively.}
$A\pinv$ denotes the Moore-Penrose inverse of $A$, which coincides with the inverse $A\inv$ when $A$ is invertible.
When $A \in \Sym^n$, $\lambda_1(A), \hdots, \lambda_n(A)$ denote the real eigenvalues of $A$ sorted in increasing order.
When $A \in \PSD^n$, we also define $\norm{X}_A \triangleq \sqrt{\tr(X^\top A X)}$ where $X$ is of compatible dimensions.

Following \cite{lee2015sparsified,kyng2016sparsified},
for $A, B \in \Sym^n$ and $\epsilon > 0$, we say that $B$ is an \emph{$\epsilon$-approximation} of $A$, denoted as $A \approx_\epsilon B$, 
if the following holds,
\begin{equation}
	e^{-\epsilon} B \preceq A \preceq e^\epsilon B,
	\label{eq:spectral_approximation_def}
\end{equation}
{where $B \preceq A$ means $A - B \in \PSD^n$.}
Note that \eqref{eq:spectral_approximation_def} is symmetric and holds under composition:
if $A \approx_\epsilon B$ and $B \approx_\delta C$, then $A \approx_{\epsilon + \delta} C$.
Furthermore, if $A$ is singular, the relation \eqref{eq:spectral_approximation_def} 
implies that $B$ is necessarily singular and $\ker(A) = \ker(B)$.

\myParagraph{Graph Theory}
A weighted undirected graph is denoted as $G = (\Vcal, \Ecal, w)$, where $\Vcal$ and $\Ecal$ denote the vertex and edge sets, and 
$w: \Ecal \to \Real_{>0}$ {is the edge weight function that assigns each edge $(i,j) \in \Ecal$ a positive weight $w_{ij}$.}
For a graph $G$ with $n$ vertices, its \emph{graph Laplacian} $L(G;w) \in \PSD^{n}$ is defined as,
\begin{equation}
	L(G;w)_{ij} = \begin{cases}
		\sum_{k \in \text{Nbr}(i)} w_{ik}, & \text{if $i = j$}, \\
		-w_{ij}, & \text{if $i\neq j, \; (i,j) \in \Ecal$}, \\
		0, & \text{otherwise.}
	\end{cases}
	\label{eq:laplacian_def}
\end{equation}
{In \eqref{eq:laplacian_def}, $\text{Nbr}(i) \subseteq \Vcal$ denotes the neighbors of vertex $i$ in the graph.}
Our notation $L(G;w)$ serves to emphasize that the Laplacian of $G$ depends on the edge weight $w$.
When the edge weight $w$ is irrelevant or clear from context,
we will write the graph as $G = (\Vcal, \Ecal)$ and its Laplacian as $L(G)$ or simply $L$.
{The graph Laplacian $L$ always has a zero eigenvalue, \ie, $\lambda_1(L) = 0$.}
The second smallest eigenvalue $\lambda_2(L)$ is known as the \emph{algebraic connectivity}, which is always positive for connected graphs.

\myParagraph{Riemannian Manifolds}
The reader is referred to \cite{Absil2009Book,Boumal20Book} for a comprehensive review of optimization on matrix manifolds.
In general, we use $\Mcal$ to denote a smooth matrix manifold.
For integer $n > 1$, $\Mcal^n$ denotes the product manifold formed by $n$ copies of $\Mcal$.
$T_{x} \Mcal$ denotes the tangent space at $x \in \Mcal$. 
For tangent vectors $\eta, \xi \in T_x \Mcal$, their inner product is denoted as $\inner{\eta}{\xi}_x$, and the corresponding norm is
$\norm{\eta}_x = \sqrt{\inner{\eta}{\eta}_x}$.
In the rest of the paper, we drop the subscript $x$ as it will be clear from context.
At $x \in \Mcal$, the injectivity radius $\inj(x)$ is a positive constant such that the exponential map $\Exp_x: T_x \Mcal \to \Mcal$ is a diffeomorphism when restricted to the domain $U = \{\eta \in T_x \Mcal: \norm{\eta} < \inj(x) \}$.
In this case, we define the logarithm map to be $\Log_x \triangleq \Exp_x\inv$.
Unless otherwise mentioned, we use $\dist(x, y)$ to denote the geodesic distance
between two points $x, y \in \Mcal$ induced by the Riemannian metric.
In addition, it  holds that $\dist(x,y) = \norm{v}$ where $v = \Log_x(y)$; see \cite[Proposition~10.22]{Boumal20Book}.

\myParagraph{The Rotation Group $\SOd(d)$}
The rotation group is denoted as $\SOd(d) = \{R \in \Real^{d\times d}: R^\top R = I, \; \det(R)=1\}$.
The tangent space at $R$ is given by $T_R \SOd(d) = \{RV: V \in \sod(d) \}$, where $\sod(d)$ is the space of $d \times d$ skew-symmetric matrices. 
In this work, we exclusively work with 2D and 3D rotations.
We define a basis for $T_R \SOd(3)$ such that 
each tangent vector $\eta \in T_R \SOd(3)$ is identified with a vector $v \in \Real^3$,
\begin{equation}
	\eta = R \hatop{v} = 
	R \bmat
	0 & -v_3 & v_2 \\
	v_3 & 0 & -v_1 \\
	-v_2 & v_1 & 0
	\emat.
	\label{eq:SO3_basis}
\end{equation}
Note that \eqref{eq:SO3_basis} defines a bijection between $\eta \in T_R \SOd(3)$ and $v \in \Real^3$.
For $d = 2$, we can define a similar basis for the 1-dimensional tangent space $T_R \SOd(2)$, where each tangent vector $\eta \in T_R \SOd(2)$ is identified by a scalar $v \in \Real$ as,
\begin{equation}
	\eta = R \hatop{v} = 
	R \bmat
	0 & -v  \\
	v & 0  \\
	\emat.
	\label{eq:SO2_basis}
\end{equation}
We have overloaded the notation $\hatop{\cdot}$ to map the input scalar or vector to the corresponding skew-symmetric matrix in $\sod(2)$ or $\sod(3)$.
Under the basis given in \eqref{eq:SO3_basis} and \eqref{eq:SO2_basis}, 
the inner product on the tangent space is defined by the corresponding vector dot product, \ie,
$\inner{\eta_1}{\eta_2} = v_1^\top v_2$ where $v_1,v_2 \in \Real^p$ are vector representations of 
$\eta_1$ and $\eta_2$, and $p = \dim \SOd(d) = d(d-1)/2$.
We define the function $\Exp: \Real^p \to \SOd(d)$ as,
\begin{equation}
	\Exp(v) = \exp(\hatop{v}),
	\label{eq:Exp_def}
\end{equation}
where $\exp(\cdot)$ denotes the conventional matrix exponential.
Note that $\Exp: \Real^p \to \SOd(d)$
should not be confused with the exponential mapping on Riemannian manifolds $\Exp_x: T_x \Mcal \to \Mcal$, 
although the two are closely related in the case of rotations.
Specifically, at a point $R \in \SOd(d)$ where $d \in \{2,3\}$,
the exponential map can be written as $\Exp_R(\eta) = R \Exp(v)$.
{Lastly, we also denote $\Log$ as the inverse of $\Exp$ in \eqref{eq:Exp_def}.}

\section{Related Works}
\label{sec:related_work}

\edit{In this section, we review related work in collaborative SLAM (\cref{sec:related_work:cslam}), 
graph structure on rotation averaging and PGO (\cref{sec:related_work:graph}), 
and the applications of spectral sparsification and Laplacian linear solvers (\cref{sec:related_work:sparsification}).}

\subsection{Collaborative SLAM}
\label{sec:related_work:cslam}
{
\myParagraph{Systems}
State-of-the-art collaborative SLAM systems rely on optimization-based back-ends to accurately estimate robots' trajectories and maps in a global reference frame.
In \emph{fully centralized} systems (\eg, \cite{Schmuck21covins,Chang2022lamp2, Cramariuc2022maplab}), robots upload their processed measurements to a central server that in practice could contain \eg, odometry factors, visual keyframes, and/or lidar keyed scans. Using this information, the server is responsible for managing the multi-robot maps and solving the entire back-end optimization problem. 
In contrast,  in systems leveraging \emph{distributed} computation (\eg \cite{Cieslewski18DataEfficient, Lajoie20DOOR, Tian21KimeraMulti, Zhu18DistMotionAverage,Tian2022Lazy}),
robots collaborate to solve back-end optimization by coordinating with a server or among themselves. 
The resulting communication usually involves exchanging intermediate iterates needed by distributed optimization to attain convergence.}

{
\myParagraph{Optimization Algorithms}
To solve factor graph optimization in a multi-robot setting,} Cunningham~\etal\ develop DDF-SAM~\cite{Cunningham10DDFSAM,Cunningham13DDFSAM2} 
{where each agent communicates a ``condensed graph'' produced by marginalizing out internal variables (those without inter-robot measurements) in its local Gaussian factor graph.}
{Researchers have also developed information-based sparsification methods to sparsify the dense information matrix after marginalization using Chow-Liu tree (\eg, \cite{kretzschmar2012compression,carlevaris2014generic}) or convex optimization (\eg, \cite{mazuran2014nonlinear,Paull15Sparsification}). }
{In these works, sparsification is guided by an information-theoretic objective such as the Kullback-Leibler divergence, and requires linearization to compute the information matrix.}
{In comparison, our approach sparsifies the graph Laplacian that does not depend on linearization, and furthermore the sparsified results are used by collaborative optimization to achieve fast convergence.}

From an optimization perspective, marginalization corresponds to a domain decomposition approach (\eg, see \cite[Chapter~14]{Saad2003Iterative}) where one eliminates a subset of variables in the Hessian using the Schur complement.
Related works use sparse approximations of the resulting matrix (\eg, with tree-based sparsity patterns) to precondition the optimization \cite{Agarwal10BAL,Dellaert10PCG,Wu11MulticoreBA,Kushal12Visibility}.
Recent work \cite{Tian2022Lazy} combines domain decomposition with event-triggered transmission to improve communication efficiency during collaborative estimation.
Zhang~\etal~\cite{Zhang21MRiSAM2} develop a centralized incremental solver for multi-robot SLAM.
Fully decentralized solvers for SLAM have also gained increasing attention; see \cite{Tron2014CameraNetwork,Choudhary17IJRR,Tian21DC2PGO,Tian20ASAPP,Fan2021MajorizationJournal,Murai2022RobotWeb,Lajoie21Survey}.
\edit{In the broader field of optimization, related works include decentralized consensus optimization methods such as \cite{Mateos10CADMM,Shi2015Extra,Nedic2017DIGing,Di2016Next}. 
Compared to these fully decentralized methods, the proposed approach assumes a central server but achieves significantly faster convergence by implementing approximate second-order optimization.}

\subsection{Graph Structure in Rotation Averaging and PGO}
\label{sec:related_work:graph}
Prior works have investigated the impact of graph structure on rotation averaging and PGO problems from different perspectives.
One line of research \cite{Boumal2014CramerRao,Khosoussi2019Reliable,Chen2021CramerRao} adopts an estimation-theoretic approach and shows that the Fisher information matrix is closely related to the underlying graph Laplacian matrix.
Eriksson~\etal~\cite{Eriksson2018Duality} establish sufficient conditions for strong duality to hold in rotation averaging, where the derived analytical error bound depends on the algebraic connectivity of the graph.
Recently, Bernreiter~\etal~\cite{bernreiter2022collaborative} use tools from graph signal processing to correct onboard estimation errors in multi-robot mapping.
Doherty~\etal~\cite{Doherty2022Spectral} propose a measurement selection approach for pose graph SLAM that seeks to maximize the algebraic connectivity of the underlying graph.
This paper differs from the aforementioned works by analyzing the impact of graph structure on the underlying optimization problems,
and exploiting the theoretical analysis to design novel optimization algorithms in the multi-robot setting.

Among related works in this area, the ones most related to this paper are \cite{Carlone2013Convergence, Tron2012Thesis, Wilson2016Convexity,Wilson2020Convexity,Nasiri2021GaussNewton}.
Carlone~\cite{Carlone2013Convergence} analyzes the influences of graph connectivity and noise level on the convergence of Gauss-Newton methods when solving PGO.
{Tron~\cite{Tron2012Thesis} derives the Riemannian Hessian of rotation averaging under the geodesic distance, and uses the results to prove convergence of Riemannian gradient descent.}
In a pair of papers~\cite{Wilson2016Convexity,Wilson2020Convexity}, Wilson~\etal\ study the local convexity of rotation averaging under the geodesic distance, by bounding the Riemannian Hessian using the Laplacian of a suitably weighted graph.
Recently, Nasiri~\etal~\cite{Nasiri2021GaussNewton} develop a Gauss-Newton method for rotation averaging under the chordal distance, and show that its convergence basin is influenced by the norm of the inverse reduced Laplacian matrix.
Our work differs from \cite{Carlone2013Convergence,Tron2012Thesis,Wilson2016Convexity,Wilson2020Convexity,Nasiri2021GaussNewton}
by focusing on the development of fast and communication-efficient solvers in multi-robot teams with provable performance guarantees.
During this process, we also prove new results on the connections between the Riemannian Hessian and graph Laplacian,
and show that they hold under both geodesic and chordal distance.

\subsection{Spectral Sparsification and Laplacian Solvers}
\label{sec:related_work:sparsification}
A remarkable property of graph Laplacians is that they admit sparse approximations; see \cite{Batson2013SpectralSurvey} for a survey.
Spielman and Srivastava~\cite{Spielman2011Sampling} show that every graph with $n$ vertices can be approximated using a sparse graph with  $O(n\log n)$ edges.
This is achieved using a random sampling procedure that selects each edge with probability proportional to its effective resistance, 
which intuitively measures the importance of each edge to the whole graph.
Batson~\etal~\cite{Batson2009twice} develop a procedure based on the so-called barrier functions for constructing \emph{linear-sized} sparsifiers.
Another line of work~\cite{kyng2016approximate,durfee2017spanning} employs sparsification during approximate Gaussian elimination.
Spectral sparsification is one of the main tools that enables recent progress in fast Laplacian solvers (\ie, for solving linear systems of the form $Lx = b$, where $L$ is a graph Laplacian); see \cite{Vishnoi2013LaplacianSolver} for a survey.
Peng and Spielman~\cite{peng2014parallel} develop a parallel solver that invokes sparsification as a subroutine, which is improved and extended in following works \cite{lee2015sparsified,kyng2016sparsified}.
\techreport{Recently, Tutunov~\cite{Tutunov19DistNewton} extends the approach in \cite{peng2014parallel} to solve decentralized consensus optimization problems.}
In this work, 
we leverage spectral sparsification to design communication-efficient collaborative optimization methods for rotation averaging with provable convergence guarantees.


{
	
\section{Problem Formulation}
\label{sec:problem_definition}

This section formally defines the rotation averaging and translation estimation problems in the multi-robot context.
For clarity, here we introduce the problems without considering outlier measurements, and present extensions to outlier-robust optimization in \cref{sec:gnc}.
We review the communication and computation architectures used by our algorithms.
Finally, we discuss relevant applications in multi-robot SLAM and SfM.

\subsection{Rotation Averaging}
\label{sec:problem_definition:rotation}
We model rotation averaging using an undirected \emph{measurement graph} $G = (\Vcal, \Ecal)$.
Each vertex $i \in \Vcal = [n]$ corresponds to a rotation variable $R_i \in \SOd(d)$ to be estimated.
Each edge $(i,j) \in \Ecal$ corresponds to a noisy relative measurement of the form,
\begin{equation}
	\Rtilde_{ij} = \Runder_i^\top \Runder_j R^{\text{err}}_{ij},
	\label{eq:rotation_noise_model}
\end{equation}
where $\Runder_i, \Runder_j \in \SOd(d)$ are the latent (ground truth) rotations
and $R^{\text{err}}_{ij} \in \SOd(d)$ is the measurement noise.
In standard rotation averaging, we aim to estimate the rotations by minimizing the sum of squared measurement residuals, which corresponds to the formulation in \cref{prob:rotation_averaging}.

\begin{problem}[Rotation Averaging]
	\normalfont
	\begin{equation}
		\label{eq:rotation_averaging}
		\underset{R =  (R_1, \hdots, R_n) \in \SOd(d)^n    }{\minimize}
		\quad \sum_{(i,j) \in \Ecal} \kappa_{ij} \varphi(R_i \Rtilde_{ij}, R_j).
	\end{equation}
	For each edge $(i,j) \in \Ecal$, $\kappa_{ij} > 0$ is the corresponding measurement weight.
	The function $\varphi$ is defined as either the squared geodesic \eqref{eq:squared_geodesic_distance} 
	or chordal distance \eqref{eq:squared_chordal_distance}, 
	\begin{subnumcases}{\varphi(R_i \Rtilde_{ij}, R_j) \triangleq }
		\frac{1}{2} \norm{\Log(\Rtilde_{ij}^\top R_i^\top R_j)}^2, 
		\label{eq:squared_geodesic_distance}
		\\
		\frac{1}{2} \norm{R_i \Rtilde_{ij} - R_j}^2_F.
		\label{eq:squared_chordal_distance}
	\end{subnumcases}
	\label{prob:rotation_averaging}
\end{problem}
In the multi-robot setting,  each robot owns a subset of all rotation variables and only knows about measurements involving its own variables; see \cref{fig:example:original_graph} for an illustration.

\subsection{Translation Estimation}
\label{sec:problem_definition:translation}
Similar to rotation averaging, we also consider the problem of estimating multiple translation vectors given noisy relative translation measurements.
\begin{problem}[Translation Estimation]
	\normalfont
	\begin{equation}
		\label{eq:translation_recovery}
		\underset{t =   (t_1, \hdots, t_n)   \in \Real^{d \times n}}{\minimize}
		\quad \sum_{(i,j) \in \Ecal} \frac{\tau_{ij}}{2} \norm{t_j - t_i - \that_{ij}}_2^2.
	\end{equation}
	\label{prob:translation_recovery}
\end{problem}
Note that \eqref{eq:translation_recovery} is a linear least squares problem.
Similar to rotation averaging,  \eqref{eq:translation_recovery} can be modeled using the undirected measurement graph $G = (\Vcal, \Ecal)$,
where vertex $i$ represents the translation variable $t_i  \in \Real^d$ to be estimated, and edge $(i,j) \in \Ecal$ represents the relative translation measurement $\that_{ij} \in \Real^d$. 
Lastly, $\tau_{ij} > 0$ is the positive weight associated with measurement $(i,j) \in \Ecal$.

\subsection{Communication and Computation Architecture}
\label{sec:problem_definition:architecture}
In this work, we consider solving \cref{prob:rotation_averaging,prob:translation_recovery}
under the \emph{server-client} architecture.
As shown in \cref{fig:example:communication_architecture}, a central server coordinates with all robots (clients) to solve the overall problem by distributing the underlying computation to the entire team.
In practice, the server could itself be a robot (\eg, in multi-robot exploration scenarios) or a remote machine (\eg, in cloud-based AR/VR applications).
Each iteration (communication round) consists of an \emph{upload} stage in which robots perform parallel local computations and transmit their intermediate information to the server,
and a \emph{download} stage in which the server aggregates information from all robots and broadcasts back the result.
When a direct communication link to the server does not exist, a robot can still participate in this framework by relaying its information through other robots.
By leveraging local computations, the server-client architecture can scale better compared to a fully centralized approach in which the server solves the entire optimization problem.
At the same time, by implementing second-order optimization algorithms, this architecture also produces significantly faster and more accurate solutions compared to fully distributed approaches that rely on first-order optimization.
In the experiments, we demonstrate the scalability and fast convergence of our approach on large SLAM and SfM problems.

\subsection{Applications}
\label{sec:problem_definition:applications}

Rotation averaging (\cref{prob:rotation_averaging}) is a fundamental problem in robotics and computer vision.
In distributed camera networks (\eg,\cite{Tron2014CameraNetwork}), rotation averaging is used to estimate the orientations of spatially distributed cameras with overlapping fields of view.
In distributed SfM (\eg, \cite{Zhu18DistMotionAverage}), rotation averaging is a key step to build large-scale 3D reconstructions from many images.
Furthermore, in the context of collaborative SLAM, 
rotation averaging and translation estimation (\cref{prob:translation_recovery}) can be combined to provide accurate \emph{initialization}
for PGO \cite{Carlone2015Initialization}.
In state-of-the-art PGO solvers, the cost function often has a separable structure between rotation and translation measurements.
For example, SE-Sync \cite{Rosen19IJRR} uses the formulation,
\begin{equation}
	\label{eq:pose_graph_optimization}
	\begin{aligned}
		\minimize_{
			\substack{
				R =  ( R_1, \hdots, R_n)   \in \SOd(d)^n, \\
				t  = (t_1, \hdots, t_n) \in \Real^{d \times n}}}
		\sum_{(i,j) \in \Ecal} {\kappa_{ij}} \norm{R_i \Rtilde_{ij} - R_j}^2_F \\
		+ \sum_{(i,j) \in \Ecal}  {\tau_{ij}} \norm{t_j - t_i - R_i \ttilde_{ij}}^2_2.
	\end{aligned}
\end{equation}
In \eqref{eq:pose_graph_optimization}, $R_i \in \SOd(d)$ and $t_i \in \Real^d$ are rotation matrices and translation vectors to be estimated,
$\Rtilde_{ij} \in \SOd(d)$ and $\ttilde_{ij} \in \Real^d$ are noisy relative rotation and translation measurements,
and $\kappa_{ij}, \tau_{ij} > 0$ are constant measurement weights.
Notice that in \eqref{eq:pose_graph_optimization}, the first sum of terms is equivalent to rotation averaging (\cref{prob:rotation_averaging}) under the chordal distance.
Furthermore, given fixed rotation estimates $\Rhat \in \SOd(d)^n$, the second sum of terms is equivalent to translation estimation (\cref{prob:translation_recovery}) 
where each $\that_{ij}$ in \eqref{eq:translation_recovery} is given by $\that_{ij} = \Rhat_i \ttilde_{ij}$.
In both cases, the equivalence is up to a multiplying factor of $1/2$, but this is inconsequential since it does not change solutions to the optimization problems.
Following Carlone~\etal~\cite{Carlone2015Initialization}, we use these observations to initialize PGO in a two-stage process.
The first stage initializes the rotation variables using the proposed rotation averaging solver (\cref{sec:collaborative_rotation_averaging}).
Given the initial rotation estimates, the second stage initializes the translations using the proposed translation estimation solver (\cref{sec:collaborative_translation_recovery}).
We note that this initialization scheme does not have theoretical guarantees with respect to the full PGO problem.
However, we still demonstrate its practical value through our experiments.


}
\section{Laplacian Systems Arising from Rotation Averaging and Translation Estimation}
\label{sec:laplacian_systems}

{
In this section, we show that for rotation averaging (\cref{prob:rotation_averaging}) and translation estimation (\cref{prob:translation_recovery}), 
their Hessian matrices are closely related to the Laplacians of suitably weighted graphs.}
The theoretical relations we establish in this section pave the way for designing fast and communication-efficient solvers in \cref{sec:algorithm}.

\subsection{Rotation Averaging}
\label{sec:laplacian_systems:rotation}

{To solve rotation averaging (\cref{prob:rotation_averaging}), we resort to an iterative Riemannian optimization framework.}
Before proceeding, however, one needs to be careful of the inherent \emph{gauge-symmetry} of rotation averaging:
in \eqref{eq:rotation_averaging}, note that left multiplying each rotation $R_i \in \SOd(d), i \in [n]$ by a common rotation $S \in \SOd(d)$
does not change the cost function.
As a result, each solution $R = (R_1, \hdots, R_n) \in \SOd(d)^n$ actually corresponds to an \emph{equivalence class} of solutions in the form of,
\begin{equation}
	[R] = \{(SR_1, \hdots, SR_n), \; S \in \SOd(d)\}.
	\label{eq:rotation_averaging_equivalent_class}
\end{equation}
The equivalence relation \eqref{eq:rotation_averaging_equivalent_class} shows that rotation averaging is actually an optimization problem defined over a \emph{quotient manifold} $\Mcal = \Mcalbar / \sim$, 
where $\Mcalbar = \SOd(d)^n$ is called the total space and $\sim$ denotes the equivalence relation underlying \eqref{eq:rotation_averaging_equivalent_class}; see \cite[Chapter~9]{Boumal20Book} for more details.
Accounting for the quotient structure is critical for establishing the relation between the Hessian and the graph Laplacian.

In this work, we are interested in applying Newton's method on the quotient manifold $\Mcal$ due to its superior convergence rate.
The Newton update can be derived by considering a local perturbation of the cost function.
Specifically, let $R = (R_1, \hdots, R_n) \in \SOd(d)^n$ be our current rotation estimates.
For each rotation matrix $R_i$, we seek a local correction to it in the form of $\Exp(v_i)R_i$, 
where $v_i \in \Real^p$ is some vector to be determined and $\Exp(\cdot)$ is defined in \eqref{eq:Exp_def}.
In \eqref{eq:rotation_averaging},
replacing each $R_i$ with its correction  
$\Exp(v_i)R_i$ leads to the following local approximation\footnote{
	\mainpaper{\edit{The approximation defined in \eqref{eq:RA_full_pullback_global_def} is closely related to the standard \emph{pullback function} in Riemannian optimization; see \cite[Appendix~II-D]{Tian2022Sparsification}.}}
	\techreport{\edit{The approximation defined in \eqref{eq:RA_full_pullback_global_def} is closely related to the standard \emph{pullback function} in Riemannian optimization; see Appendix~\ref{sec:rotation_averaging_hessian_analysis:pullback_discussion}.}}
	In this work, we use \eqref{eq:RA_full_pullback_global_def} since the resulting Hessian has a particularly interesting relationship with the graph Laplacian matrix, as shown in \cref{thm:Hessian_approximation}.} of the optimization problem,

\begin{equation}
\underset{v \in \Real^{pn}}{\min} 
\quad
h(v; R) 
\triangleq \sum_{(i,j) \in \Ecal} 
\kappa_{ij}
\varphi(\Exp(v_i) R_i \Rtilde_{ij}, \Exp(v_j) R_j).
\label{eq:RA_full_pullback_global_def}
\end{equation}
In \eqref{eq:RA_full_pullback_global_def}, the overall vector $v \in \Real^{pn}$ is formed by concatenating all $v_i$'s.
Compared to \eqref{eq:rotation_averaging}, the optimization variable in \eqref{eq:RA_full_pullback_global_def} becomes the vector $v$ and the rotations $R$ are treated as fixed.
Furthermore, 
we note that the quotient structure of \cref{prob:rotation_averaging} gives rise to the following \emph{vertical space} \cite[Chapter~9.4]{Boumal20Book}
that summarizes all directions of change along which \eqref{eq:RA_full_pullback_global_def} is invariant,
\begin{equation}
	\Ncal = \image(1_n \otimes I_p) \subset \Real^{pn}.
	\label{eq:vertical_space_RA}
\end{equation}
Intuitively, $\Ncal$ captures the action of any global left rotation.
Indeed, for any $v \in \Ncal$, we have $\Exp(v_i) = \Exp(v_j)$ for all $i,j \in [n]$,
and thus the cost function \eqref{eq:RA_full_pullback_global_def} remains constant.
Let us denote the gradient and Hessian of \eqref{eq:RA_full_pullback_global_def} as follows,
\begin{equation}
	\gbar(R) \triangleq \nabla h(v; R) \rvert_{v = 0}, \quad 
	\Hbar(R) \triangleq \nabla^2 h(v; R) \rvert_{v = 0}.
	\label{eq:gradient_and_hessian_def}
\end{equation}
Our notations $\gbar(R)$ and $\Hbar(R)$ serve to emphasize that the gradient and Hessian are defined in the total space $\Mcalbar$ and depend on the current rotation estimates $R$. 
\edit{Let $\Hcal \triangleq \Ncal^\perp$ denote the \emph{horizontal space}, which is the orthogonal complement of the vertical space $\Ncal$.}
In \cite[Chapter~9.12]{Boumal20Book}, it is shown that executing the Newton update on the quotient manifold amounts to finding \edit{the solution $v \in \Hcal$ to
the linear system,}
\begin{equation}
	(\underbrace{P_H \, \Hbar(R) \, P_H}_{H(R)}) v 
	= -\gbar(R),  
	\label{eq:Newton_step_quotient}
\end{equation}
where $P_H$ is the orthogonal projection onto $\Hcal$.
{We note that $P_H$ is symmetric, and so is $H(R)$. 
Furthermore, it holds that $\gbar(R) = P_H \gbar(R)$, which follows from known results on optimization over quotient manifolds (see \cref{rm:approximate_newton_step_feasibility} for details).}
Intuitively, including $P_H$ in \eqref{eq:Newton_step_quotient} accounts for the gauge symmetry by eliminating the effect of any vertical component from $v$.
The following theorem reveals an interesting connection between $H(R)$ defined in \eqref{eq:Newton_step_quotient} and the Laplacian of the underlying graph.

\begin{theorem}[Local Hessian Approximation for Rotation Averaging]
	\label{thm:Hessian_approximation}
	Let $\Runder \in \SOd(d)^n$ denote the set of ground truth rotations 
	from which the noisy measurements $\Rtilde_{ij}$ are generated according to \eqref{eq:rotation_noise_model}.
	For any $\delta \in (0, 1/2)$, there exist constants $\thetabar, r > 0$ such that if,
	\begin{equation}
		\dist(\Rtilde_{ij}, \Runder_i^\top \Runder_j) \leq \thetabar, \; \forall (i,j) \in \Ecal,
		\label{thm:Hessian_approximation:noise_bound}
	\end{equation}
	then for all $R \in B_r(\Rstar) = \{R \in \SOd(d)^n: \dist(R, \Rstar) < r\}$ where 
	$\Rstar \in \SOd(d)^n$ is a global minimizer of \cref{prob:rotation_averaging}, 
	it holds that,
	\begin{equation}
		H(R) \approx_\delta L(G; w) \otimes I_p.
		\label{thm:Hessian_approximation:spectral_approximation}
	\end{equation} 
	In \eqref{thm:Hessian_approximation:spectral_approximation},  
	$G = (\Vcal, \Ecal)$ is the measurement graph, \edit{and $p = \dim \SOd(d)$.}
    {For edge $(i,j) \in \Ecal$, its edge weight $w_{ij}$ 
        is given by $w_{ij} = \kappa_{ij}$ for the squared geodesic distance
        cost \eqref{eq:squared_geodesic_distance}, and $w_{ij} = 2 \kappa_{ij}$ for the squared chordal distance cost    \eqref{eq:squared_chordal_distance}.   }
\end{theorem}
Before proceeding, we note that \cref{thm:Hessian_approximation} directly implies the following bound on the Hessian $H(R)$. 
\begin{corollary}[Local Hessian Bound and Condition Number for Rotation Averaging]
	\label{cor:Hessian_bound}
	Under the assumptions of \cref{thm:Hessian_approximation}, define constants
	$\mu_H = e^{-\delta} \lambda_2(L(G;w))$ and $L_H = e^{\delta} \lambda_n(L(G;w))$.
	Then for all $R \in B_r(\Rstar)$,
	\begin{equation}
		\mu_H P_H \preceq H(R) \preceq L_H P_H.
		\label{eq:Hessian_bound}
	\end{equation}
	In the following, $\kappa_H = L_H / \mu_H$ is referred to as the \emph{condition number}.
\end{corollary}

\techreport{{We prove \cref{thm:Hessian_approximation} and \cref{cor:Hessian_bound} in Appendix~\ref{sec:rotation_averaging_hessian_analysis}.}}
\mainpaper{{We prove \cref{thm:Hessian_approximation} and \cref{cor:Hessian_bound} in \cite[Appendix~II]{Tian2022Sparsification}.}}
\cref{thm:Hessian_approximation} shows that under small measurement noise, the Hessian near a global minimizer is well approximated by the Laplacian of an appropriately weighted graph.\footnote{{Currently, \cref{thm:Hessian_approximation} only shows the existence of constants 
	$\thetabar, r > 0$ such that the approximation relation \eqref{thm:Hessian_approximation:spectral_approximation} holds.
	In a nutshell, this is because our proof is based on the following key relation that holds in the limit: 
	if we define $\theta_{ij}(R) = \dist(\Rtilde_{ij}, R_i^\top R_j)$ as the measurement residual of edge $(i,j) \in \Ecal$ at a solution $R \in \SOd(d)^n$,
	then we can show that $H(R) \to L(G; w) \otimes I_p$ as $\theta_{ij}(R) \to 0$ for all $(i,j) \in \Ecal$; 
	\techreport{see discussions around \eqref{eq:H_limit_angular} in the appendix.}
        \mainpaper{see discussions around (98) in \cite[Appendix~II]{Tian2022Sparsification}.}
	While it would be interesting to derive explicit and accurate bounds for $\thetabar$ and $r$ (as a function of $\delta$),
	this would require a substantial improvement to our current proof technique, which we leave for future work.}
}
{In \cref{fig:laplacian_apx_eval}, we perform numerical validation of this result 
using synthetic chordal rotation averaging problems defined over a 3D grid with 125 rotation variables (\cref{fig:laplacian_apx_eval:dataset}).
With a probability of 0.3, we generate noisy relative measurements between pairs of nearby rotations, corrupted by increasing levels of Langevin noise  \cite[Appendix~A]{Rosen19IJRR}.
At each noise level, 
we obtain the global minimizer $\Rstar$ (global optimality is certified using the approach in \cite{Eriksson2018Duality})
and numerically compute the smallest constant $\delta$ such that $H(\Rstar) \approx_\delta L \otimes I_p$.
\cref{fig:laplacian_apx_eval:delta} shows the evolution of the empirical approximation constant $\delta$ as a function of noise level.
In the special case when there is no noise, it can be shown that $H(\Rstar) = L \otimes I_p$, and thus the empirical $\delta$ is zero.
In general, the empirical value of $\delta$ increases smoothly as the noise level increases.
Since the Hessian $H(R)$ varies smoothly with $R$, 
our results confirm that the Laplacian is a good approximation of the Hessian locally around $\Rstar$, as predicted by \cref{thm:Hessian_approximation}.}

\begin{figure}[t]
	\centering
	\begin{subfigure}[t]{0.45\linewidth}
		\includegraphics[width=\textwidth, trim=100 50 100 30, clip]{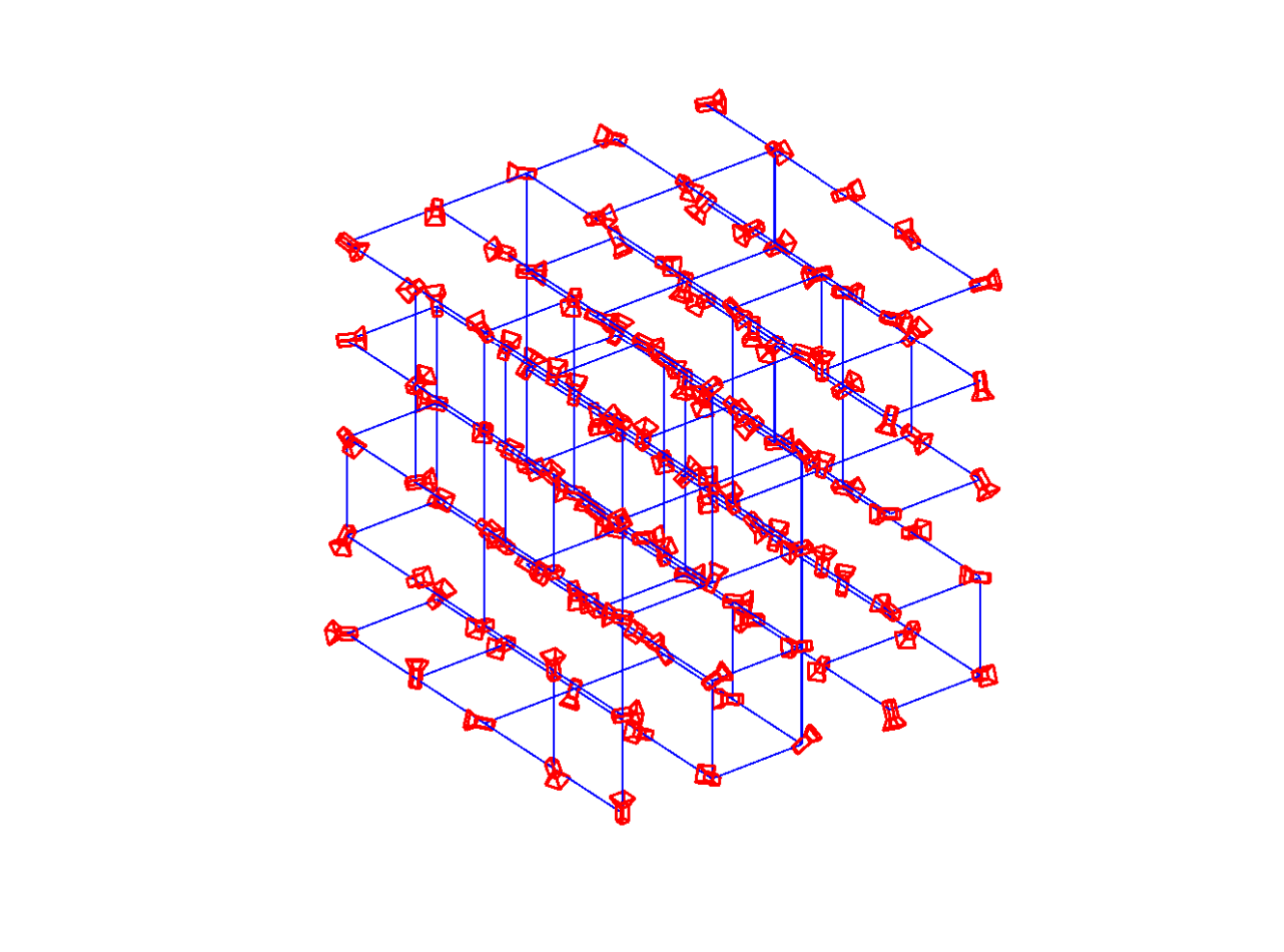}
		\caption{Grid simulation}
		\label{fig:laplacian_apx_eval:dataset}
	\end{subfigure}
	\hfill
	\begin{subfigure}[t]{0.50\linewidth}
		\includegraphics[width=\textwidth]{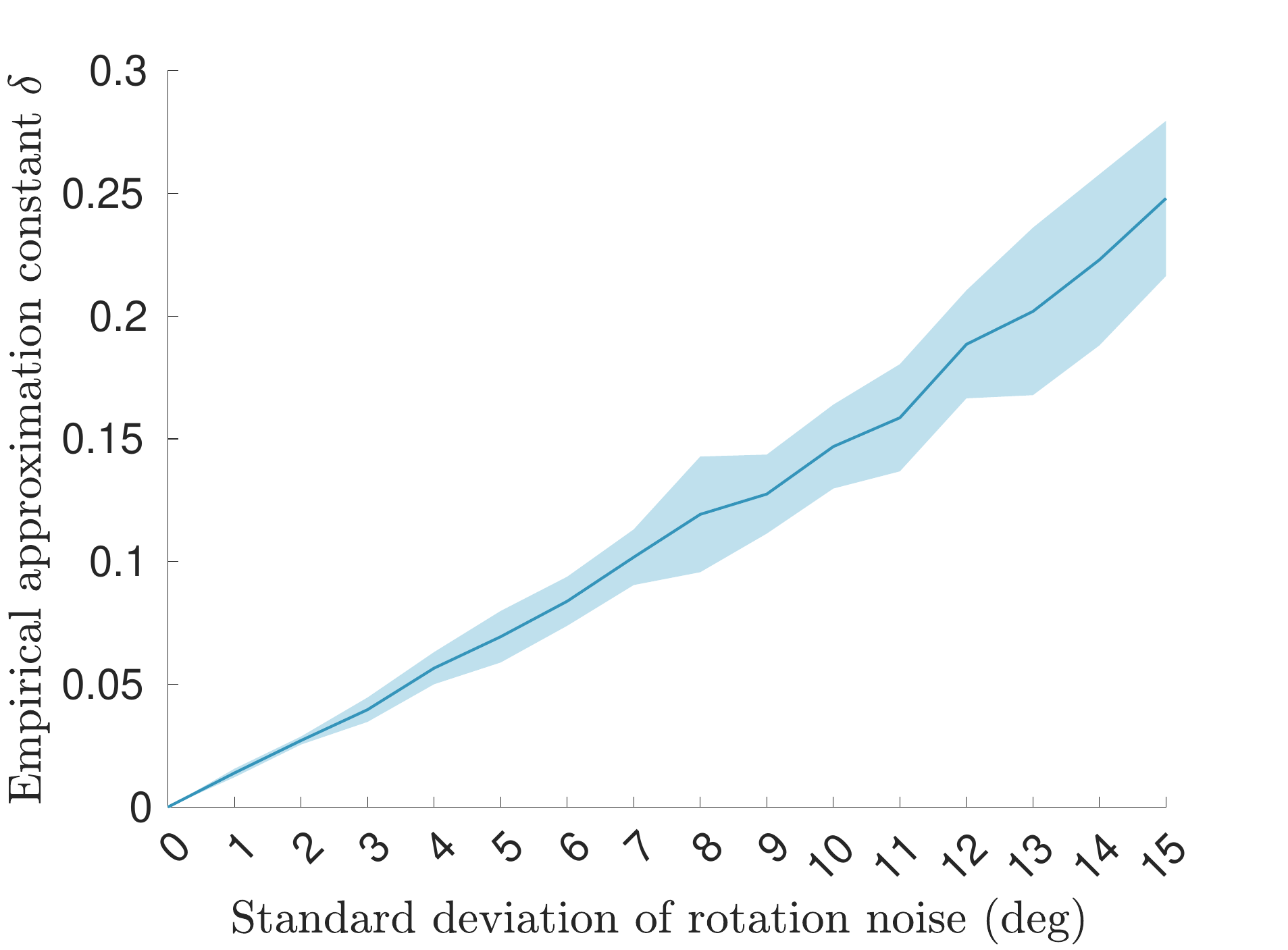}
		\caption{Empirical values of $\delta$}
		\label{fig:laplacian_apx_eval:delta}
	\end{subfigure}	
	\caption{
		Empirical validation of the Hessian approximation relation in \cref{thm:Hessian_approximation}.
		(a) Example synthetic chordal rotation averaging problem with 125 rotations.
		Each rotation is visualized as an oriented camera.
		Each blue edge shows a relative rotation measurement corrupted by Langevin noise.
		(b) Evolution of the empirical approximation constant $\delta$ such that $H(\Rstar) \approx_\delta L \otimes I_p$.
		We perform 20 random runs for each noise level.
		Solid line denotes the average value for $\delta$ and the surrounding shaded area shows one standard deviation.  
	} 
	\label{fig:laplacian_apx_eval}
        \vspace{-.3cm}
\end{figure}

The result in \cref{thm:Hessian_approximation} directly motivates an \emph{approximate Newton method} that replaces the Hessian with its Laplacian approximation.
Specifically, instead of solving \eqref{eq:Newton_step_quotient}, one solves the following \emph{approximate Newton} system,
\begin{equation}
	\left( L(G;w) \otimes I_p \right) v = -\gbar(R).
	\label{eq:rotation_laplacian_system}
\end{equation}
In the following, it would be more convenient to consider the matrix form of the above linear system.
For this purpose, let us define matrices $V, B(R) \in \Real^{n \times p}$, 
\begin{equation}
	V \triangleq \bmat
	v_1^\top \\
	\vdots \\
	v_n^\top
	\emat, \;\;
	B(R) \triangleq \bmat
	-\gbar_1(R)^\top \\
	\vdots \\
	-\gbar_n(R)^\top
	\emat.
	\label{eq:rotation_averaging_matrix_variables}
\end{equation}
Using properties of the Kronecker product, we can show that \eqref{eq:rotation_laplacian_system} is equivalent to, 
\begin{equation}
	L(G;w) V = B(R).
	\label{eq:rotation_laplacian_system_matrix_form}
\end{equation}
\cref{alg:laplacian_newton_method} shows the pseudocode of the approximate Newton algorithm.
Compared to the original Newton's method, \cref{alg:laplacian_newton_method} uses a \emph{constant matrix} 
across all iterations, and hence could be significantly more computationally efficient {by avoiding to re-compute and re-factorize the Hessian matrix at every iteration.}
For this reason, we believe that \cref{alg:laplacian_newton_method} could be of independent interest for standard (centralized) rotation averaging.
Furthermore, in \cref{sec:algorithm}, we show that \cref{alg:laplacian_newton_method} admits communication-efficient extensions in multi-robot settings.

\begin{remark}[Connections with prior work]
	{\cref{thm:Hessian_approximation} leverages prior theories developed by Tron~\cite{Tron2012Thesis} and Wilson~\etal~\cite{Wilson2016Convexity,Wilson2020Convexity} and 
	extend them to cover rotation averaging under both geodesic and chordal distance metrics.}
	Nasiri~\etal~\cite{Nasiri2021GaussNewton} first developed
	\cref{alg:laplacian_newton_method} for chordal rotation averaging using a Gauss-Newton formulation.
	In contrast,
	we motivate \cref{alg:laplacian_newton_method} by proving the theoretical approximation relation between the Hessian and the graph Laplacian (\cref{thm:Hessian_approximation}).
	Lastly, the theoretical approximation relation we establish also allows us to prove local linear convergence for our methods.
	\label{rm:connections_with_Nasiri} 
\end{remark}

\begin{remark}[Feasibility of the approximate Newton system]
	Using the properties of the graph Laplacian and the Kronecker product, 
	we see that $\ker(L(G;w) \otimes I_p) = \Ncal$ where $\Ncal$ is the vertical space
	defined in \eqref{eq:vertical_space_RA}. 
	Furthermore, in \cite[Chapter~9.8]{Boumal20Book}, it is shown that $\gbar(R) \perp \Ncal$. 
	Thus, we conclude that $\gbar(R) \in \image(L(G;w) \otimes I_p)$, 
	\ie, the linear system \eqref{eq:rotation_laplacian_system} and its equivalent matrix form \eqref{eq:rotation_laplacian_system_matrix_form} are always feasible.
	In fact, the system is singular and hence admits infinitely many solutions. 
	Similar to the original Newton's method on quotient manifold, 
	we will select the minimum norm solution $v$
	\edit{which guarantees that $v \in \Hcal$ \cite[Chapter~9.12]{Boumal20Book}.}
	\label{rm:approximate_newton_step_feasibility}
\end{remark}

\begin{algorithm}[!t]
	\caption{\textsc{Approximate Newton's Method for Rotation Averaging}}
	\label{alg:laplacian_newton_method}
	\begin{algorithmic}[1]
		\small
		\For{iteration $k = 0, 1, \hdots$}
		\State Compute approximate Newton update by solving
		$ L(G;w) V^k = B(R^k). $
		\label{alg:laplacian_newton_method:newton_step}
		\State Update iterate by $R^{k+1}_i = \Exp(v^k_i) R_i^k$, for all $i \in [n]$.
		\EndFor
	\end{algorithmic}
\end{algorithm}

\subsection{Translation Estimation}
\label{sec:laplacian_systems:translation}

{Unlike rotation averaging, translation estimation (\cref{prob:translation_recovery}) is a convex linear least squares problem.}
In particular, it can be shown that \cref{prob:translation_recovery} is equivalent to a linear system involving the graph Laplacian $L(G; \tau)$, {where $\tau: \Ecal \to \Real_{>0}$ is the edge weight function that assigns each edge $(i,j) \in \Ecal$ a weight given by the corresponding translation measurement weight $\tau_{ij}$ in \cref{prob:translation_recovery}.}
Denote 
$M_t =
\bmat
t_1 & \hdots & t_n
\emat^\top
\in \Real^{n \times d}$ 
as the matrix where each row corresponds to a translation vector to be estimated.
One can show that the optimal translations are solutions of,
\begin{equation}
	L(G; \tau) M_t = B_t,
	\label{eq:translation_laplacian_system_matrix_form}
\end{equation}
where $B_t \in \Real^{n \times d}$ is a constant matrix that only depends on the measurements.
Furthermore, each column of $B_t$ belongs to the image of the Laplacian $L(G; \tau)$, so \eqref{eq:translation_laplacian_system_matrix_form} is always feasible;
see \cite[Appendix~B.2]{Rosen19IJRR} for details.
To conclude this section,  
we note that similar to rotation averaging, translation estimation (\cref{prob:translation_recovery}) is subject to a gauge symmetry.
Specifically, two translation solutions $M_t$ and $M_t'$ are equivalent if they only differ by a global translation. 
\edit{Mathematically, this means that $M_t = M_t' + 1_n t_0^\top$ where $1_n \in \Real^n$ is the vector of all ones and $t_0 \in \Real^d$ is the constant global translation vector.}

\section{Algorithms and Performance Guarantees}
\label{sec:algorithm}

In \cref{sec:laplacian_systems}, we have shown that \emph{Laplacian systems} naturally arise when solving the rotation averaging and translation estimation problems; 
see \eqref{eq:rotation_laplacian_system_matrix_form} and \eqref{eq:translation_laplacian_system_matrix_form}, respectively.
Recall that we seek to find the solution $X \in \Real^{n \times p}$ to a linear system of the form,
\begin{equation}
	LX = B,
	\label{eq:laplacian_system_matrix_form}
 \vspace{-0.15cm}
\end{equation}
where $L \in \PSD^n$ is the Laplacian {of the multi-robot measurement graph (see \cref{fig:example:original_graph})},
and each column of $B \in \Real^{n \times p }$ is in the image of $L$ so that \eqref{eq:laplacian_system_matrix_form} is always feasible.
\edit{For rotation averaging, we have $p = \dim \SOd(d)$, and for translation estimation, we have $p = \dim \Real^d = d$.}
In \cref{sec:collaborative_laplacian_solver}, we develop a communication-efficient {solver for \eqref{eq:laplacian_system_matrix_form} under the server-client architecture described in \cref{sec:problem_definition:architecture}.} 
Then, in \cref{sec:collaborative_rotation_averaging} and \cref{sec:collaborative_translation_recovery}, we use the developed solver to design communication-efficient algorithms for collaborative rotation averaging and translation estimation, and establish convergence guarantees for both cases.
{Lastly, in \cref{sec:gnc}, we present extension to outlier-robust estimation based on GNC.}

\subsection{A Collaborative Laplacian Solver with Spectral Sparsification}
\label{sec:collaborative_laplacian_solver}

We propose to solve \eqref{eq:laplacian_system_matrix_form} using the \emph{domain decomposition} framework \cite[Chapter~14]{Saad2003Iterative},
which has been utilized in earlier works such as DDF-SAM \cite{Cunningham10DDFSAM,Cunningham12DataAssociation,Cunningham13DDFSAM2} to solve collaborative SLAM problems. 
This is motivated by the fact that in the multi-robot measurement graph with $m$ robots, 
there is a natural disjoint partitioning of the vertex set $\Vcal$:
\begin{equation}
	\Vcal = \Vcal_1 \uplus \hdots \uplus \Vcal_m,
	\label{eq:vertex_partition}
\end{equation}
where $\Vcal_\alpha$ contains all vertices (variables) of robot $\alpha \in [m]$ \edit{and $\uplus$ denotes the disjoint union}.
Furthermore, $\Vcal_\alpha$ can be partitioned as
$\Vcal_\alpha = \Fcal_\alpha \, \uplus \, \Ccal_\alpha$ 
where $\Ccal_\alpha$ denotes all separator (interface) vertices
and $\Fcal_\alpha$ denotes all interior vertices of robot $\alpha$.
In multi-robot SLAM, the separators are given by the set of variables that have inter-robot measurements; see \cref{fig:example:original_graph}.
Note that given the set of all separators 
$\Ccal = \Ccal_1 \, \uplus \, \hdots \, \uplus \, \Ccal_m$,
robots' interior vertices $\Fcal_\alpha$ become disconnected from each other.
The natural vertex partitioning in \eqref{eq:vertex_partition} further gives rise to a disjoint partitioning of the 
edge set,
\begin{equation}
	\Ecal = \Ecal_1 \uplus \hdots \uplus \Ecal_m \uplus \Ecal_c.
	\label{eq:edge_partition}
\end{equation}
For each robot $\alpha \in [m]$, its local edge set $\Ecal_\alpha$ consists of all edges that connect two vertices from $\Vcal_\alpha$.
In \cref{fig:example:original_graph}, the local edges are shown using colors corresponding to the robots.
The remaining \emph{inter-robot} edges form $\Ecal_c$,
which are highlighted as bold black edges in \cref{fig:example:original_graph}.

In domain decomposition, we adopt a variable ordering in which the interior nodes 
$\Fcal = \Fcal_1 \uplus \hdots \uplus \Fcal_m$ appear before the separators 
$\Ccal = \Ccal_1 \uplus \hdots \uplus \Ccal_m$.
With this variable ordering, the Laplacian system \eqref{eq:laplacian_system_matrix_form} can be rewritten as,
\begin{equation}
	\begin{bmatrix}
		L_{11}   &         &         & L_{1c}         \\
		         & \ddots  &         & \vdots         \\
		         &         & L_{mm}  & L_{mc}         \\
		L_{c1}   & \hdots  & L_{cm}  & L_{cc}
	\end{bmatrix}
	\begin{bmatrix}
		X_1 \\ 
		\vdots \\
		X_m \\
		X_c
	\end{bmatrix}
	= 
	\begin{bmatrix}
		B_1 \\ 
		\vdots \\
		B_m \\
		B_c
	\end{bmatrix}.
	\label{eq:laplacian_arrowhead_pattern}
\end{equation}
For $\alpha \in [m]$, $X_\alpha$ and $B_\alpha$ denote the rows of $X$ and $B$ in \eqref{eq:laplacian_system_matrix_form} that correspond to robot $\alpha$'s interior variables $\Fcal_\alpha$.
On the other hand, we treat separators \emph{from all robots} as a single block $\Ccal = \Ccal_1 \uplus \hdots \uplus \Ccal_m$.
\edit{In \eqref{eq:laplacian_arrowhead_pattern}, we use the subscript $c$ to index rows and columns of matrices that correspond to $\Ccal$.}

\begin{remark}[{Computation of \eqref{eq:laplacian_arrowhead_pattern} under the server-client architecture}]
	{Under the server-client architecture we consider, the overall Laplacian system  \eqref{eq:laplacian_arrowhead_pattern} is stored {distributedly} across the robots (clients) and the server. }
	Specifically, since each robot $\alpha$ knows the subgraph induced by its own vertices $\Vcal_\alpha$ (\eg, in \cref{fig:example:original_graph}, robot 2 knows all edges incident to the blue vertices), it independently computes and stores its Laplacian blocks $L_{\alpha\alpha}$ and $L_{\alpha c}$.
	Similarly, each robot $\alpha$ also independently computes and stores the block $B_\alpha$.
	Meanwhile, we assume that the blocks defined over separators $L_{cc}$ and $B_c$ are handled by the central server that performs additional computations.
	\label{rm:distributed_computation_of_laplacian}
\end{remark}

In \eqref{eq:laplacian_arrowhead_pattern}, the special ``arrowhead'' sparsity pattern motivates us to first solve the \emph{reduced system} defined over the separators, \edit{obtained by eliminating all interior nodes using the Schur complement \cite[Chapter~14.2]{Saad2003Iterative}:}
\begin{equation}
	\underbrace{
	\left(
	L_{cc} - \sum_{\alpha \in [m]} L_{c\alpha} L_{\alpha \alpha}^{-1} L_{\alpha c}	
	\right)
	}_{S = \Sc(L, \Fcal)}
	X_c
	= 
	\underbrace{
	B_c - \sum_{\alpha \in [m]} L_{c\alpha} L_{\alpha \alpha}^{-1} B_\alpha
	}_{U}.
	\label{eq:laplacian_reduced_system}
\end{equation}
In the following, let us define $U_\alpha \triangleq L_{c\alpha} L_{\alpha \alpha}^{-1} B_\alpha$ for each robot $\alpha \in [m]$.
Then, the matrix on the right-hand side of \eqref{eq:laplacian_reduced_system} can be written as,
\begin{equation}
	U \triangleq B_c - \sum_{\alpha\in [m]} U_\alpha.
	\label{eq:U_sum}
\end{equation}
Meanwhile, the matrix $S$ defined on the left-hand side of \eqref{eq:laplacian_reduced_system} is the Schur complement resulting from eliminating all interior nodes $\Fcal$ from the full Laplacian matrix $L$, denoted as $S = \Sc(L, \Fcal)$.
The next lemma shows $S$ is the sum of multiple smaller Laplacian matrices.
\begin{lemma}
	For each robot $\alpha \in [m]$, define $G_\alpha = (\Fcal_\alpha \uplus \Ccal, \Ecal_\alpha)$ as its local graph induced by its interior edges $\Ecal_\alpha$.
	Let $S_\alpha$ be the matrix resulting from eliminating robot $\alpha$'s interior vertices 
	$\Fcal_\alpha$ from the Laplacian of $G_\alpha$, \ie, 
	$S_\alpha = \Sc(L(G_\alpha), \Fcal_\alpha)$.
	Furthermore, define $G_c = (\Ccal, \Ecal_c)$ as the graph induced by inter-robot loop closures $\Ecal_c$.
	Then, the matrix $S$ that appears in \eqref{eq:laplacian_reduced_system} can be written as, 
	\begin{equation}
		S = L(G_c) + \sum_{\alpha \in [m]} S_\alpha.
		\label{eq:schur_complement_sum}
	\end{equation}
	\label{lem:schur_complement_sum}
        \vspace{-0.3cm}
\end{lemma}

\techreport{{\cref{lem:schur_complement_sum} is proved in Appendix~\ref{sec:laplacian_appendix:schur_complement_sum}.}}
\mainpaper{{\cref{lem:schur_complement_sum} is proved in \cite[Appendix~III-A]{Tian2022Sparsification}.}}
Since Laplacian matrices are closed under Schur complements~\cite[Fact~4.2]{lee2015sparsified},
each $S_\alpha$ defined in \cref{lem:schur_complement_sum} is also a Laplacian matrix.\techreport{\footnote{
	In \cref{lem:schur_complement_sum}, 
	we can technically define $G_\alpha = (\Fcal_\alpha \uplus \Ccal_\alpha, \Ecal_\alpha)$ 
	since $\Ecal_\alpha$ only involves robot $\alpha$'s vertices.
	However, we choose to involve all separators and define $G_\alpha = (\Fcal_\alpha \uplus \Ccal, \Ecal_\alpha)$, where any separator from $\Ccal \setminus \Ccal_\alpha $ simply does not have any edges.
	This is done for notation simplicity, so that after eliminating $\Fcal_\alpha$ from $G_\alpha$, the resulting $S_\alpha$ matrix is defined over all separators and thus can be added together as in 
	\eqref{eq:schur_complement_sum}.
}}
Furthermore, as a result of \cref{rm:distributed_computation_of_laplacian}, each robot $\alpha$ can independently compute $S_\alpha = \Sc(L(G_\alpha), \Fcal_\alpha)$ and $U_\alpha = L_{c\alpha} L_{\alpha \alpha}^{-1} B_\alpha$.
This observation motivates a method in which robots first transmit their $S_\alpha$ and $U_\alpha$ to the server in parallel.
Upon collecting $S_\alpha$ and $U_\alpha$ from all robots,
the server can then form $S$ using \eqref{eq:schur_complement_sum} and $U$ using \eqref{eq:U_sum}.
It then solves the linear system $S X_c = U$ \eqref{eq:laplacian_reduced_system} and broadcasts the solution $X_c$ back to all robots.
Finally, once robots receive the separator solutions $X_c$, 
they can in parallel recover their interior solutions via back-substitution,
\begin{equation}
	X_\alpha = L_{\alpha \alpha}^{-1} \left( 
	B_\alpha - L_{\alpha c}X_c
	\right).
	\label{eq:laplacian_back_substitution}
\end{equation}
The aforementioned method is a multi-robot implementation of domain decomposition.
While it effectively exploits the separable structure in the problem,
this method can incur significant communication cost as it requires each robot $\alpha$ to transmit its Schur complement matrix $S_\alpha$ that is potentially dense. 
This issue is illustrated in \cref{fig:example:reduced_graph_dense}, 
where for robot 2 (blue) its $S_\alpha$ corresponds to a dense graph over its separators.

\begin{algorithm}[t]
	\caption{\footnotesize \textsc{Sparsified Schur Complement}}
	\label{alg:sparsified_schur_complement}
	\begin{algorithmic}[1]
		\small 
		\Function{$\Stilde$ = SparsifiedSchurComplement}{$L$, $\epsilon$}
		\For{each robot $\alpha$ \textbf{in parallel}}
		\State Compute a sparse approximation $\Stilde_\alpha$ such that $\Stilde_\alpha \approx_\epsilon S_\alpha$.
		\State Upload $\Stilde_\alpha$ to the server. \label{alg:sparsified_schur_complement:upload}
		\EndFor
		\State Server computes and stores $\Stilde = L(G_c) + \sum_{\alpha \in [m]} \Stilde_\alpha$.\label{alg:sparsified_schur_complement:server_sum}
		\EndFunction
	\end{algorithmic}
\end{algorithm}

\begin{algorithm}[t]
	\caption{\footnotesize \textsc{Sparsified Laplacian Solver}}
	\label{alg:sparsified_laplacian_solver}
	\begin{algorithmic}[1]
		\small 
		\Function{$X$ = SparsifiedLaplacianSolver}{$L$, $B$, $\Stilde$}
		\For{each robot $\alpha$ \textbf{in parallel}}
		\State Compute $U_\alpha = L_{c\alpha} L_{\alpha \alpha}^{-1} B_\alpha$.
		\State Upload $U_\alpha$ to the server.\label{alg:sparsified_laplacian_solver:upload}
		\EndFor
		\State Server collects $U_\alpha$ and computes 
		$U = B_c - \sum_{\alpha \in [m]} U_\alpha$.
		\State Server solves $\Stilde X_c = U$ (where $\Stilde$ is obtained from \cref{alg:sparsified_schur_complement}), and broadcasts solution $X_c$ to all robots.
		\label{alg:sparsified_laplacian_solver:apx_reduced_system}
		\For{each robot $\alpha$ \textbf{in parallel}}
		\State Compute interior solution
		$X_\alpha = L_{\alpha \alpha}^{-1} \left( 
		B_\alpha - L_{\alpha c}X_c
		\right).$\label{alg:sparsified_laplacian_solver:back_substitution}
		\EndFor
		\EndFunction
	\end{algorithmic}
\end{algorithm}

In the following, we propose an approximate domain decomposition algorithm that is significantly more communication-efficient while providing \emph{provable approximation guarantees}.
Our method is based on the facts that 
(i) each local Schur complement $S_\alpha$ is itself a graph Laplacian,
and (ii) graph Laplacians admit \emph{spectral sparsifications}~\cite{Batson2013SpectralSurvey}, 
\ie, for a given approximation threshold $\epsilon > 0$, 
one can compute a sparse Laplacian $\Stilde_\alpha$ such that $\Stilde_\alpha \approx_\epsilon S_\alpha$.
{Generally, a larger value of $\epsilon$ leads to a sparser $\Stilde_\alpha$.}
{In this work, we implement the method of Spielman and Srivastava~\cite{Spielman2011Sampling} that sparsifies $S_\alpha$ by sampling edges in the corresponding dense graph based on their \emph{effective resistances}.
Intuitively, the effective resistances measure the importance of edges to the overall graph connectivity. 
The sparse matrix $\Stilde_\alpha$ produced by this method has $O(|\Ccal|\log |\Ccal|)$ entries, as opposed to the worst case $O(|\Ccal|^2)$ entries in $S_\alpha$.}
\techreport{{Appendix~\ref{sec:spectral_sparsification} provides the complete description and pseudocode of the sparsification algorithm.}}
\mainpaper{{We provide the complete description and pseudocode of the sparsification algorithm in \cite[Appendix~I]{Tian2022Sparsification}.}}
\cref{fig:example:reduced_graph_sparse}  illustrates a spectral sparsification for robot 2's dense reduced graph.
In the proposed method, each robot transmits its sparse approximation $\Stilde_\alpha$ instead of the original Schur complement $S_\alpha$.
By summing together these $\Stilde_\alpha$ matrices, the server can obtain a sparse approximation to the original dense Schur complement $S$; see \cref{alg:sparsified_schur_complement}.
Then, we can follow the same procedure as standard domain decomposition to obtain an approximate solution to the Laplacian system \eqref{eq:laplacian_system_matrix_form}; see \cref{alg:sparsified_laplacian_solver}. 
Specifically, the server first solves an \emph{approximate} reduced system using $\Stilde$ obtained from \cref{alg:sparsified_schur_complement} (line~\ref{alg:sparsified_laplacian_solver:apx_reduced_system}).
Then, the interior solution for each robot is recovered using back-substitution (line~\ref{alg:sparsified_laplacian_solver:back_substitution}).

Together, \cref{alg:sparsified_schur_complement,alg:sparsified_laplacian_solver} provide a parallel procedure  for computing an approximate solution to the original Laplacian system \eqref{eq:laplacian_system_matrix_form} {in the server-client architecture}.
Crucially, the use of spectral sparsifiers allows us to establish theoretical guarantees on the accuracy of the approximate solution as stated in the following theorem.

\begin{theorem}[Approximation guarantees of \cref{alg:sparsified_schur_complement,alg:sparsified_laplacian_solver}]
	\label{thm:approximation_guarantees}
	Given a Laplacian system $LX = B$, \cref{alg:sparsified_schur_complement,alg:sparsified_laplacian_solver} together return a solution $\Xtilde \in \Real^{n \times p}$ such that
	$\Ltilde \Xtilde = B$, where $\Ltilde \in \PSD^n$ satisfies,
	\begin{equation}
		\Ltilde \approx_\epsilon L.
		\label{thm:approximation_guarantees:spectrum}
	\end{equation}
	Furthermore, let $\Xstar \in \Real^{n \times p}$ be an exact solution to the input linear system, \ie, $L\Xstar = B$. It holds that,
	\begin{equation}
		\norm{\Xtilde - \Xstar}_L \leq c(\epsilon) \norm{\Xstar}_L,
		\label{thm:approximation_guarantees:solution}
	\end{equation}
	where the constant $c(\epsilon)$ is defined as,
	\begin{equation}
		c(\epsilon) = \sqrt{1 + e^{2\epsilon} - 2e^{-\epsilon}}.
		\label{eq:c_epsilon_def}
	\end{equation}
\end{theorem}

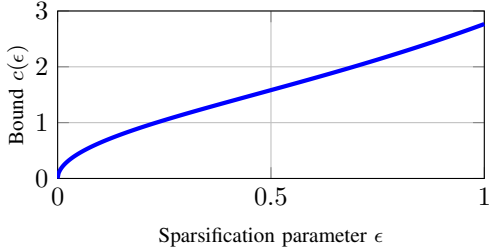
\begin{figure}[t]
	\centering
	\begin{tikzpicture}
		\begin{axis} [width=0.40\textwidth, height=1.5in, smooth, no markers, grid, domain=0:1, xmax=1, ymax=3, xmin=0, ymin=0,xtick={0,0.5,1},y label style={at={(axis description cs:0.17,.5)},rotate=0,anchor=south},
			xlabel={\footnotesize Sparsification parameter $\epsilon$},ylabel={\footnotesize Bound $c(\epsilon)$}]
			\addplot +[ultra thick, blue, samples=500] {pow(1+exp(2*x)-2*exp(-x),0.5)};
		\end{axis}
	\end{tikzpicture}
	\caption{Visualization of $c(\epsilon)$ in \cref{thm:approximation_guarantees}. }\label{fig:c_epsilon}
	\vspace{-.3cm}
\end{figure}

\techreport{{We prove \cref{thm:approximation_guarantees} in Appendix~\ref{sec:laplacian_appendix:approximation_guarantees}.}}
\mainpaper{{We prove \cref{thm:approximation_guarantees} in  \cite[Appendix~III-B]{Tian2022Sparsification}.}}
We have shown that the approximate solution $\Xtilde$ produced by \cref{alg:sparsified_schur_complement,alg:sparsified_laplacian_solver} remains close to the exact solution $\Xstar$ when measured 
	using the ``norm'' induced by the original Laplacian $L$.\techreport{\footnote{
		The reader might question the use of $\norm{\cdot}_L$ in \eqref{thm:approximation_guarantees:solution} because the Laplacian $L$ is singular.
		Indeed, due to the singularity of $L$, $||\Xstar - \Xtilde||_L$ ignores any component of $\Xstar - \Xtilde$ that lives on the kernel of $L$, which is spanned by the vector of all ones $1_n$.
		However, this does not create a problem for us	
		since we only seek to compare
		$\Xstar$ and $\Xtilde$ \emph{when considering both as solutions to the Laplacian system $LX = B$},
		and using $\norm{\cdot}_L$ naturally eliminates any difference on $\ker(L)$ that is inconsequential. 
	}}
	Furthermore, the quality of the approximation is controlled by the sparsification parameter $\epsilon$ through the function 
	$c(\epsilon)$ visualized in \cref{fig:c_epsilon}.
	Note that when $\epsilon = 0$, sparsification is effectively skipped and robots 
	transmit the original dense matrices $S_\alpha$.
	In this case, we have $c(\epsilon) = 0$ and the solution $\Xtilde$ produced by our methods is exact, \ie, $L\Xtilde = B$.
	Meanwhile, by increasing $\epsilon$, our methods smoothly trade off accuracy with communication efficiency.

\begin{remark}[Connections with existing Laplacian solvers \cite{lee2015sparsified,kyng2016sparsified}]
	Our {collaborative} Laplacian solver (\cref{alg:sparsified_schur_complement,alg:sparsified_laplacian_solver}) is inspired by the centralized solvers developed in 
	\cite{lee2015sparsified,kyng2016sparsified} for solving Laplacian systems in nearly linear time.
	However, our result differs from these works by focusing on the use of spectral sparsification in the multi-robot setting to achieve communication efficiency.
	Furthermore, in \cref{sec:collaborative_rotation_averaging},  
	we apply our Laplacian solver on the non-convex Riemannian optimization problem underlying rotation averaging,
	and establish provable convergence guarantees for the resulting Riemannian optimization algorithm.
	\label{rm:connections_with_laplacian_solvers}
\end{remark}

{\begin{remark}[Communication efficiency of \cref{alg:sparsified_schur_complement,alg:sparsified_laplacian_solver} ]
		\label{rm:comm_efficiency_laplacian_solver}
		We discuss the communication costs of \cref{alg:sparsified_schur_complement,alg:sparsified_laplacian_solver} under the server-client architecture.
		Denote the number of separators in the measurement graph as $|\Ccal|$.
		In \cref{alg:sparsified_schur_complement}, each robot uploads the sparsified matrix $\Stilde_\alpha$ to the server (\cref{alg:sparsified_schur_complement:upload}), 
		which is guaranteed to have $O(|\Ccal|\log |\Ccal|)$ entries \cite{Spielman2011Sampling}.
		Consequently, \cref{alg:sparsified_schur_complement} incurs a total \emph{upload} cost of $O(m|\Ccal|\log |\Ccal|)$, where $m$ is the number of robots.
		In \cref{alg:sparsified_laplacian_solver}, robots upload their block vectors $U_\alpha$ in parallel (\cref{alg:sparsified_laplacian_solver:upload})
		and the server broadcasts back the solution $X_c$ (\cref{alg:sparsified_laplacian_solver:apx_reduced_system}).
		Since both $U_\alpha$ and $X_c$ have a dimension of $|\Ccal|$-by-$p$ (where $p = \dim \SOd(d)$ is constant),
		\cref{alg:sparsified_laplacian_solver} uses  $O(m|\Ccal|)$ communication in both \emph{upload} and \emph{download} stages.
	\end{remark}}

\subsection{Collaborative Rotation Averaging}
\label{sec:collaborative_rotation_averaging}

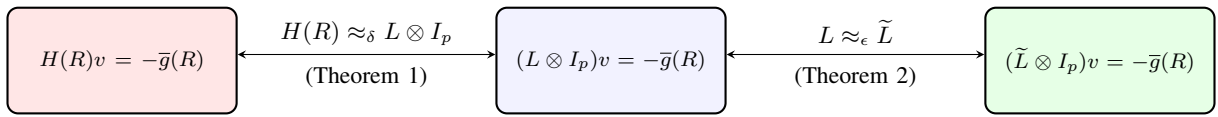
\begin{figure*}
	\centering
	\begin{tikzpicture}[node distance = 6.5cm, auto]
		\node [fblock,fill=red!10!white, thick] (hessian_system) {\footnotesize
			$H(R) v = -\gbar(R)$};
		\node[fblock, fill=blue!5!white, right of=hessian_system, thick] (laplacian_system) {\footnotesize 
			$(L \otimes I_p) v = -\gbar(R)$};
		\node[fblock, fill=green!10!white, right of=laplacian_system, thick] (apx_system) {\footnotesize 
			$(\Ltilde \otimes I_p) v = -\gbar(R)$};
		\draw [stealth-stealth] (hessian_system.3) -- node[above] (delta_apx) {\small $H(R) \approx_\delta L \otimes I_p$}  node[below of=delta_apx,node distance = 0.3cm] {\small (\cref{thm:Hessian_approximation})} (laplacian_system.177);
		\draw [stealth-stealth] (laplacian_system.3) -- node[above] (epsilon_apx) {\small $L \approx_\epsilon \Ltilde$} 
		node[below of=epsilon_apx,node distance = 0.3cm] {\small (\cref{thm:approximation_guarantees})} (apx_system.177);
	\end{tikzpicture}
	\caption
	{
		Intuitions behind the convergence rate in \cref{thm:rotation_averaging_convergence}.
		Recall from \cref{thm:Hessian_approximation} that under bounded measurement noise, the original Newton system (left box) is locally $\delta$-approximated by a linear system specified by a Laplacian $L$ (middle box).
		In addition, in \cref{thm:approximation_guarantees} we have shown that our distributed Laplacian solver approximates $L$ with $\Ltilde$ where $L \approx_\epsilon \Ltilde$ (right box).
		The composition of the two approximation relations thus gives $H(R) \approx_{\delta + \epsilon} (\Ltilde \otimes I_p)$, which intuitively explains why \eqref{thm:rotation_averaging_convergence:rate} depends on a function of $\delta + \epsilon$. 
	}
	\label{fig:rotation_averaging_convergence_illustration}
        \vspace{-0.3cm}
\end{figure*}

In this section, we utilize the Laplacian solver developed in the previous section to design a fast and communication-efficient solver for rotation averaging.
Recall the centralized method in \cref{alg:laplacian_newton_method}, where each iteration solves a Laplacian system 
$LV = B(R)$.
In the multi-robot setting, we can use the solver developed in \cref{sec:collaborative_laplacian_solver} to obtain an approximate solution to this system.
\cref{alg:collaborative_rotation_averaging} shows the pseudocode.   
First, an initial guess $R^0$ is computed (line~\ref{alg:collaborative_rotation_averaging:init}).
Then, at line~\ref{alg:collaborative_rotation_averaging:sparse_schur_complement}, robots first form the approximate Schur complement $\Stilde$ using \SparsifiedSchurComplement (\cref{alg:sparsified_schur_complement}).   
Each iteration consists of three main steps.
At the first step (line~\ref{alg:collaborative_rotation_averaging:B_start}-\ref{alg:collaborative_rotation_averaging:B_end}), robots compute and store the right-hand side $B(R)$. 
Specifically, recall from \cref{rm:distributed_computation_of_laplacian} that the overall $B(R)$ is divided into multiple blocks,
\begin{equation}
	B(R) = 
	\bmat
	B(R)_1^\top & \hdots & B(R)_m^\top & B(R)_c^\top
	\emat^\top.
\end{equation}
In our algorithm, 
each robot $\alpha \in [m]$ computes the block $B(R)_\alpha$ corresponding to its interior variables $\Fcal_\alpha$,
and the server computes the block $B(R)_c$ corresponding to all separators.
At the second step (line~\ref{alg:collaborative_rotation_averaging:solve}),
robots collaboratively solve for the update vector $V^k$ by calling \SparsifiedLaplacianSolver (\cref{alg:sparsified_laplacian_solver}).
Finally, at the last step (line~\ref{alg:collaborative_rotation_averaging:retraction_start}-\ref{alg:collaborative_rotation_averaging:retraction_end}),
we obtain the next iterate using the solutions $V^k$, where robots in parallel update the rotation variables they own.

\begin{algorithm}[t]
	\caption{\textsc{Collaborative Rotation Averaging}}
	\label{alg:collaborative_rotation_averaging}
	\begin{algorithmic}[1]
		\small
		\State Initialize rotation estimates $R^0$.\label{alg:collaborative_rotation_averaging:init}
		\State $\Stilde = \SparsifiedSchurComplement(L, \epsilon)$. \label{alg:collaborative_rotation_averaging:sparse_schur_complement}
		\For{iteration $k = 0, 1, \hdots$}
			\State {\small \textcolor{green!50!black}{// Distributed computation of $B(R^k)$}} \label{alg:collaborative_rotation_averaging:B_start}
			\State Server computes $B(R^k)_c$ that corresponds to all separators.
			\For{each robot $\alpha$ \textbf{in parallel}}
				\State Compute $B(R^k)_\alpha$ that corresponds to interior $\Fcal_\alpha$.
			\EndFor \label{alg:collaborative_rotation_averaging:B_end}
			\State {\small \textcolor{green!50!black}{// Single round of communication to compute $V^k$}}
			\State Solve $V^k = \SparsifiedLaplacianSolver(L, B(R^k), \Stilde).$\label{alg:collaborative_rotation_averaging:solve}
			\State {\small \textcolor{green!50!black}{// Distributed updates of all rotation variables}} \label{alg:collaborative_rotation_averaging:retraction_start}
			\For{each robot $\alpha$ \textbf{in parallel}}
				\State Update iterates by $R^{k+1}_i = \Exp(v^k_i) R_i^k$, for each rotation variable $R_i$ owned by robot $\alpha$.
				\label{alg:collaborative_rotation_averaging:retraction}
			\EndFor \label{alg:collaborative_rotation_averaging:retraction_end}
		\EndFor
	\end{algorithmic}
\end{algorithm}

In the following, we proceed to establish theoretical guarantees for our collaborative rotation averaging algorithm. 
We will show that starting from a suitable initial guess, \cref{alg:collaborative_rotation_averaging} converges to a global minimizer at a \emph{linear} rate.
One might be tempted to state the linear convergence result
on the total space, \ie, $\dist(R^{k+1}, \Rstar) \leq \gamma \dist(R^k, \Rstar)$ where $k$ is the iteration number, $\gamma \in (0,1)$ is a constant, and $\Rstar$ is a global minimizer.
However, it is challenging to prove this statement due to the gauge symmetry of rotation averaging.
The iterates $\{R^k\}$ might converge to a solution $R^\infty$ that is only equivalent to $\Rstar$ \emph{up to a global rotation}, \ie, 
\begin{equation}
	(SR^\infty_1, \hdots, SR^\infty_n) = (\Rstar_1, \hdots, \Rstar_n), \; \text{for some } S \in \SOd(d),
\end{equation}
and as a result $\dist(R^\infty, \Rstar) \neq 0$ in general.
Fortunately, this issue can be resolved using the machinery of Riemannian quotient manifolds.
Instead of measuring the distance on the total space $\dist(R^k, \Rstar)$,
we will compute the distance between the underlying equivalence classes $\dist([R^k], [\Rstar])$.
We note that $\dist([R^k], [\Rstar])$ is well-defined 
since a quotient manifold inherits the Riemannian metric from its total space \cite[Chapter~9]{Boumal20Book}.
Equipped with this distance metric, we are ready to formally state the convergence result for \cref{alg:collaborative_rotation_averaging}.

\begin{theorem}[Convergence rate of \cref{alg:collaborative_rotation_averaging}]
	\label{thm:rotation_averaging_convergence}
	Define $\gamma(x) = 2\sqrt{\kappa_H} c(x)$ 
	where $\kappa_H = L_H / \mu_H$ is the condition number in \cref{cor:Hessian_bound} 
	and $c(\cdot)$ is defined in \eqref{eq:c_epsilon_def}.
	Under the assumptions of \cref{thm:Hessian_approximation}, 
	suppose $\epsilon$ is selected such that $\gamma(\delta + \epsilon) < 1$.
	In addition, suppose at each iteration $k$, the update vector 
	$v^k \in \Real^{pn}$
	is orthogonal to the vertical space, \ie, $v^k \perp \Ncal$.
	Let $\Rstar$ be an optimal solution to \cref{prob:rotation_averaging}.
	There exists $r' > 0$ such that for any $R^0$
	where $\dist([R^0], [\Rstar]) < r'$, 
	\cref{alg:collaborative_rotation_averaging} generates an infinite sequence $\{R^k\}$ 
	where the corresponding sequence of equivalence classes $[R^k]$ converges linearly to $[\Rstar]$.
	Furthermore, the convergence rate factor is,
	\begin{equation}
		\underset{k \to \infty}{\lim \sup}
		\frac{\dist([R^{k+1}], [\Rstar])}{\dist([R^k], [\Rstar])} = 
		\gamma(\delta + \epsilon).
		\label{thm:rotation_averaging_convergence:rate}
	\end{equation}
\end{theorem}

\techreport{{We prove \cref{thm:rotation_averaging_convergence} in Appendix~\ref{sec:linear_convergence_RA}.}}
\mainpaper{{We prove \cref{thm:rotation_averaging_convergence} in \cite[Appendix~IV-B]{Tian2022Sparsification}.}}
\cref{thm:rotation_averaging_convergence} shows that using the distance metric on the quotient manifold, 
\cref{alg:collaborative_rotation_averaging} locally converges to the global minimizer at a linear rate.\footnote{
		In \cref{thm:rotation_averaging_convergence}, the orthogonality assumption $v^k \perp \Ncal$ is needed to ensure that the update vector $v^k$ corresponds to a valid tangent vector on the tangent space of the underlying quotient manifold; 
        \techreport{see Appendix~\ref{sec:linear_convergence_RA} for details.} 
        \mainpaper{see \cite[Appendix~IV-B]{Tian2022Sparsification} for details.} 
        One can satisfy this assumption by projecting $v^k$ to the horizontal space, which requires a single round of communication between the server and robots. However, in practice, we find that this has have negligible impact on the iterates and thus skip this step in our implementation.}
\cref{fig:rotation_averaging_convergence_illustration} provides intuitions behind the convergence rate in \eqref{thm:rotation_averaging_convergence:rate}.
Recall that $\delta$ appears in \cref{thm:Hessian_approximation} where we show $H(R) \approx_\delta (L \otimes I_p)$ under bounded measurement noise.
On the other hand, $\epsilon$ is the parameter for spectral sparsification and is controlled by the user.
In \cref{thm:approximation_guarantees}, we showed that our methods transform the input Laplacian $L$ into an approximation $\Ltilde$ such that $L \approx_\epsilon \Ltilde$.
The composition of the two approximation relations thus gives $H(R) \approx_{\delta + \epsilon} (\Ltilde \otimes I_p)$, which intuitively explains why the convergence rate depends on a function of $\delta + \epsilon$.
{Lastly, we note that while our theoretical convergence guarantees require $\gamma(\delta + \epsilon) < 1$, 
our experiments (\cref{sec:experiments}) show that \cref{alg:collaborative_rotation_averaging} is not sensitive to the choice of $\epsilon$ and converges under a wide range of parameter settings.}

{
	\begin{remark}[Communication efficiency of \cref{alg:collaborative_rotation_averaging}]
		In \cref{alg:collaborative_rotation_averaging}, note that only a single call to  \SparsifiedSchurComplement (\cref{alg:sparsified_schur_complement}) is needed,
		which incurs a total upload of $O(m|\Ccal| \log|\Ccal|)$; see \cref{rm:comm_efficiency_laplacian_solver}.
		In each iteration, a single call to \SparsifiedLaplacianSolver (\cref{alg:sparsified_laplacian_solver}) is made, which requires a single round of upload and download.
		Furthermore, by \cref{rm:comm_efficiency_laplacian_solver}, both upload and download costs are bounded by $O(m|\Ccal|)$.
		Therefore, after $K > 0$ iterations, 
		\cref{alg:collaborative_rotation_averaging} uses a total upload of $O(m|\Ccal| \log|\Ccal| + mK |\Ccal| )$
		and a total download of $O(mK |\Ccal|)$.
		In particular, the terms that involve the number of iterations $K$ scales \emph{linearly} with the number of separators $|\Ccal|$, which makes the algorithm very communication-efficient. 
		\label{rm:comm_requirements_rotation_averaging}
	\end{remark}
}

\subsection{Collaborative Translation Estimation}
\label{sec:collaborative_translation_recovery}

\begin{algorithm}[t]
	\caption{\textsc{Collaborative Translation Estimation}}
	\label{alg:collaborative_translation_recovery}
	\begin{algorithmic}[1]
		\small
		\State Initialize translation estimates $M^0_t = 0_{n \times d}$.
		\State $\Stilde = \SparsifiedSchurComplement(L, \epsilon)$.
		\For{iteration $k = 0, 1, \hdots$}
			\State {\small \textcolor{green!50!black}{// Distributed computation of $E^k$}} \label{alg:collaborative_translation_recovery:B_start}
			\State Server computes $E^k_{c}$ that corresponds to all separators.
			\For{each robot $\alpha$ \textbf{in parallel}}
				\State Compute $E^k_\alpha$ that corresponds to interior $\Fcal_\alpha$.
			\EndFor \label{alg:collaborative_translation_recovery:B_end}
			\State {\small \textcolor{green!50!black}{// Single round of communication to compute $D^k$}}
			\State Solve $D^k = \SparsifiedLaplacianSolver(L, E^k, \Stilde).$ \label{alg:collaborative_translation_recovery:solve}
			\State {\small \textcolor{green!50!black}{// Distributed updates of all translations: $M_t^{k+1} = M_t^k + D^k$}}\label{alg:collaborative_translation_recovery:retraction_start}
			\For{each robot $\alpha$ \textbf{in parallel}}
				\State Update iterates by $t^{k+1}_i = t_i^k + (D^{k}_{[i,:]})^\top$ for each translation variable $t_i$ owned by robot $\alpha$.
				\label{alg:collaborative_translation_recovery:retraction}
			\EndFor \label{alg:collaborative_translation_recovery:retraction_end}
		\EndFor
	\end{algorithmic}
\end{algorithm}

Similar to rotation averaging, we can develop a fast and communication-efficient method to solve translation estimation, which is equivalent to the Laplacian system \eqref{eq:translation_laplacian_system_matrix_form} as shown in \cref{sec:laplacian_systems:translation}.
Specifically, we employ our collaborative Laplacian solver (\cref{sec:collaborative_laplacian_solver}) in an \emph{iterative refinement} framework.
Let $M^k_t \in \Real^{n \times d}$ be our estimate for the translation variables at iteration $k$
(in practice $M_t^0$ can simply be initialized at zero).
We seek a correction $D^k$ to $M^k_t$ 
by solving the \emph{residual system} corresponding to \eqref{eq:translation_laplacian_system_matrix_form}:
\begin{equation}
	L (M_t^k + D^k)  = B_t  \iff L D^k = B_t - L M_t^k \triangleq E^k.
	\label{eq:translation_recovery_residual_system}
\end{equation}
Observing that the system on the right-hand side of \eqref{eq:translation_recovery_residual_system} is another Laplacian system in 
$L \equiv L(G; \tau)$, we can deploy our Laplacian solver to find an approximate solution $D^k$.
\cref{alg:collaborative_translation_recovery} shows the pseudocode, which shares many similarities with the proposed collaborative rotation averaging method \cref{alg:collaborative_rotation_averaging}.
In particular, the computation of the right-hand side $E^k$ (line~\ref{alg:collaborative_translation_recovery:B_start}-\ref{alg:collaborative_translation_recovery:B_end})
and the update step (line~\ref{alg:collaborative_translation_recovery:retraction_start}-\ref{alg:collaborative_translation_recovery:retraction_end}) are 
performed in a distributed fashion.
The two methods also share the same communication complexity; see \cref{rm:comm_requirements_rotation_averaging}.
The following theorem states the theoretical guarantees for \cref{alg:collaborative_translation_recovery}.

\begin{theorem}[Convergence rate of \cref{alg:collaborative_translation_recovery}]
	\label{thm:translation_recovery_convergence}
	Suppose $\epsilon$ is selected such that the constant $c(\epsilon)$ defined in \eqref{eq:c_epsilon_def} satisfies $c(\epsilon) < 1$.
	Let $M_t^\star$ be an optimal solution to \cref{prob:translation_recovery}
	and let $M_t^k$ denote the solution computed by \cref{alg:collaborative_translation_recovery} at iteration $k \geq 1$.
	It holds that, 
	\begin{equation}
		\norm{M_t^k - M_t^\star}_L \leq c(\epsilon)^k \norm{M_t^\star}_L,
		\label{eq:translation_recovery_convergence_rate}
	\end{equation}
	where $L \equiv L(G; \tau)$.
\end{theorem}

\techreport{{We prove \cref{thm:translation_recovery_convergence} in Appendix~\ref{sec:linear_convergence_translation}.}}
\mainpaper{{We prove \cref{thm:translation_recovery_convergence} in \cite[Appendix~IV-C]{Tian2022Sparsification}.}}
\cref{thm:translation_recovery_convergence} is simpler compared to its counterpart for rotation averaging (\cref{thm:rotation_averaging_convergence}).
The convergence rate \eqref{eq:translation_recovery_convergence_rate} only depends on the sparsification parameter $\epsilon$.
Furthermore, since the translation estimation problem is convex, the convergence guarantee is \emph{global} and holds for any initial guess.\techreport{\footnote{
	In \eqref{eq:translation_recovery_convergence_rate}, 
	the use of $\norm{\cdot}_L$ naturally accounts for the global translation symmetry of \cref{prob:translation_recovery} (see \cref{sec:laplacian_systems:translation}).
	Specifically, since $\ker(L) = \image(1_n)$, 
	$||M_t^k - M_t^\star||_L$ disregards any difference between $M_t^k$ and $M_t^\star$ that corresponds to a global translation.
}}
While \cref{thm:translation_recovery_convergence} requires $c(\epsilon) < 1$, 
our experiments show that \cref{alg:collaborative_translation_recovery} is not sensitive to the choice of sparsification parameter $\epsilon$ and converges under a wide range of parameter settings.

{

\subsection{Extension to Outlier-Robust Optimization}
\label{sec:gnc}

So far, we have considered estimation using the standard least squares cost function, which is sensitive to \emph{outlier measurements} that might arise in practice (\eg, due to incorrect loop closures in multi-robot SLAM).
In this section, we present an extension to \emph{outlier-robust} optimization by embedding the developed solvers in the graduated non-convexity (GNC) framework \cite{Yang2020Gnc,Black1996Gnc}.
\edit{We select GNC for its good performance as reported in recent works \cite{Yang2020Gnc,Tian21KimeraMulti}. However, similar robust optimization frameworks such as iterative reweighted least squares \cite{Chatterjee18pami-rotationAveraging} can also be used.}
Consider robust estimation using the truncated least squares (TLS) cost:\footnote{Other robust cost functions, such as the Geman McClure function, can also be used in the same framework; see \cite{Yang2020Gnc}.}
\begin{equation}
	\label{eq:robust_estimation}
	\underset{x \in \Xcal}{\minimize}
	\quad \sum_{(i,j) \in \Ecal} \TLS(e_{ij} (x)).
\end{equation}
In \eqref{eq:robust_estimation}, $x \in \Xcal$ is the model to be estimated, and $e_{ij} (x)$ is the measurement error associated with edge $(i,j) \in \Ecal$ in the measurement graph.
For the robust extension of rotation averaging (\cref{prob:rotation_averaging}), we define $x = (R_1, \hdots, R_n) \in \SOd(d)^n$, and $e_{ij} (x) = \sqrt{\kappa_{ij}/2} \dist(R_i \Rtilde_{ij}, R_j)$ where $\dist(\cdot, \cdot)$ is the geodesic or the chordal distance.
For the robust extension of translation estimation (\cref{prob:translation_recovery}), we define $x = (t_1, \hdots, t_n) \in \Real^{d \times n}$ and $e_{ij}(x) = \sqrt{\tau_{ij}/2} \norm{t_j - t_i - \that_{ij}}$.
Notice that $e_{ij}(x)$ is simply the square root of a single cost term in \cref{prob:rotation_averaging} or \cref{prob:translation_recovery}.
Finally, $\TLS(e) \triangleq \min(e^2, \ebar^2)$ denotes the TLS cost function, where $\ebar$ is a constant threshold that specifies the maximum acceptable error of inlier measurements.
Intuitively, the TLS cost function achieves robustness by eliminating the impact of any outliers with error larger than $\ebar$.

To mitigate the non-convexity introduced by robust cost functions, GNC solves \eqref{eq:robust_estimation} by optimizing a sequence of easier (\ie, less non-convex) surrogate functions $\TLS_\mu$ that gradually converges to the original, highly non-convex cost function $\TLS$.
Here, $\mu$ is the control parameter and for the TLS function, it satisfies that (i) $\TLS_\mu$ is convex for $\mu \to 0$, and (ii) $\TLS_\mu$ recovers $\TLS$ for $\mu \to +\infty$; see \cite[Example~2]{Yang2020Gnc}.
In practice, we initialize by setting $\mu \approx 0$, and gradually increase $\mu$ as optimization progresses.
Furthermore, leveraging the Black-Rangarajan duality \cite{Black1996Gnc}, each surrogate problem can be formulated as follows,
 \begin{equation}
 	\label{eq:gnc_surrogate_func}
 	\underset{x \in \Xcal, \wgnc_{ij} \in [0,1]}{\minimize}
 	\quad \sum_{(i,j) \in \Ecal} \left[\wgnc_{ij} e^2_{ij} (x) + \Phi_\mu (\wgnc_{ij})\right].
 \end{equation}
In \eqref{eq:gnc_surrogate_func}, $\wgnc_{ij}$ is a mutable weight attached to the measurement error $e_{ij}$,
and $\Phi_\mu$ acts as a regularization term on the weight whose expression depends on the control parameter $\mu$.

GNC leverages \eqref{eq:gnc_surrogate_func} by performing alternating updates on the model $x$ and the weights $\wgnc_{ij}$, while simultaneously updating the control parameter $\mu$.
Specifically, each GNC outer iteration consists of three steps:
\begin{enumerate}[leftmargin=0.3cm]
	\item \textbf{Variable update}: optimize the surrogate problem \eqref{eq:gnc_surrogate_func} with respect to $x$, under fixed weights $\wgnc_{ij}$. Notice that this amounts to a standard weighted least squares problem,
	\begin{equation}
		\label{eq:gnc_variable_update}
		\underset{x \in \Xcal}{\minimize}
		\quad \sum_{(i,j) \in \Ecal} \wgnc_{ij} e^2_{ij} (x).
	\end{equation}
	\item \textbf{Weight update}: optimize the surrogate problem \eqref{eq:gnc_surrogate_func} with respect to all $\wgnc_{ij}$, under fixed model $x$. For TLS, the resulting $\wgnc_{ij}$ has a \emph{closed-form} solution,
	\begin{equation}
		\wgnc_{ij} \leftarrow \begin{cases}
			0, 
			& \text{ if } e^2_{ij} \in \left[\frac{\mu+1}{\mu}\ebar^2,  +\infty \right], \\
			\frac{\ebar}{e_{ij}} \sqrt{\mu(\mu+1)} - \mu, 
			& \text{ if } e^2_{ij} \in \big [ \frac{\mu}{\mu + 1}\ebar^2, \frac{\mu+1}{\mu}\ebar^2 \big], \\
			1, 
			& \text{ if } e^2_{ij} \in \big [0, \frac{\mu}{\mu+1} \ebar^2],
		\end{cases}
		\label{eq:gnc_weight_update}
	\end{equation}
	where $e_{ij} \equiv e_{ij}(x)$ is the current measurement error.
	\item \textbf{Parameter update}: update control parameter $\mu$ via $\mu \leftarrow 1.4 \mu$ (recommended in \cite[Remark~5]{Yang2020Gnc}), and move on to the next surrogate problem.
\end{enumerate}
Initially, all measurement weights are initialized at one.

\begin{algorithm}[t]
	\caption{{Outlier-robust rotation averaging with GNC}}
	\label{alg:gnc}
	\begin{algorithmic}[1]
		\State Initialize control parameter $\mu$ and measurement weights by setting $w_{ij} = 1$ for all measurements $(i,j) \in \Ecal$.
		\While{not converged}
		\State \textbf{Variable update}:
			under fixed weights, solve the weighted rotation averaging problem by executing \cref{alg:collaborative_rotation_averaging} under the server-client architecture. \label{alg:gnc:var_update}
		\State \textbf{Weight update}:
			in parallel, server computes \eqref{eq:gnc_weight_update} for all inter-robot measurements $\Ecal_c$, and each robot $\alpha$ computes \eqref{eq:gnc_weight_update} for its local measurements $\Ecal_\alpha$. \label{alg:gnc:weight_update}
		\State \textbf{Parameter update}:
		in parallel, server and all robots updates the control parameter $\mu$.\label{alg:gnc:mu_update}
		\EndWhile
	\end{algorithmic}
\end{algorithm}

Next, we show that our algorithms developed in this work can be used within GNC to perform outlier-robust optimization.
\cref{alg:gnc} shows the pseudocode for robust rotation averaging (the case for translation estimation is analogous).
The main observation is that, in the context of robust rotation averaging and translation estimation, 
the weighted least squares problems \eqref{eq:gnc_variable_update} solved during the \textbf{variable update} step 
have identical forms as \cref{prob:rotation_averaging,prob:translation_recovery}.
The only difference is that each measurement is now discounted by the GNC weight $\wgnc_{ij}$, as shown in \eqref{eq:gnc_variable_update}.
Therefore, we can use \cref{alg:collaborative_rotation_averaging} to perform the variable update for rotation averaging (\cref{alg:gnc:var_update}),
and \cref{alg:collaborative_translation_recovery} for translation estimation.
Furthermore, the \textbf{weight update} step can also be executed under the server-client architecture,
where each robot $\alpha$ computes \eqref{eq:gnc_weight_update} for its local measurements $\Ecal_\alpha$,
and the server handles the inter-robot measurements $\Ecal_c$; see \cref{alg:gnc:weight_update}.
Lastly, the server and all robots can in parallel perform the \textbf{parameter update} step by updating their local copies of the control parameter $\mu$ (\cref{alg:gnc:mu_update}).

\begin{remark}[Implementation details of GNC]
	\label{rm:gnc_implementation}
We discuss several implementation details for GNC.
\begin{itemize}[leftmargin=0.3cm]
	\item \emph{Initialization}. Prior works (\eg, \cite{Tian21KimeraMulti}) have observed that using an outlier-free initial guess when solving the variable update step is critical to ensure good performance.
	For multi-robot SLAM, we adopt the method described in \cite[Section~V-B]{Tian21KimeraMulti} that aligns each robot's odometry 
	in the global reference frame by solving a robust single pose averaging problem.
	Notably, this method does not require iterative communication and hence is very efficient.
	
	\item \emph{Known inliers}.  In many cases, a subset of measurements $\Ecal_\text{in} \subseteq \Ecal$ are known to be inliers.
	For instance, $\Ecal_\text{in}$ may contain robots' odometry measurements. 
	In our implementation, we use the standard least squares cost for $\Ecal_\text{in}$ and only apply GNC on the remaining measurements.
	
	\item \emph{Approximate optimization}.  Recall that each outer iteration of GNC invokes \cref{alg:collaborative_rotation_averaging} or \cref{alg:collaborative_translation_recovery} 
	to perform the variable update step. 
	Thus, when the number of outer iterations is large, the resulting optimization might become expensive in terms of both runtime and communication.
	However, in practice, we observe that GNC only requires a few outer iterations before the resulting estimates stabilize (see \cref{sec:experiments:jackal_dataset}). 
	This suggests that instead of running GNC to full convergence (\ie, fully classifying each measurement as either inlier or outlier),
	we can perform approximate optimization by limiting the number of outer iterations while still achieving robust estimation.
	In our experiments, we set the maximum number of GNC outer iterations to 20. 
\end{itemize}
\end{remark}

We conclude this subsection by noting that the linear convergence results (\cref{thm:rotation_averaging_convergence,thm:translation_recovery_convergence}) we prove in this paper only hold for the outlier-free case.
Extending the linear convergence to the case with outliers is challenging because GNC (and the similar method of iterative reweighted least square) is itself a heuristic.
Nevertheless, our experiments demonstrate that in practice, the proposed outlier-robust extension is very effective and produces accurate solutions on real-world SLAM and SfM problems contaminated by outlier measurements.

}
\section{Experimental Results}
\label{sec:experiments}

{In this section, we extensively evaluate our proposed methods and demonstrate their fast convergence and communication efficiency.
In addition, we show that the combination of our rotation estimation and translation estimation algorithms can be used for accurate PGO initialization.
\cref{sec:experiments:basic,sec:experiments:pgo_datasets} show evaluations using synthetic and benchmark datasets.
Then, \cref{sec:experiments:jackal_dataset} and \cref{sec:experiments:sfm} demonstrate \emph{outlier-robust} estimation using our approach
on real-world collaborative SLAM and SfM problems.
Lastly, \cref{sec:experiments:discussion} provides additional discussions on the performance of our approach in real-world problem instances.}
\edit{All proposed algorithms (including the GNC extension in \cref{sec:gnc}) are implemented in MATLAB.}
\edit{Some experiments use GTSAM \cite{gtsam} and the Theia SfM library \cite{TheiaSfM} for comparison, where we run their original implementations in C++.}
\edit{All experiments are performed on  a computer with an Intel i7-7700K CPU and 16 GB RAM,
and communication is simulated in memory in MATLAB.}

\begin{figure*}[t]
	\centering
	\begin{subfigure}[t]{0.23\linewidth}
		\includegraphics[width=\textwidth, trim=0 0 50 0, clip]{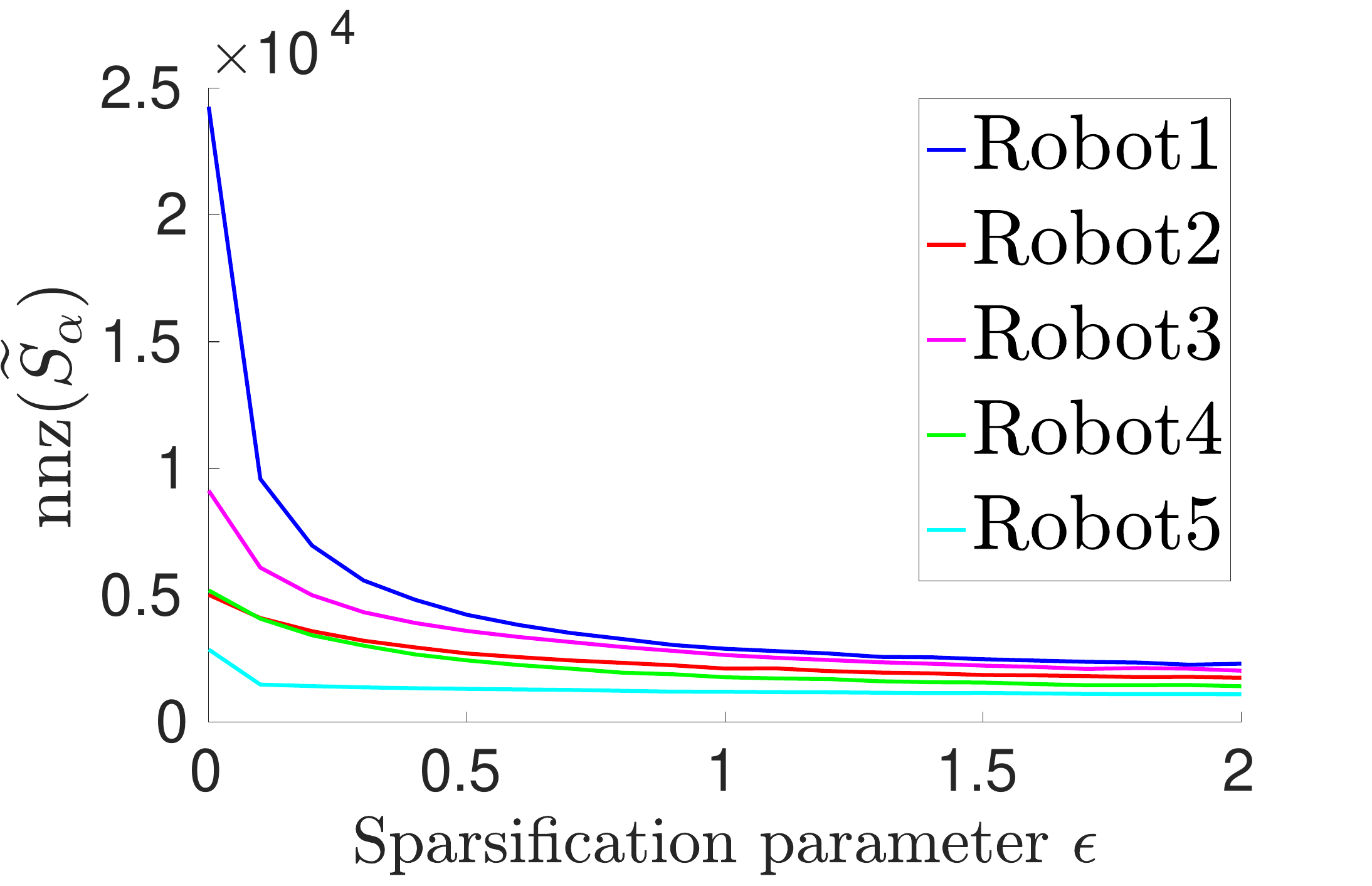}
		\caption{Sparsity of $\Stilde_\alpha$}
		\label{fig:cubicle_eval:nnzS}
	\end{subfigure}
	~
	\begin{subfigure}[t]{0.23\linewidth}
		\includegraphics[width=\textwidth, trim=0 0 80 50, clip]{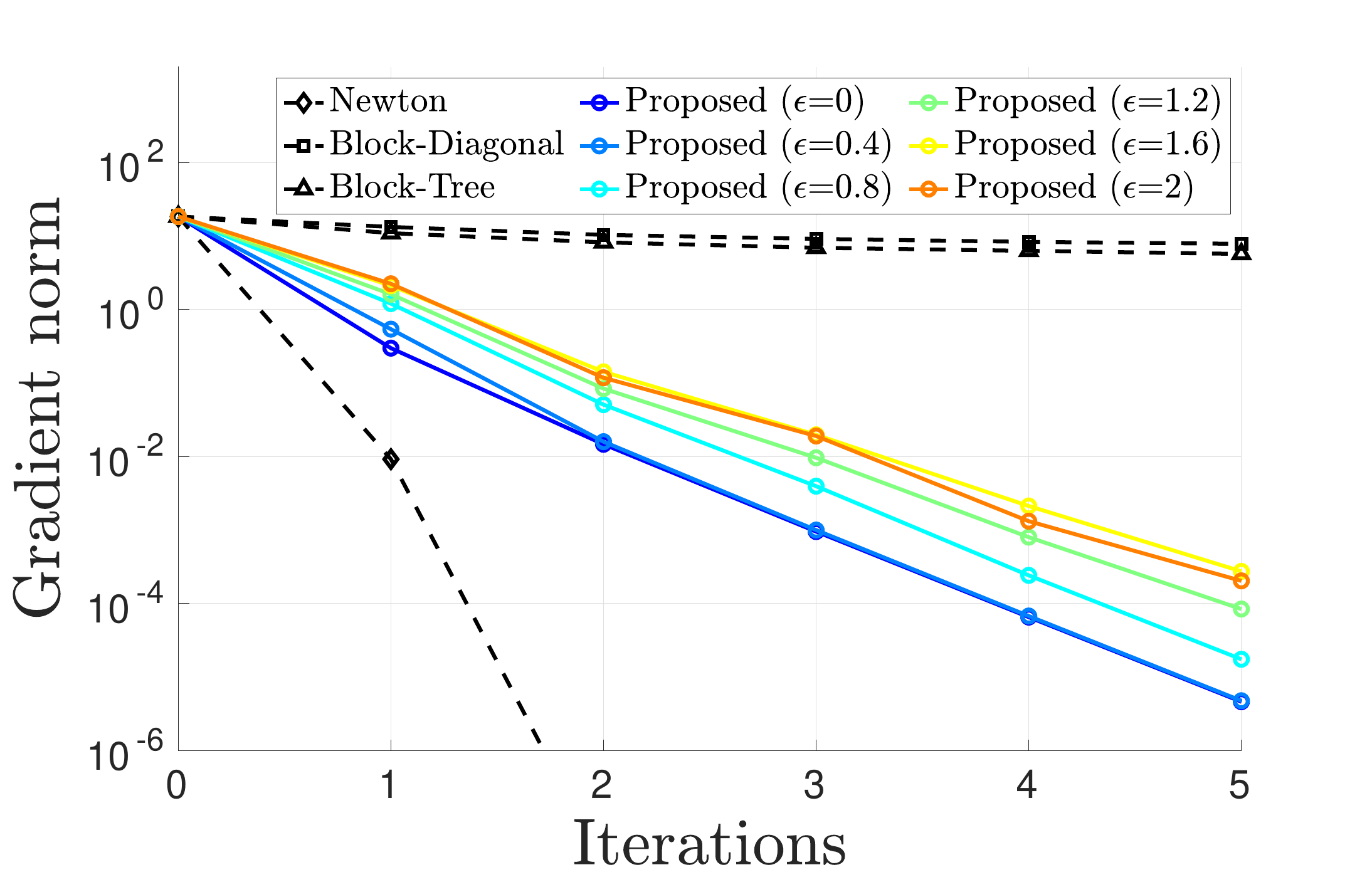}
		\caption{Gradient norm vs. iterations}
		\label{fig:cubicle_eval:iteration}
	\end{subfigure}
	~
	\begin{subfigure}[t]{0.23\linewidth}
		\includegraphics[width=\textwidth, trim=0 0 80 50, clip]{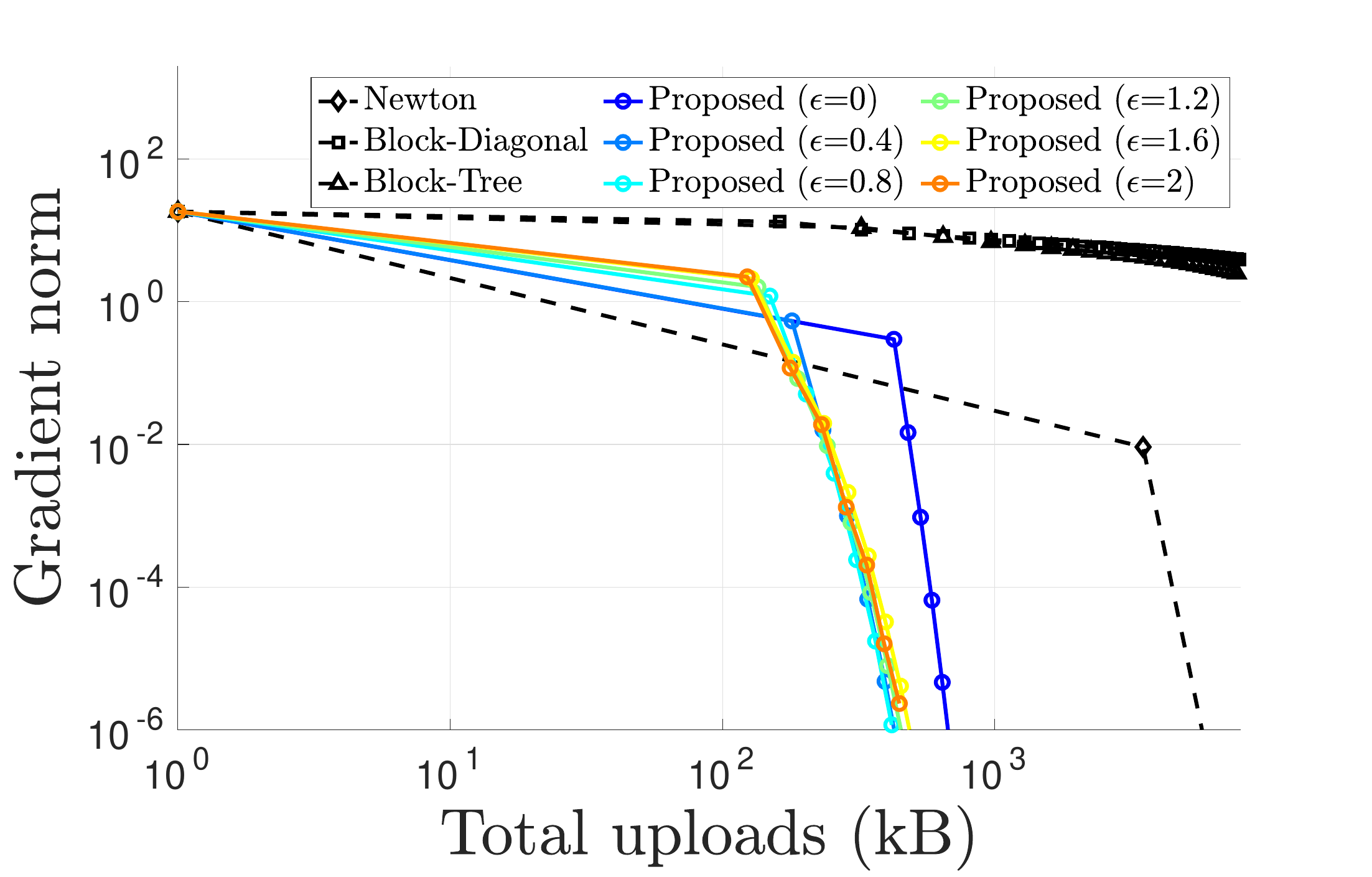}
		\caption{Gradient norm vs. uploads}
		\label{fig:cubicle_eval:communication}
	\end{subfigure}	
	~
	\begin{subfigure}[t]{0.23\linewidth}
		\includegraphics[width=\textwidth, trim=0 0 80 50, clip]{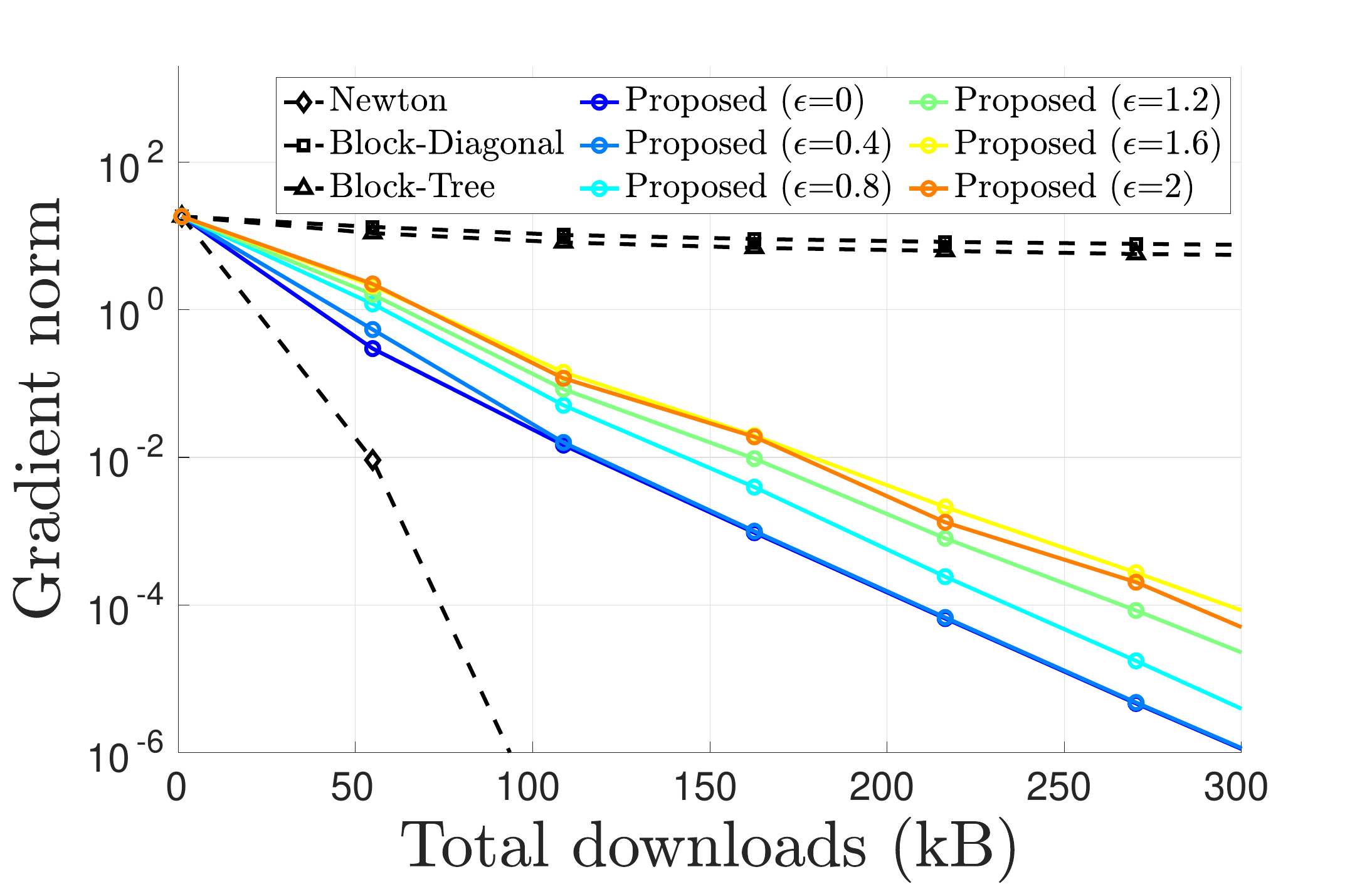}
		\caption{{Gradient norm vs. downloads}}
		\label{fig:cubicle_eval:download}
	\end{subfigure}		
	\caption{
			Evaluation of \cref{alg:collaborative_rotation_averaging} on the 5-robot rotation averaging problem from the \dataset{Cubicle} dataset.
			(a) For each robot $\alpha$, we show the number of nonzero entries (nnz) in its sparsified matrix $\Stilde_\alpha$ as a function sparsification parameter $\epsilon$.
			(b) Evolution of Riemannian gradient norm as a function of iterations.
			(c) Evolution of Riemannian gradient norm as a function of total uploads. 
			{(d) Evolution of Riemannian gradient norm as a function of total downloads.}} 
	\label{fig:cubicle_eval}
        \vspace{-0.5cm}
\end{figure*}

{\myParagraph{Performance Metrics}
	In the experiments, we use the following metrics to evaluate algorithm performance.
	First, we compute the evolution of \emph{gradient norm} that measures the rate of convergence.
	Second, to quantify communication efficiency, we record the \emph{total communication} used by an algorithm.
	For the server-client architecture, communication is reported for both the \emph{upload} and \emph{download} stages. 
	When evaluating the proposed PGO initialization method, we also compute the relative \emph{optimality gap} in the cost function,
	defined as $(f_\text{init} - f_\text{opt}) / f_\text{opt}$, where $f_\text{init}$ and $f_\text{opt}$ denote the cost achieved by our initialization and the global minimizer, respectively.
	Lastly, we also report the \emph{solution distance} to the global minimizer and optionally to the ground truth (the latter is only available in our synthetic experiments).
	Specifically, for rotation estimation, we compute the distance between our solution $\Rhat \in \SOd(d)^n$ and the reference $\Rref \in \SOd(d)^n$ (either global minimizer or ground truth)
	using the orbit distance:
	\begin{equation}
		\RMSE(\Rhat, \Rref) \triangleq \min_{S \in \SOd(d)} \sqrt{\frac{1}{n} \sum_{i=1}^n \norm{S \Rhat_i - \Rref_i}^2_F}.
		\label{eq:rotation_RMSE}
	\end{equation}
	Intuitively, \eqref{eq:rotation_RMSE} computes the root-mean-square error (RMSE) between two sets of rotations after alignment by a global rotation. 
	The optimal alignment $S$ in \eqref{eq:rotation_RMSE} has a closed-form expression; see \cite[Appendix~C.1]{Rosen19IJRR}.
	Similarly, for translations, we report the RMSE between our solution and the reference after a global alignment.}

{\subsection{Evaluation of Estimation Accuracy and Communication Efficiency}\label{sec:experiments:basic}}

{In this section, we evaluate the estimation accuracy and communication efficiency of the proposed methods under varying problem setups and algorithm parameters.}
Unless otherwise mentioned, we initialize \cref{alg:collaborative_rotation_averaging} using the distributed chordal initialization approach in \cite{Choudhary17IJRR}, where the number of iterations is limited to 50.
Our experiments mainly consider rotation averaging problems under the chordal distance metric.
\techreport{{Appendix~\ref{sec:additional_experiments} provides additional results using the geodesic distance.}}
\mainpaper{{We provide additional results using the geodesic distance in \cite[Appendix~VI]{Tian2022Sparsification}.}}

\begin{figure*}[t]
	\centering
	\begin{subfigure}[t]{0.23\linewidth}
		\includegraphics[width=\textwidth, trim=100 100 100 50, clip]{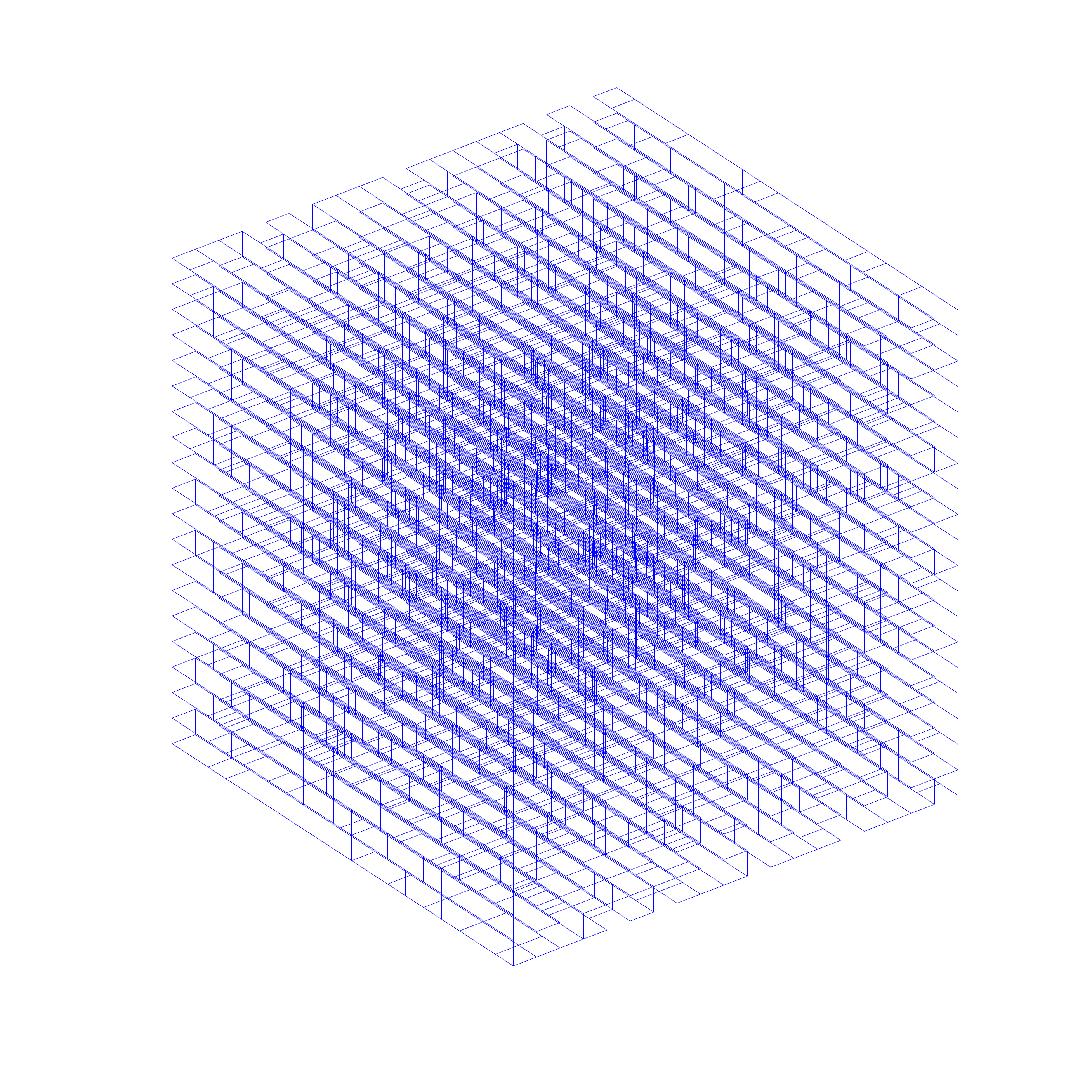}
		\caption{Synthetic rotation averaging}
		\label{fig:num_robots:dataset}
	\end{subfigure}
	~
	\begin{subfigure}[t]{0.23\linewidth}
		\includegraphics[width=\textwidth, trim=0 0 0 0, clip]{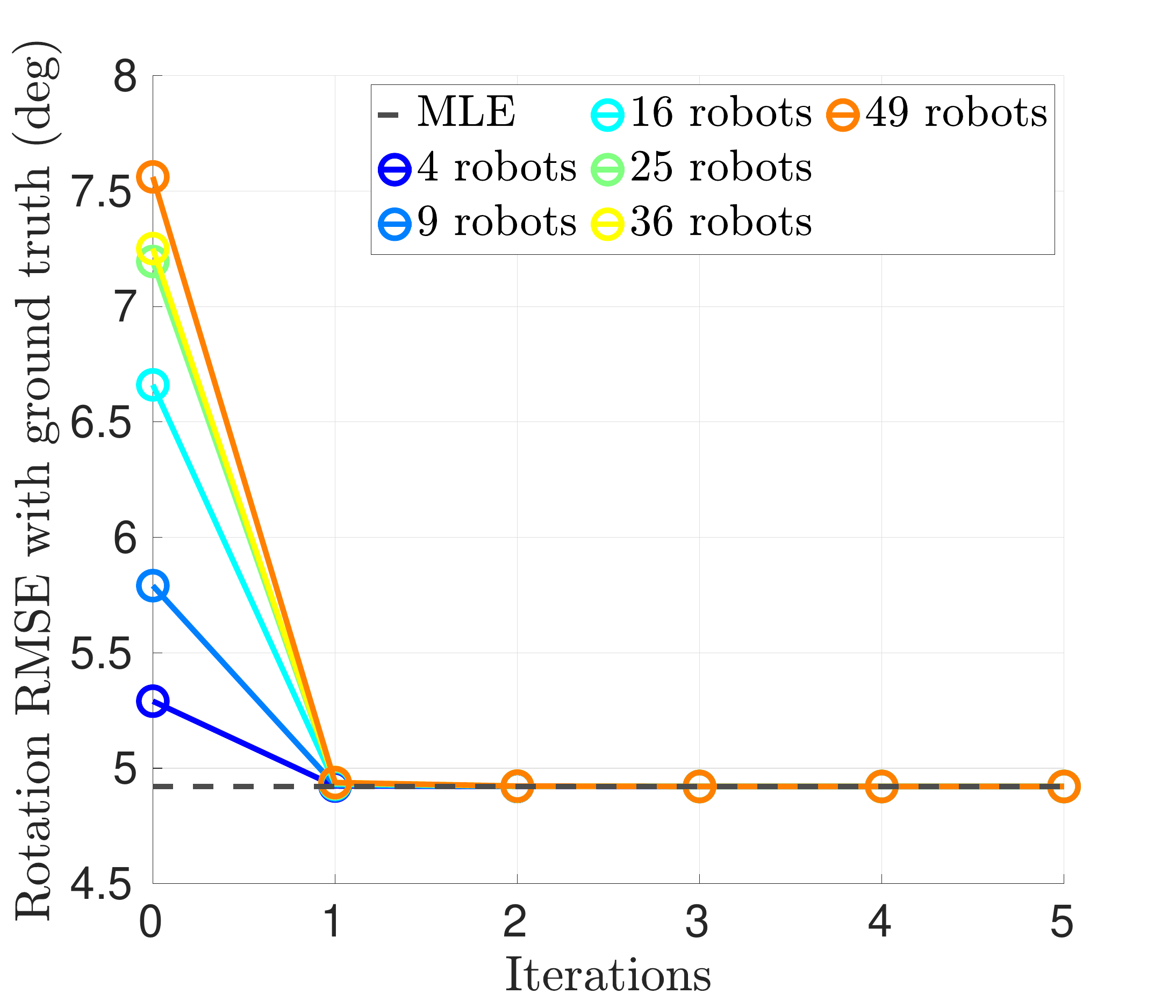}
		\caption{RMSE vs. iterations}
		\label{fig:num_robots:rmse_with_gt}
	\end{subfigure}
	~
	\begin{subfigure}[t]{0.23\linewidth}
		\includegraphics[width=\textwidth, trim=0 0 0 0, clip]{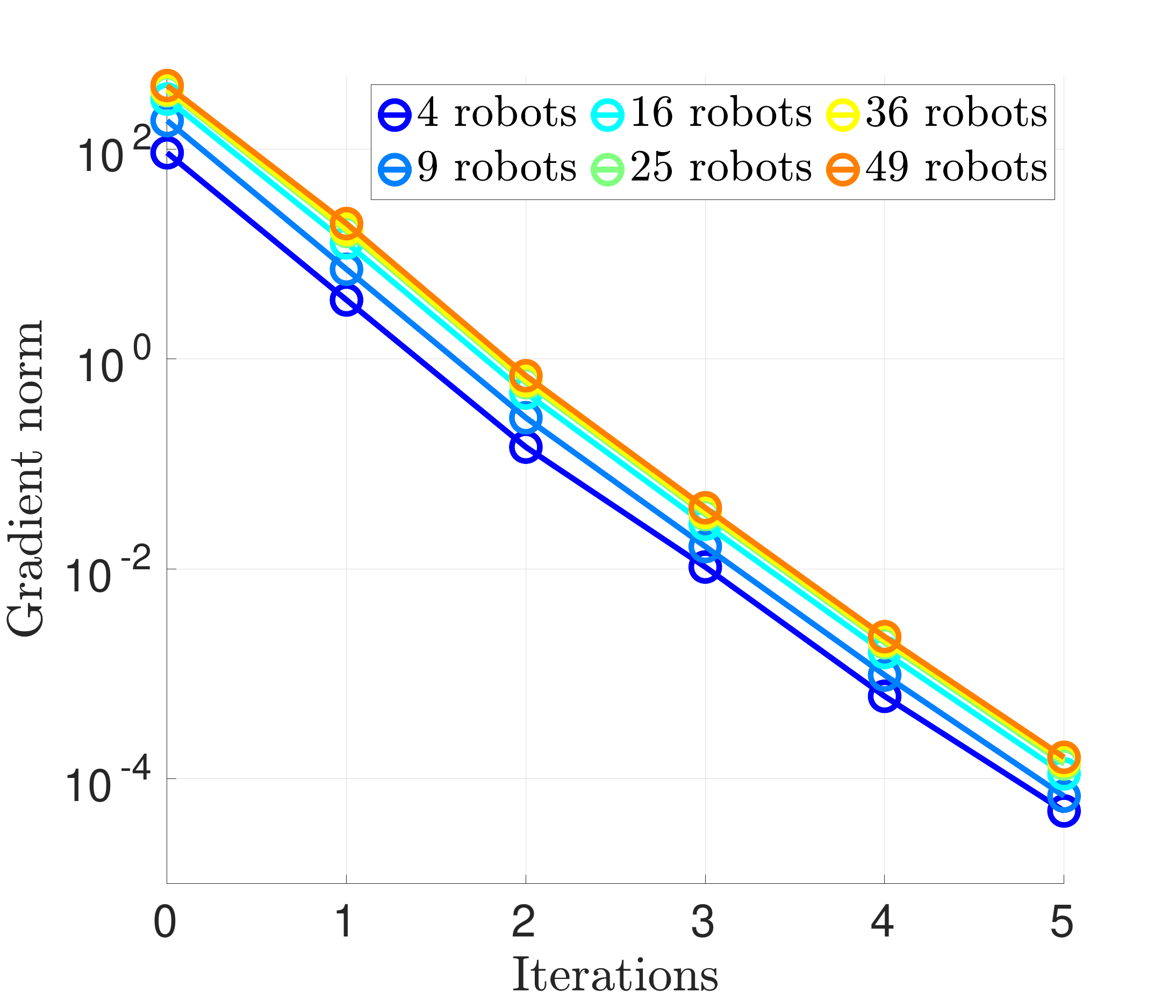}
		\caption{Gradient norm vs. iterations}
		\label{fig:num_robots:iterations}
	\end{subfigure}
	~
	\begin{subfigure}[t]{0.23\linewidth}
		\includegraphics[width=\textwidth, trim=0 0 0 0, clip]{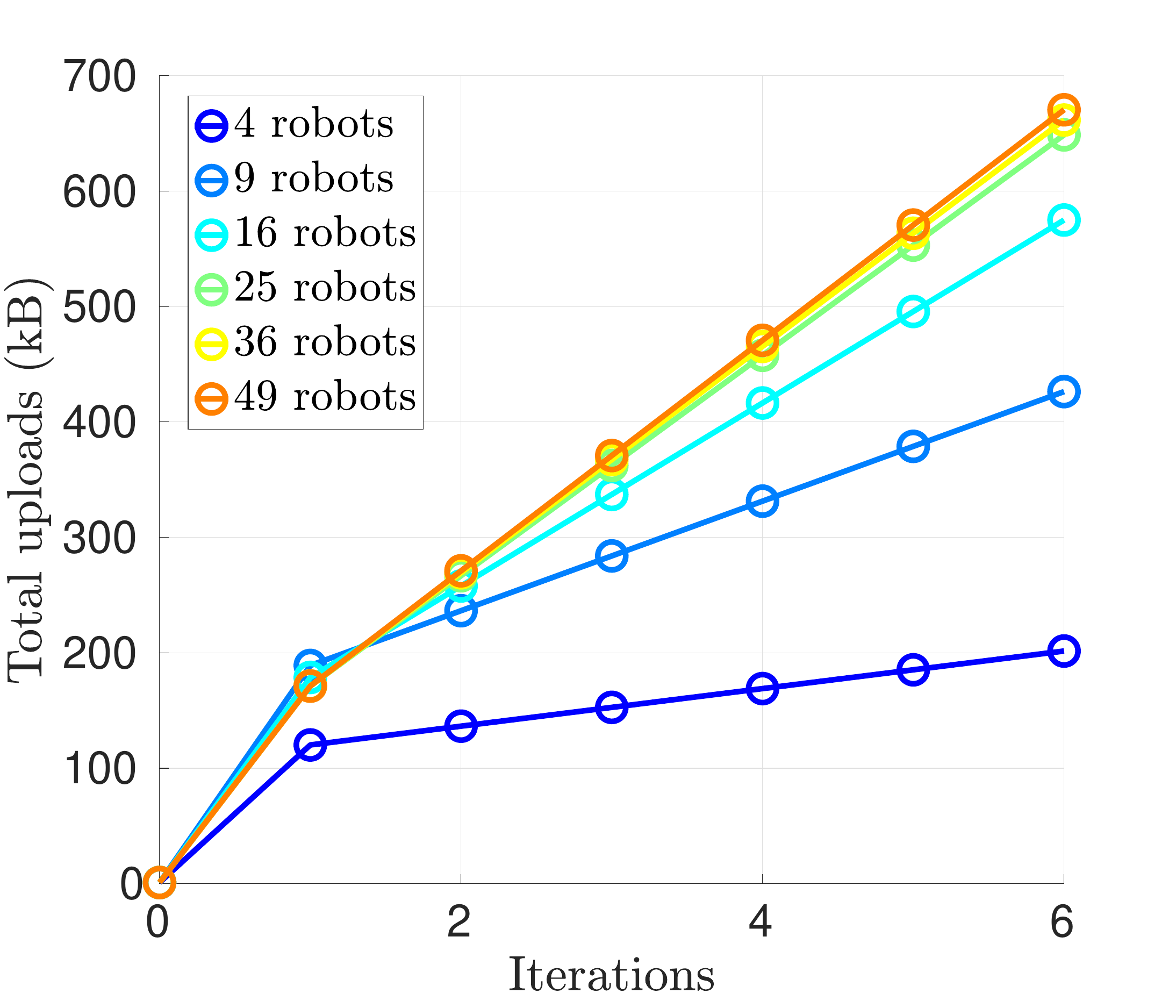}
		\caption{Uploads vs. iterations}
		\label{fig:num_robots:uploads}
	\end{subfigure}	
	\caption{
		{Scalability of \cref{alg:collaborative_rotation_averaging} as the number of robots increases.
			(a) Synthetic chordal rotation averaging problem with 8000 total rotation variables arranged in a 3D grid.
			Each edge indicates a relative rotation measurement corrupted by Langevin noise.
			(b) Evolution of RMSE (in degree) with respect to ground truth rotations.
			(c) Evolution of Riemannian gradient norm as a function of iterations.
			(d) Evolution of total uploads as a function of iterations. }} 
	\label{fig:num_robots}
        \vspace{-0.5cm}
\end{figure*}

{\myParagraph{Impact of Spectral Sparsification on Convergence and Communication}
First, we evaluate the impact of spectral sparsification on convergence rate and communication efficiency.}
{We start by evaluating the proposed collaborative rotation averaging solver (\cref{alg:collaborative_rotation_averaging}), by simulating a 5-robot problem using the \dataset{Cubicle} dataset.}
\techreport{{In Appendix~\ref{sec:additional_experiments}, we present similar analysis for translation estimation.}}
\mainpaper{{In \cite[Appendix~VI]{Tian2022Sparsification}, we present similar analysis for translation estimation.}}
Recall that \cref{alg:collaborative_rotation_averaging} calls the \SparsifiedSchurComplement procedure (\cref{alg:sparsified_schur_complement}), which requires each robot $\alpha$ to transmit its sparsified matrix $\Stilde_\alpha$.
\cref{fig:cubicle_eval:nnzS} shows the number of nonzero entries in $\Stilde_\alpha$ as a function of the sparsification parameter $\epsilon$.
Note that when $\epsilon = 0$, sparsification is effectively skipped and each robot transmits its exact $S_\alpha$ matrix that is potentially large and dense.
In \cref{fig:cubicle_eval:nnzS}, this is reflected on robot 1 (blue curve) whose exact $S_\alpha$ matrix has more than $2 \times 10^4$ nonzero entries and hence is expensive to transmit.
However, spectral sparsification significantly reduces the density of the matrix and hence improves communication efficiency.
In particular, for robot 1, applying sparsification with $\epsilon=2$ creates a sparse $\Stilde_\alpha$ with 2300 nonzero entries, which is much sparser than the original $S_\alpha$.

Next, we evaluate the convergence rate and communication efficiency of \cref{alg:collaborative_rotation_averaging} with varying sparsification parameter $\epsilon$.
We introduce 
three baseline methods for the purpose of comparison.
The first baseline, called \method{Newton} in \cref{fig:cubicle_eval}, implements the exact Newton update using domain decomposition, \edit{where each robot communicates its exact (dense) Schur complement similar to DDF-SAM \cite{Cunningham10DDFSAM}.}
\edit{In addition, we also implement two baselines that apply heuristic sparsification to \method{Newton}:}
in \method{Block-Diagonal}, each robot only transmits the diagonal blocks of its Schur complement,
whereas in \method{Block-Tree}, each robot transmits both diagonal blocks and off-diagonal blocks that form a tree sparsity pattern.
\edit{These two baselines are similar to the Jacobi and tree preconditioning \cite{Kushal12Visibility},
as well as the approximate summarization strategy in DDF-SAM~2.0 \cite{Cunningham13DDFSAM2}.}
\cref{fig:cubicle_eval:iteration} shows the accuracy achieved by all methods (measured by norm of the Riemannian gradient)
as a function of iterations.
As expected, \method{Newton} achieves the best convergence speed and converges to a high-precision solution in two iterations.
However, when combined with heuristic sparsifications in \method{Block-Diagonal} and \method{Block-Tree}, 
the resulting methods have very slow convergence.
Intuitively, this result shows that a diagonal or tree sparsity pattern is not sufficient for preserving the spectrum of the original dense matrix.\techreport{\footnote{
	In centralized optimization (\eg, \cite{Agarwal10BAL,Dellaert10PCG,Wu11MulticoreBA,Kushal12Visibility}), these heuristic sparsifications often serve as preconditioners and need to be used within iterative methods such as conjugate gradient to provide the best performance.}}
In contrast, our proposed method achieves fast convergence under a wide range of sparsification parameter $\epsilon$.
Furthermore, by varying $\epsilon$, the proposed method provides a principled way to trade off convergence speed with communication efficiency.

\cref{fig:cubicle_eval:communication} visualizes the accuracy as a function of total uploads to the server.
Since both the Hessian and Laplacian matrices are symmetric, we only record the communication when uploading their upper triangular parts as sparse matrices.
To convert the result to kilobyte ({\kB}), we assume each scalar is transmitted in double precision.
Our results show that the proposed method achieves the best communication efficiency under various settings of the sparsification parameter $\epsilon$.
Moreover, even without sparsification (\ie, $\epsilon = 0$), the proposed method is still more communication-efficient than \method{Newton}.
This result is due to the following reasons.
First, since the Hessian matrix varies across iterations, \method{Newton} requires communication of the updated Hessian Schur complements at every iteration.
In contrast, the proposed method works with a \emph{constant} graph Laplacian, and hence only requires a one-time communication of its Schur complements; see \cref{alg:collaborative_rotation_averaging:sparse_schur_complement} in \cref{alg:collaborative_rotation_averaging}.
Second, \method{Newton} requires communication to form the Schur complement of the original $pn$-by-$pn$ Hessian matrix, where $n$ is the number of rotation variables and $p = \dim \SOd(d)$ is the intrinsic dimension of the rotation group (for the \dataset{Cubicle} dataset, $n = 5750$ and $p = 3$).
In contrast, the proposed method operates on the smaller \emph{$n$-by-$n$} Laplacian matrix, 
and the decrease in matrix size directly translates to communication reduction.

{Lastly, \cref{fig:cubicle_eval:download} visualizes the accuracy as a function of total communication in the download stage.
	Notice that the evolution follows the same trend as \cref{fig:cubicle_eval:iteration}, where the horizontal axis shows the number of iterations.
	This observation is expected as a result of \cref{rm:comm_requirements_rotation_averaging}, which shows that the 
	communication complexity in the download stage is $O(mK|\Ccal|)$, \ie, the total downloads grows \emph{linearly} with respect to the number of iterations $K$.
}

{
\myParagraph{Scalability with Number of Robots}
In this experiment, we evaluate the scalability of \cref{alg:collaborative_rotation_averaging}.
For this purpose, we generate a large-scale synthetic rotation averaging problem with 8000 rotations arranged in a 3D grid (\cref{fig:num_robots:dataset}).
With probability 0.3, we add relative measurements between nearby rotations, which are corrupted by Langevin noise with a standard deviation of 5~deg.
Then, we divide the dataset to simulate increasing number of robots, and 
run \cref{alg:collaborative_rotation_averaging} with sparsification parameter $\epsilon = 0.5$ until the Riemannian gradient norm reaches $10^{-5}$.
\cref{fig:num_robots:rmse_with_gt} shows the evolution of the estimation RMSE with respect to the ground truth rotations.
For reference, we also show the RMSE achieved by the global minimizer to \cref{prob:rotation_averaging} (denoted as ``MLE'' in the figure).
Note that due to measurement noise, the MLE is in general different from the ground truth.
The proposed method is able to achieve an RMSE similar to the MLE after a single iteration, despite the worse initialization as the number of robots increases.
\cref{fig:num_robots:iterations} shows the evolution of gradient norm as a function of iterations.
Note that all curves in \cref{fig:num_robots:iterations} have similar slopes, 
which suggests that the empirical convergence rate of our method is not sensitive to the number of robots.
This observation is compatible with the (local) convergence rate established in Theorem~3, which does not depend on the number of robots $m$.
This property makes our method more appealing than existing fully distributed methods, whose convergence speed typically degrades as the number of robots increases (\eg, see \cite[Fig.~8]{Tian21DC2PGO}).
Lastly, \cref{fig:num_robots:uploads} shows the evolution of total uploads as a function of iterations.
As we divide the dataset to simulate more robots, both the number of inter-robot measurements and the number of separators $|\Ccal|$ increase, and thus each iteration requires more communication.}

\begin{figure}[t]
	\centering
	\begin{subfigure}[t]{0.48\linewidth}
		\includegraphics[width=\textwidth, trim=0 0 0 0, clip]{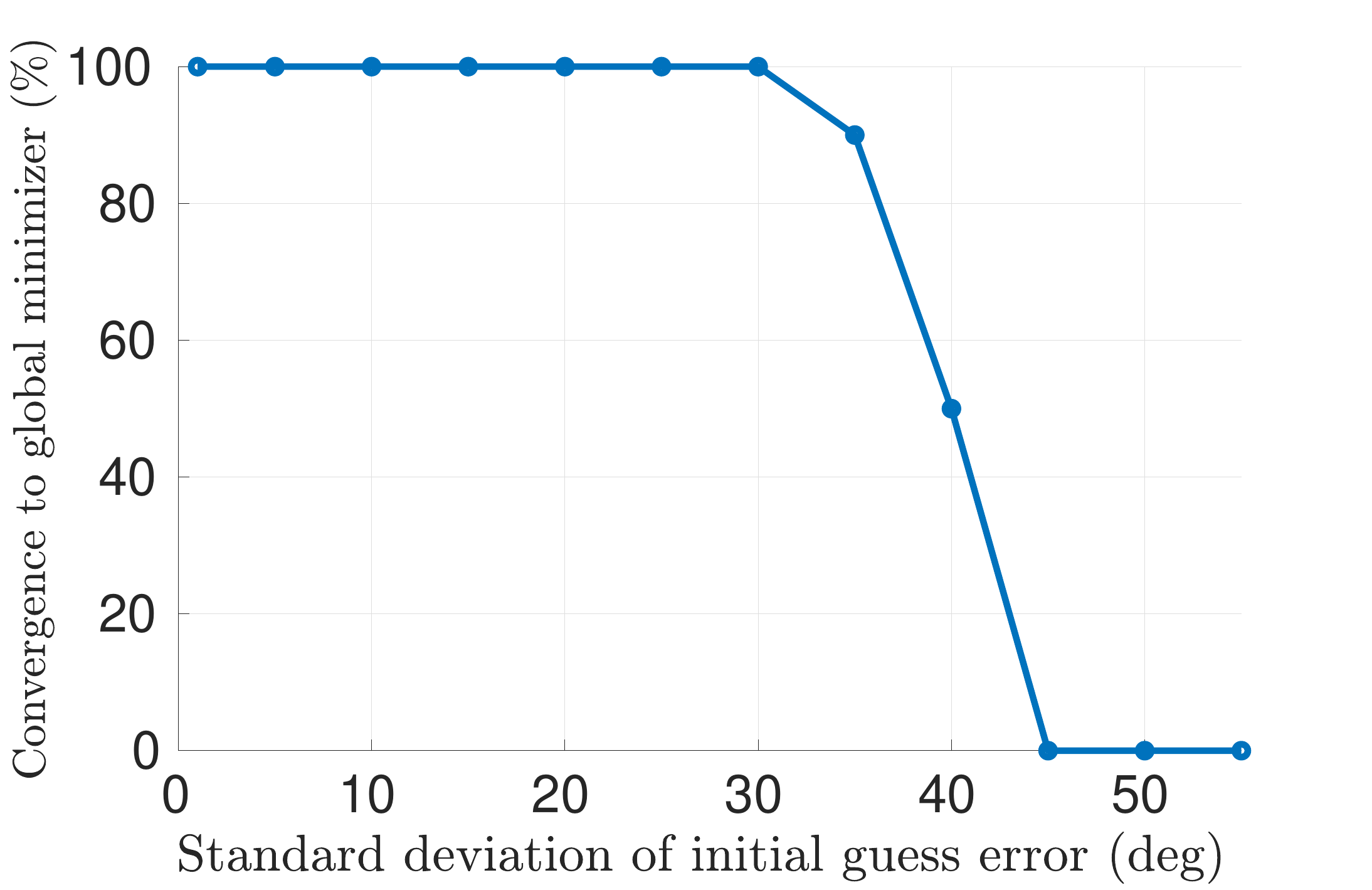}
		\caption{Success rate}
		\label{fig:convergence_basin:success_rate}
	\end{subfigure}
	~
	\begin{subfigure}[t]{0.48\linewidth}
		\includegraphics[width=\textwidth]{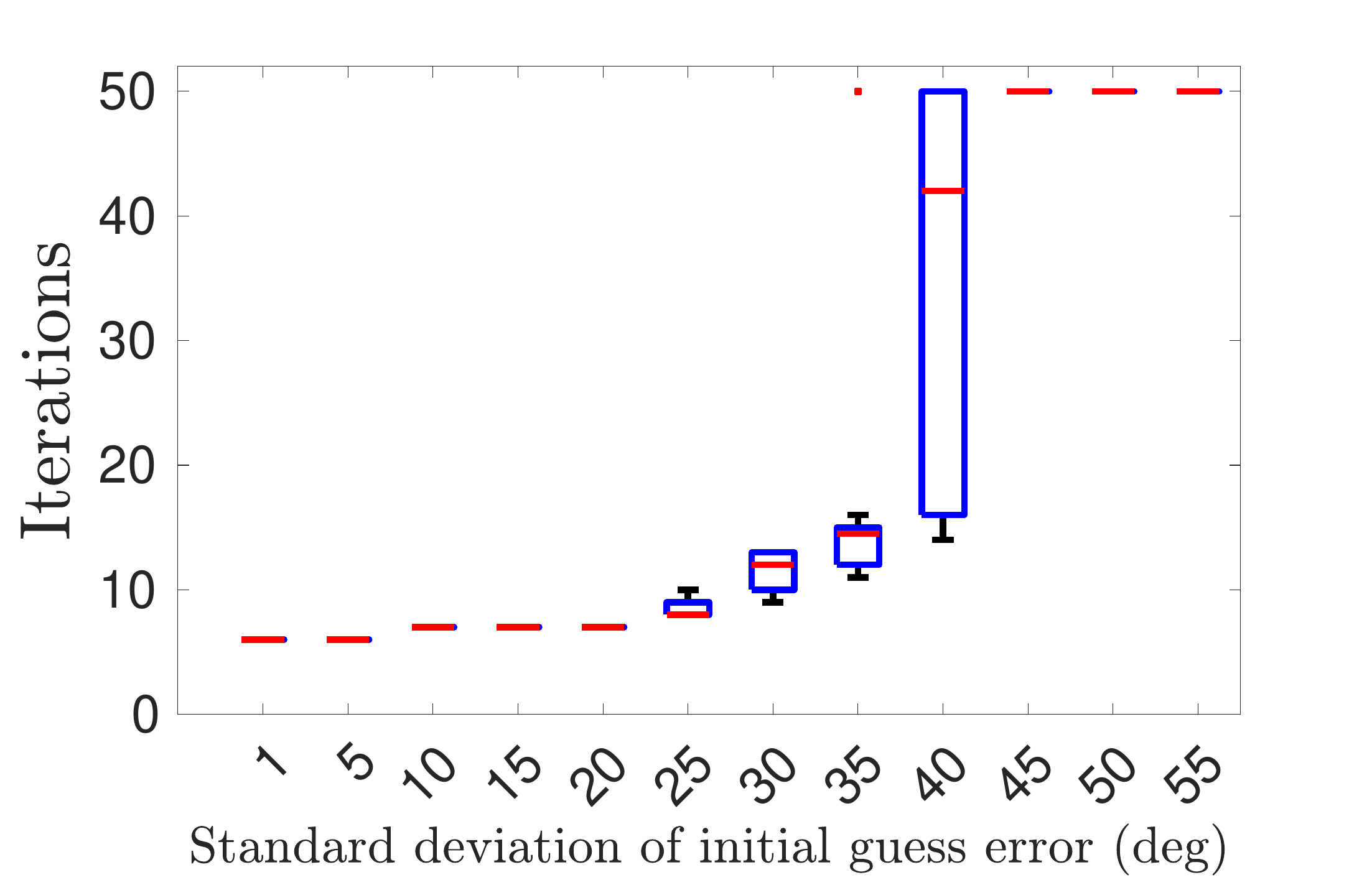}
		\caption{Number of iterations}
		\label{fig:convergence_basin:iterations}
	\end{subfigure}	
	\caption{
		{Sensitivity of \cref{alg:collaborative_rotation_averaging} to accuracy of initial guess.}
		{We generate synthetic initial guesses with degrading accuracy by perturbing the global minimizer with increasing levels of Langevin noise.}
		At each level of perturbation, 10 random runs are performed.
		(a) Percentage of runs that converge to the global minimizer.
		(b) Boxplot of number of iterations used by \cref{alg:collaborative_rotation_averaging}.} 
	\label{fig:convergence_basin}
        \vspace{-0.5cm}
\end{figure}

{\myParagraph{Sensitivity to Initial Guess}}
So far, we have used the distributed chordal initialization technique \cite{Choudhary17IJRR} to initialize  \cref{alg:collaborative_rotation_averaging}.
In the next experiment, we test the sensitivity of our proposed method to poor initial guesses.
For this purpose, we use a 9-robot simulation where each robot owns 512 rotation variables, and generate synthetic initial guesses by perturbing the global minimizer with increasing level of Langevin noise.
Using the synthetic initialization, we run \cref{alg:collaborative_rotation_averaging} with sparsification parameter $\epsilon=0.5$ until the Riemannian gradient norm reaches $10^{-5}$ or the number of iterations exceeds 50.
At each noise level, 10 random runs are performed.
\cref{fig:convergence_basin:success_rate}  shows the fraction of trials that successfully converge to the global minimizer.
We observe that \cref{alg:collaborative_rotation_averaging} enjoys a large convergence basin:
the success rate only begins to decrease at a large {initial guess error} of $35$~deg.
\cref{fig:convergence_basin:iterations} shows the number of iterations used by \cref{alg:collaborative_rotation_averaging} to reach convergence.
Our results suggest that the proposed method is not sensitive to the quality of initialization and usually requires a small number of iterations to converge.

\begin{figure}[t]
	\centering
	\begin{subfigure}[t]{0.48\linewidth}
		\includegraphics[width=\textwidth, trim=0 0 0 0, clip]{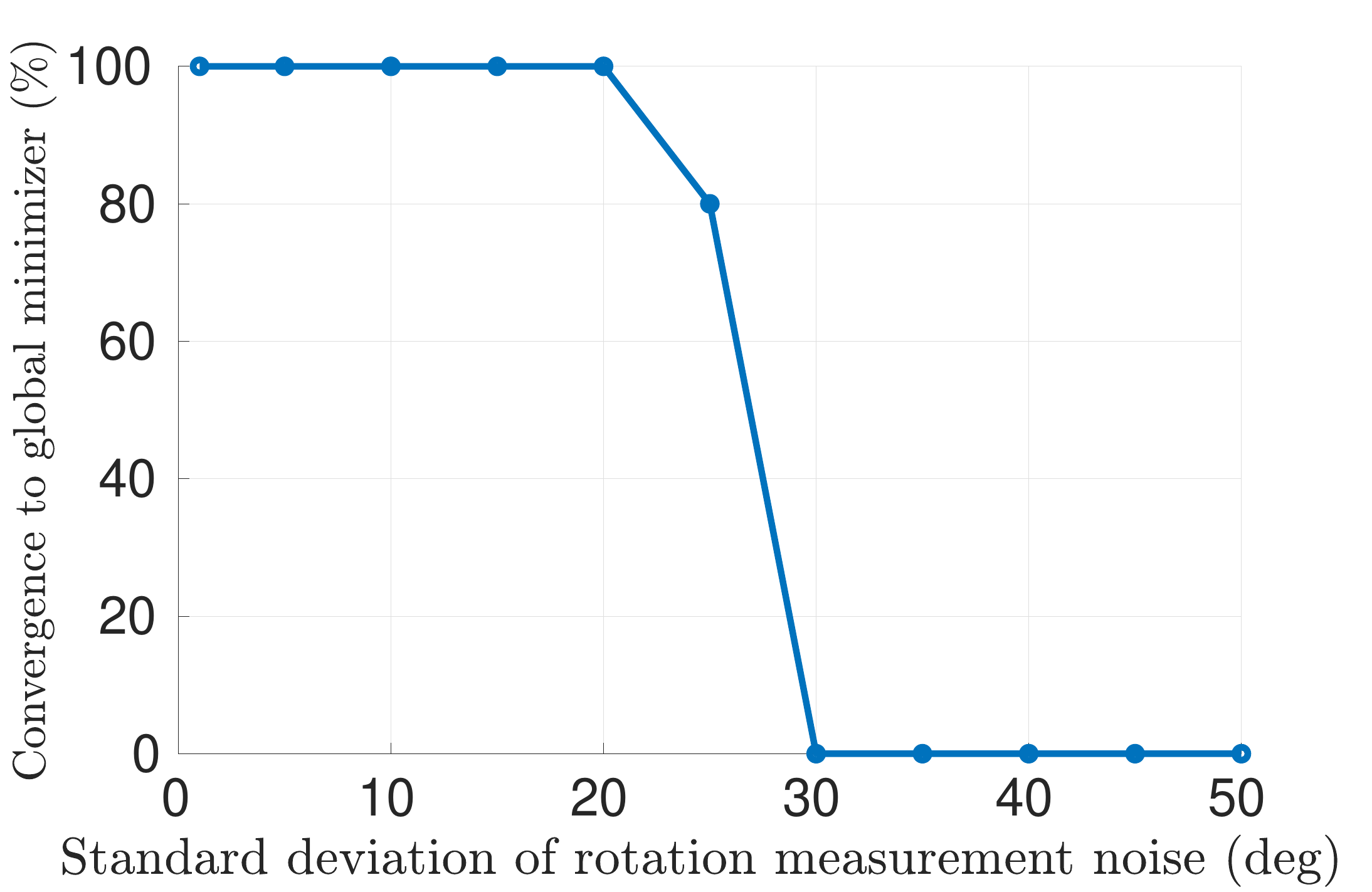}
		\caption{Success rate}
		\label{fig:noise_level:success_rate}
	\end{subfigure}
	~
	\begin{subfigure}[t]{0.48\linewidth}
		\includegraphics[width=\textwidth]{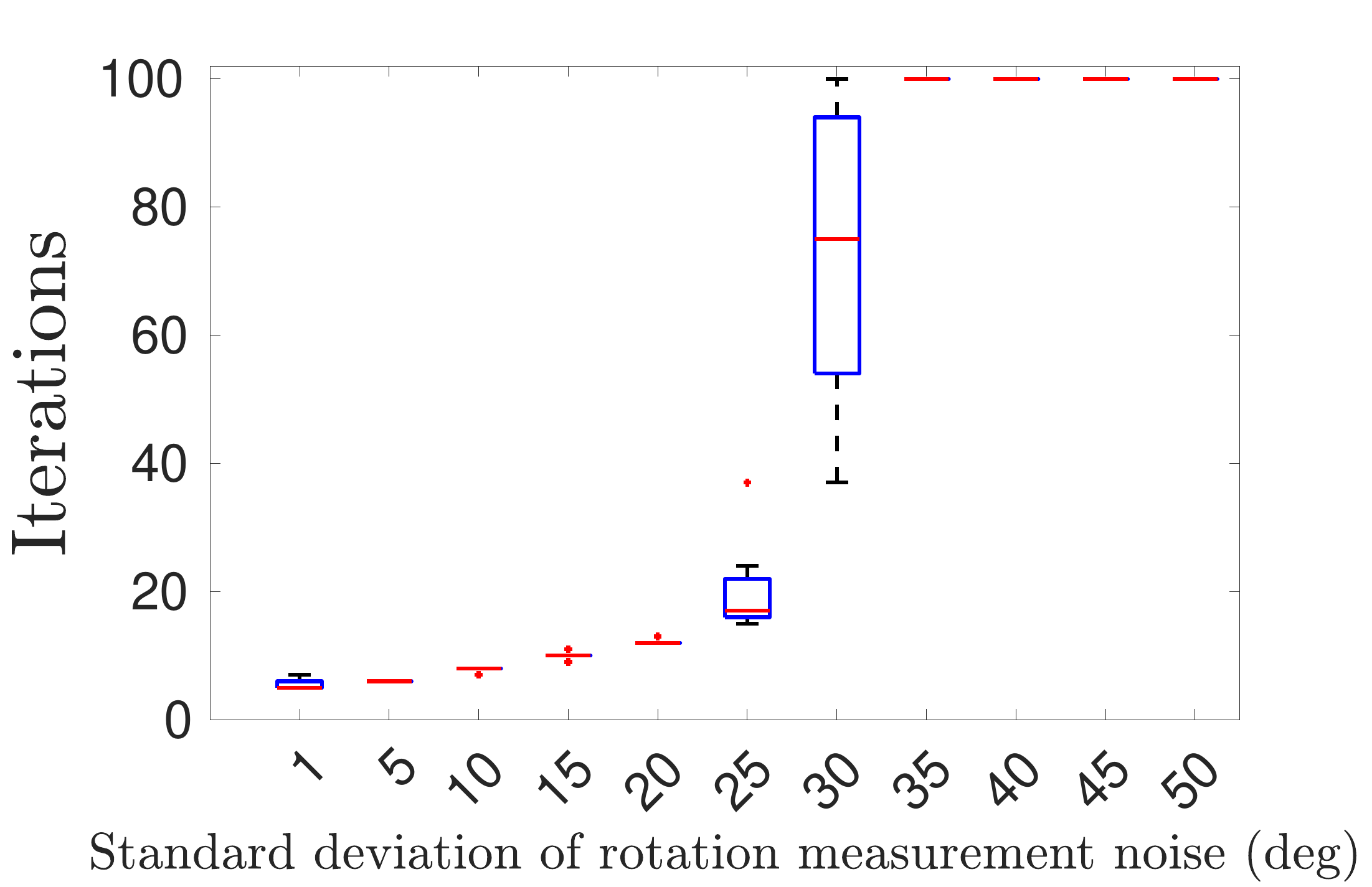}
		\caption{Number of iterations}
		\label{fig:noise_level:iterations}
	\end{subfigure}	
	\caption{
		{Sensitivity of \cref{alg:collaborative_rotation_averaging} to rotation measurement noise.
			We generate synthetic chordal rotation averaging problems with increasing magnitude of measurement noise.
			At each noise level, 10 random runs are performed.
			(a) Percentage of runs that converge to the global minimizer.
			(b) Boxplot of number of iterations used by \cref{alg:collaborative_rotation_averaging}.}} 
	\label{fig:noise_level}
        \vspace{-0.5cm}
\end{figure}

{\myParagraph{Sensitivity to Measurement Noise}
Next, we analyze the sensitivity of \cref{alg:collaborative_rotation_averaging} to increasing levels of measurement noise.
The setup is similar to the previous experiment, where we use a 9-robot simulation and each robot owns 512 rotations.
However, instead of varying the quality of the initial guess, we vary the noise level when generating the synthetic problem.
\cref{fig:noise_level} shows the results.
We find that \cref{alg:collaborative_rotation_averaging} is relatively more sensitive to the measurement noise, and start to converge to suboptimal local minima as the noise level increases above $25$~deg.
Nevertheless, we note that the level of rotation noise encountered in practice is usually much lower,\footnote{Here we only consider rotation noise of \emph{inlier} measurements. \emph{Outlier} measurements will be handled using the robust optimization framework presented in \cref{sec:gnc}.} and thus we expect our algorithm to still provide effective estimation (see real-world evaluations in \cref{sec:experiments:jackal_dataset,sec:experiments:sfm}).}

{\myParagraph{Outlier-Robust Optimization}
Lastly, we evaluate the proposed outlier-robust optimization method to solve robust rotation averaging problems.
In this experiment, we use a 9-robot simulation where each robot owns 512 rotations.
In \cref{sec:experiments:jackal_dataset,sec:experiments:sfm}, we demonstrate our method on real-world SLAM and SfM problems.
As in common SLAM scenarios, we assume each robot has a backbone of odometry measurements within its own trajectory that are free of outliers.
Then, with increasing probability, we replace the remaining measurements (corresponding to intra-robot and inter-robot loop closures) with gross outliers.
{All inlier measurements (including odometry) are corrupted by Langevin noise with a standard deviation of $3$~deg,
and we set the TLS threshold $\ebar$ to correspond to $10$~deg.}
\cref{fig:outlier:rmse} visualizes the RMSE with respect to ground truth rotations.
As expected, \cref{alg:collaborative_rotation_averaging} without GNC is not robust to outliers and shows significant error as soon as outlier measurements are introduced.
Nevertheless, by using \cref{alg:collaborative_rotation_averaging} within GNC as described in \cref{sec:gnc}, the resulting approach becomes robust and is able to tolerate {up to 70\% of outlier loop closures}. 
In \cref{fig:outlier:iterations}, we study the efficiency of our approach by showing the total number of inner iterations of \cref{alg:collaborative_rotation_averaging} used by GNC.
Recall that each inner iteration also corresponds to a single round of communication.
When the outlier ratio is zero, GNC reduces to the standard \cref{alg:collaborative_rotation_averaging} and only requires a few iterations to converge.
When outliers are added, GNC requires multiple outer iterations and thus multiple calls to \cref{alg:collaborative_rotation_averaging}, resulting in increased communication rounds.
Nevertheless, for all test cases with less than 70\% outlier measurements, the number of communication rounds is approximately 100, which is a reasonable requirement for a real system.}

\begin{figure}[t]
	\centering
	\begin{subfigure}[t]{0.48\linewidth}
		\includegraphics[width=\textwidth, trim=0 0 0 0, clip]{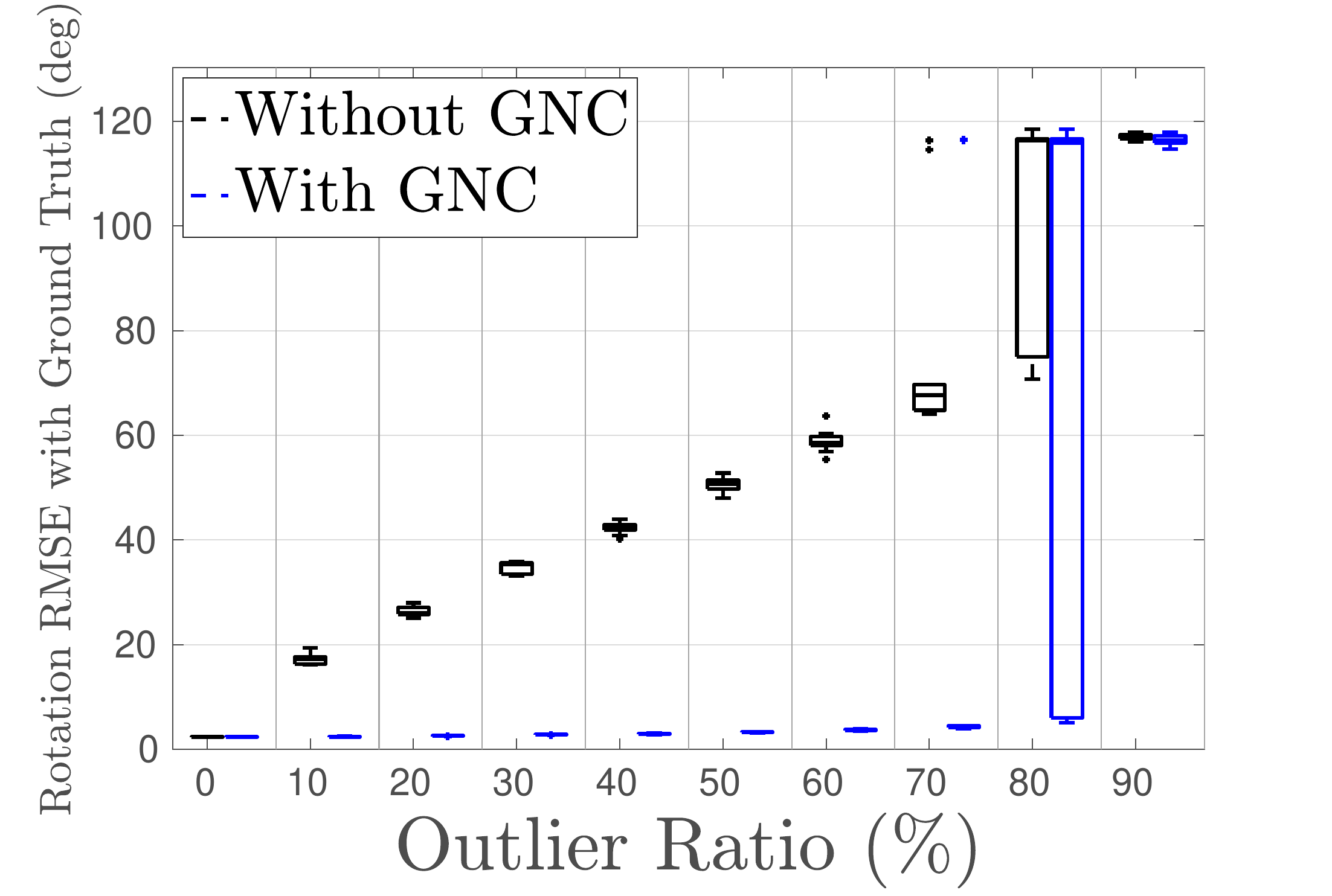}
		\caption{RMSE with ground truth}
		\label{fig:outlier:rmse}
	\end{subfigure}
	~
	\begin{subfigure}[t]{0.48\linewidth}
		\includegraphics[width=\textwidth]{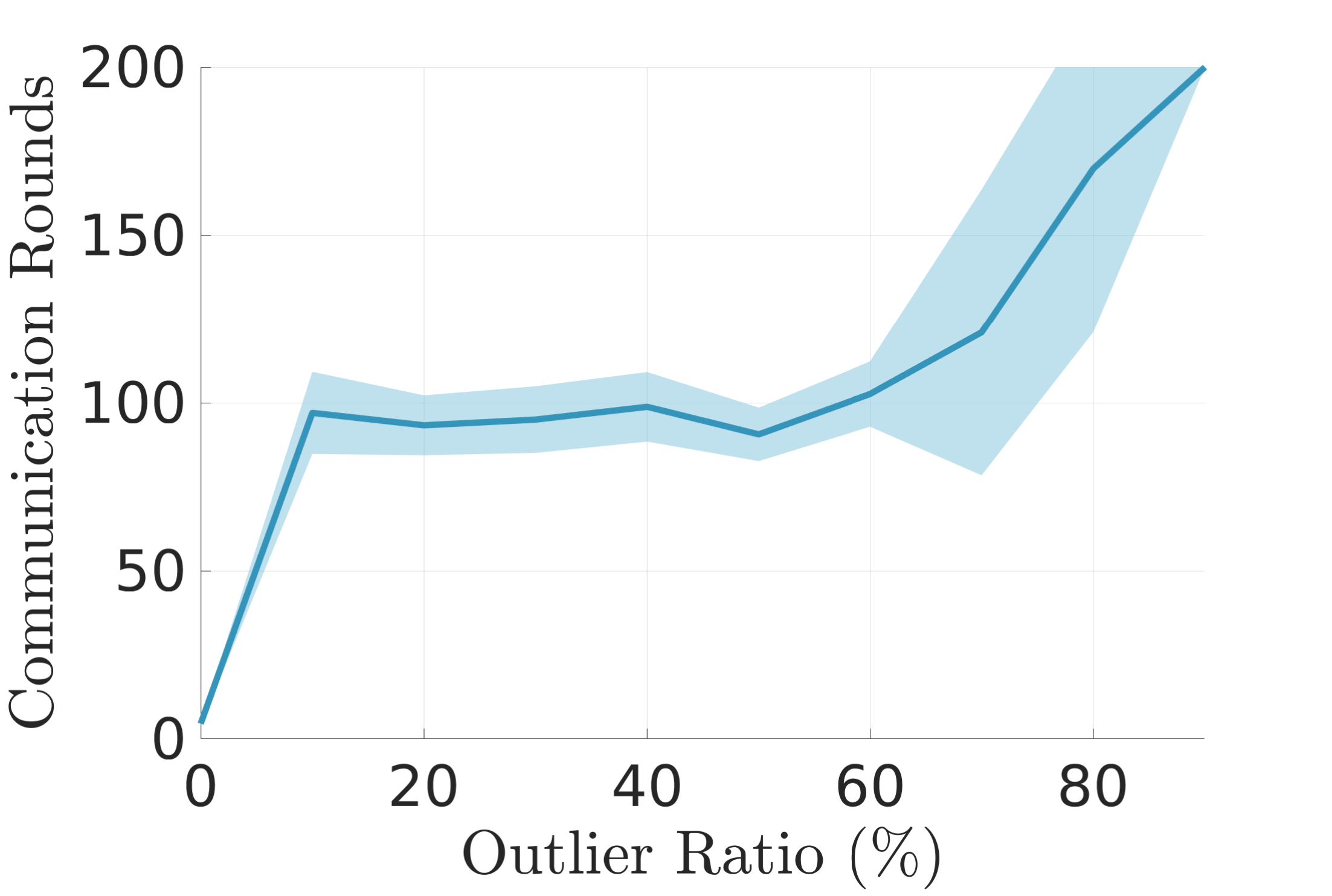}
		\caption{Communication rounds}
		\label{fig:outlier:iterations}
	\end{subfigure}	
	\caption{
		{Evaluation of robust optimization on synthetic rotation averaging problems corrupted by increasing percentage of outlier loop closures.
			At each outlier percentage, 10 random runs are performed.
		(a) RMSE with respect to ground truth rotations.
		(b) Communication rounds used by GNC. Sold line and shaded area correspond to the mean and one standard deviation, respectively.}
	} 
	\label{fig:outlier}
    \vspace{-0.5cm}
\end{figure}


\begin{table*}[t]
	\centering
	\renewcommand{\arraystretch}{1.1}
	\caption{Rotation averaging on benchmark SLAM datasets with 5 robots.
		$|\Vcal|$ and $|\Ecal|$ denote the total number of rotation variables and measurements, respectively.
		We run the baseline \method{Newton} method and the proposed method (\cref{alg:collaborative_rotation_averaging}) with sparsification parameter $\epsilon=1.5$,
		and compare the number of iterations, uploads, and downloads to reach a Riemannian gradient norm of $10^{-5}$.
		{For the proposed method, we also show the sparsity achieved by sparsification (lower is better).}
		Results averaged across 5 runs.
	}
	\label{tab:rotation_averaging}
	\resizebox{\textwidth}{!}{%
		\begin{tabular}{|l|r|r| rr|rr|rr|r|}
			\hline
			\multirow{2}{*}{Dataset} &
			\multicolumn{1}{c|}{\multirow{2}{*}{$|\Vcal|$}} &
			\multicolumn{1}{c|}{\multirow{2}{*}{$|\Ecal|$}} &
			\multicolumn{2}{c|}{Iterations} &
			\multicolumn{2}{c|}{Upload (kB)} &
			\multicolumn{2}{c|}{Download (kB)} &
			\multicolumn{1}{c|}{\multirow{2}{*}{\begin{tabular}[c]{@{}c@{}}Achieved sparsity \\ by proposed (\%)\end{tabular}}} \\ \cline{4-9}
			&
			\multicolumn{1}{c|}{} &
			\multicolumn{1}{c|}{} &
			\multicolumn{1}{c|}{Newton} &
			\multicolumn{1}{c|}{Proposed} &
			\multicolumn{1}{c|}{Newton} &
			\multicolumn{1}{c|}{Proposed} &
			\multicolumn{1}{c|}{Newton} &
			\multicolumn{1}{c|}{Proposed} &
			\multicolumn{1}{c|}{} \\ \hline
			Killian Court (2D) & 808   & 827   & \multicolumn{1}{r|}{2} & 3    & \multicolumn{1}{r|}{1.6}      & 1.1    & \multicolumn{1}{r|}{0.5}   & 0.8    & 100  \\ \hline
			CSAIL (2D)         & 1045  & 1171  & \multicolumn{1}{r|}{2} & 4    & \multicolumn{1}{r|}{7.2}      & 5.9    & \multicolumn{1}{r|}{2.3}   & 4.6    & 97.3 \\ \hline
			INTEL (2D)         & 1228  & 1483  & \multicolumn{1}{r|}{3} & 4.2  & \multicolumn{1}{r|}{10.5}     & 5.8    & \multicolumn{1}{r|}{3.3}   & 4.6    & 96.4 \\ \hline
			Manhattan (2D)     & 3500  & 5453  & \multicolumn{1}{r|}{2} & 5    & \multicolumn{1}{r|}{118.9}    & 49.5   & \multicolumn{1}{r|}{12.5}  & 31.3   & 38.7 \\ \hline
			KITTI 00 (2D)      & 4541  & 4676  & \multicolumn{1}{r|}{2} & 2    & \multicolumn{1}{r|}{13.2}     & 6.6    & \multicolumn{1}{r|}{4.4}   & 4.4    & 100  \\ \hline
			City (2D)          & 10000 & 20687 & \multicolumn{1}{r|}{2} & 4    & \multicolumn{1}{r|}{450.3}    & 351.5  & \multicolumn{1}{r|}{129}   & 258.1  & 97.3 \\ \hline
			Garage (3D)        & 1661  & 6275  & \multicolumn{1}{r|}{1} & 2    & \multicolumn{1}{r|}{274.4}    & 88.9   & \multicolumn{1}{r|}{35.8}  & 71.6   & 93.2 \\ \hline
			Sphere (3D)        & 2500  & 4949  & \multicolumn{1}{r|}{2} & 8.6  & \multicolumn{1}{r|}{2548.8}   & 106    & \multicolumn{1}{r|}{19.2}  & 82.6   & 16.9 \\ \hline
			Torus (3D)         & 5000  & 9048  & \multicolumn{1}{r|}{3} & 9.6  & \multicolumn{1}{r|}{10423.7}  & 229.5  & \multicolumn{1}{r|}{57}    & 182.2  & 12.4 \\ \hline
			Grid (3D)          & 8000  & 22236 & \multicolumn{1}{r|}{3} & 9.2  & \multicolumn{1}{r|}{206871.6} & 886.6  & \multicolumn{1}{r|}{220.8} & 677    & 2.7  \\ \hline
			Cubicle (3D)       & 5750  & 16869 & \multicolumn{1}{r|}{2} & 6.8  & \multicolumn{1}{r|}{7015}     & 440.3  & \multicolumn{1}{r|}{107.7} & 366.2  & 19.9 \\ \hline
			Rim (3D)           & 10195 & 29743 & \multicolumn{1}{r|}{4} & 23.4 & \multicolumn{1}{r|}{53657.9}  & 1320.9 & \multicolumn{1}{r|}{209.1} & 1223.2 & 6.6  \\ \hline
		\end{tabular}%
	}
\end{table*}

\subsection{Evaluation on Benchmark PGO Datasets}
\label{sec:experiments:pgo_datasets}

{In this subsection, we evaluate our approach on 12 benchmark pose graph SLAM datasets.
For these datasets, we do not explicitly handle outliers. 
Outlier-robust estimation will be evaluated in
\cref{sec:experiments:jackal_dataset,sec:experiments:sfm}.
}

{\myParagraph{Evaluation on Rotation Averaging Subproblem}
We first evaluate \cref{alg:collaborative_rotation_averaging} on the rotation averaging subproblems extracted from the benchmark datasets.
For each problem, we simulate a scenario with 5 robots, and run the proposed method (\cref{alg:collaborative_rotation_averaging}) with sparsification parameter $\epsilon=1.5$ and the baseline \method{Newton} method.}
Both methods are terminated when the Riemannian gradient norm is smaller than $10^{-5}$.
Since the spectral sparsification method we use \cite{Spielman2011Sampling} is randomized, 
we perform 5 random runs of our method.
\cref{tab:rotation_averaging} shows the {average number of iterations, uploads, and downloads} to reach the desired precision.
On all datasets, we are able to verify that all methods converge to the global minima of the considered rotation averaging problems.
The proposed method achieves an empirical convergence speed that is close to \method{Newton}
and typically converges in a few iterations.\footnote{
	One notable exception is the \dataset{Rim} dataset, for which our method uses more than 20 iterations to converge. 
	A closer investigation reveals that this dataset actually contains some outlier measurements.
	Specifically, at the global minimizer $\Rstar$, there are 28 measurements $\Rtilde_{ij}$ for which $\dist(\Rstar_i \Rtilde_{ij}, \Rstar_j) > 60~\text{deg}$.
	Since these outliers have large residuals, their contributions to the Hessian can no longer be well approximated by the corresponding Laplacian terms.
	As a result, the performance of our method is negatively impacted.}
We note that this is significantly faster than existing fully distributed methods (\eg, \cite{Tian21DC2PGO,Fan2021MajorizationJournal}) that often require hundreds of iterations to achieve moderate precision.
{For both \method{Newton} and the proposed method, the total download is proportional to the number of iterations (see \cref{rm:comm_requirements_rotation_averaging}). Thus, our method uses more downloads since it requires more iterations.
However, we note that compared to the download stage, the upload stage is more communication-intensive since robots need to transmit (potentially dense) Schur complements to the server. Using spectral sparsification, the proposed approach achieves significant reduction in uploads, especially on challenging datasets such as \dataset{Grid} and \dataset{Rim}. 
Finally, the last column of \cref{tab:rotation_averaging} shows the achieved sparsification as the percentage of nonzero elements that remain after spectral sparsification.}
We observe that the benefit of sparsification varies across datasets.
For example, on \dataset{Killian Court} and \dataset{INTEL}, the effect of sparsification is limited because the exact Schur complement $S$ is already sparse.
Meanwhile, on datasets such as \dataset{Grid} and \dataset{Rim}, the benefit of sparsification is substantial and the results have less than 10\% nonzero elements.
{In \cref{sec:experiments:discussion}, we provide a thorough discussion on the impact of problem properties on sparsification performance.}

\begin{figure}[t]
	\centering
	\includegraphics[width=0.25\textwidth,trim=0 0 50 0, clip]{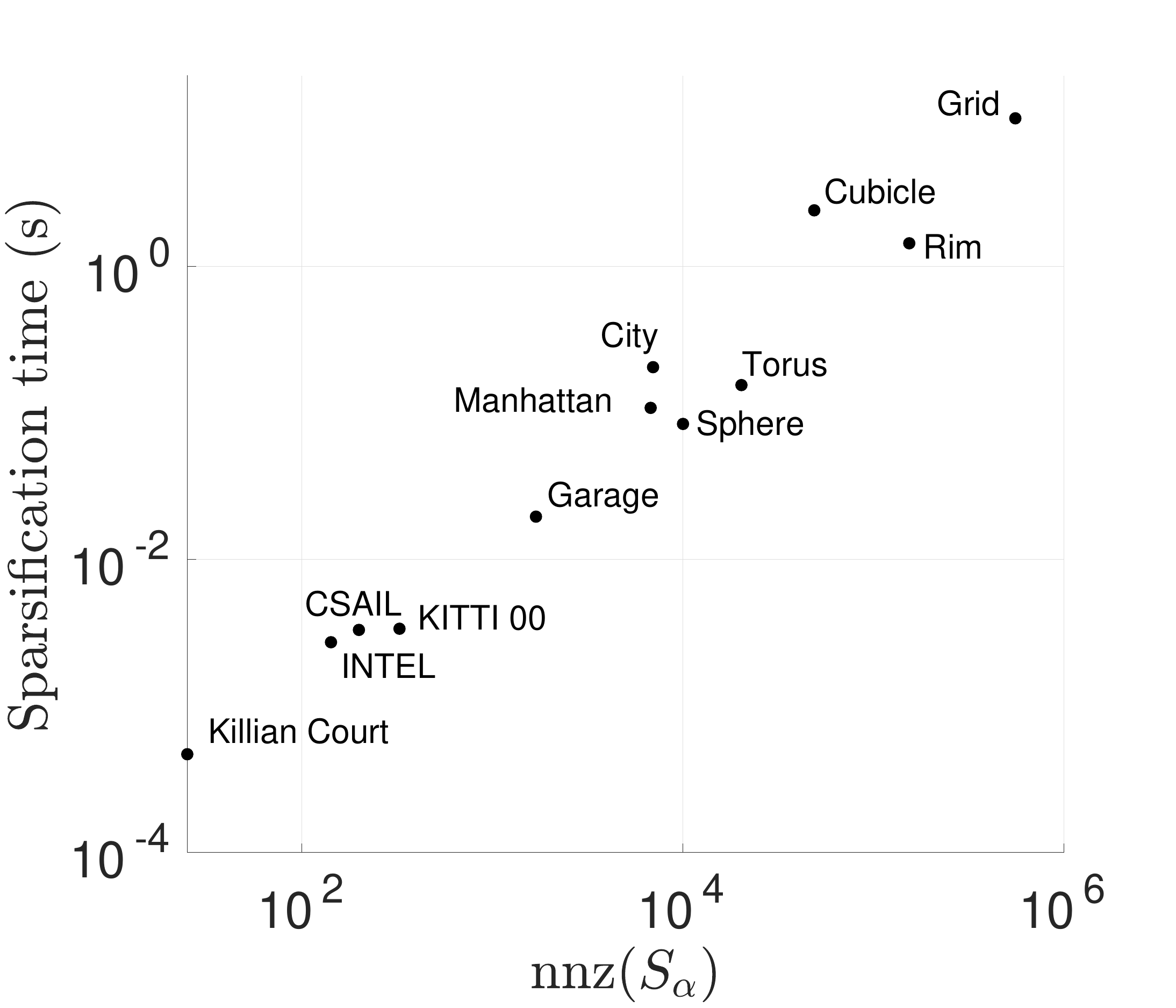}
	\caption{Spectral sparsification runtime on benchmark datasets.} 
	\label{fig:sparsification_time}
        \vspace{-0.5cm}
\end{figure}

\myParagraph{Sparsification Runtime}
Recall that in the \SparsifiedSchurComplement step in \cref{alg:collaborative_rotation_averaging}, each robot $\alpha$ sparsifies its $S_\alpha$ matrix and transmits the result $\Stilde_\alpha$ to the server.
This step uses the majority of robots' local computation time. 
In \cref{fig:sparsification_time}, we evaluate the runtime of the sparsification algorithm \cite{Spielman2011Sampling} on the 12 benchmark datasets shown in \cref{tab:rotation_averaging}. 
For each dataset, we record the maximum sparsification time among all robots, and visualize the result
as a function of the number of nonzero entries in the input matrix $S_\alpha$.
On most datasets, the maximum runtime is below one second.
On the \dataset{Grid} dataset, the input matrix has more than $5 \times 10^5$ nonzero entries and our implementation uses $10.2$ seconds.
Overall, we conclude that the runtime of our implementation is still reasonable.
However, we believe that further improvements are possible, \eg, by approximately computing effective resistances during spectral sparsification as suggested in \cite{Spielman2011Sampling}.


\begin{table*}[t]
	\centering
	\renewcommand{\arraystretch}{1.1}
	\caption{{PGO initialization on benchmark SLAM datasets with 5 robots. 
			$e^\Sigma$ measures the average error in marginal covariance due to decoupled rotation and translation estimation.
			Optimality gap and RMSE are computed with respect to optimal solutions from \method{SE-Sync}~\cite{Rosen19IJRR}.
			In addition, we also show the number of iterations and communication (both upload and download)
			used by the rotation and translation estimation stages in our approach,
			and compare the results with \method{RBCD++}~\cite{Tian21DC2PGO} to achieve the same optimality gap.}}
	\label{tab:pgo_initialization}
	\resizebox{\textwidth}{!}{%
		\begin{tabular}{|l|r|r|r|r|rr|rrr|rrr|}
			\hline
			\multicolumn{1}{|c|}{\multirow{2}{*}{Dataset}} & \multicolumn{1}{c|}{\multirow{2}{*}{$|\Vcal|$}} & \multicolumn{1}{c|}{\multirow{2}{*}{$|\Ecal|$}} & \multicolumn{1}{c|}{\multirow{2}{*}{$e^\Sigma$}} & \multicolumn{1}{c|}{\multirow{2}{*}{\begin{tabular}[c]{@{}c@{}}Optimality\\ Gap\end{tabular}}} & \multicolumn{2}{c|}{RMSE with optimal PGO solution} & \multicolumn{3}{c|}{Iterations} & \multicolumn{3}{c|}{Total communication (kB)} \\ \cline{6-13} 
			\multicolumn{1}{|c|}{} & \multicolumn{1}{c|}{} & \multicolumn{1}{c|}{} & \multicolumn{1}{c|}{} & \multicolumn{1}{c|}{} & \multicolumn{1}{c|}{Rotation (deg)} & \multicolumn{1}{c|}{Translation (m)} & \multicolumn{1}{c|}{Rot.} & \multicolumn{1}{c|}{Tran.} & \multicolumn{1}{c|}{RBCD++} & \multicolumn{1}{c|}{Rot.} & \multicolumn{1}{c|}{Tran.} & \multicolumn{1}{c|}{RBCD++} \\ \hline
			Killian Court (2D) & 808 & 827 & 0.76 & 0.12 & \multicolumn{1}{r|}{4.48} & 4.12 & \multicolumn{1}{r|}{5} & \multicolumn{1}{r|}{2} & \edit{141} & \multicolumn{1}{r|}{3.0} & \multicolumn{1}{r|}{2.4} & \edit{202} \\ \hline
			CSAIL (2D) & 1045 & 1171 & 0.33 & $4.6 \times 10^{-4}$ & \multicolumn{1}{r|}{0.06} & 0.01 & \multicolumn{1}{r|}{3} & \multicolumn{1}{r|}{3} & \edit{367} & \multicolumn{1}{r|}{8.3} & \multicolumn{1}{r|}{15.2} & \edit{$1.9 \times 10^3$} \\ \hline
			INTEL (2D) & 1228 & 1483 & 0.21 & $2.2 \times 10^{-3}$ & \multicolumn{1}{r|}{0.36} & 0.03 & \multicolumn{1}{r|}{4} & \multicolumn{1}{r|}{4} & \edit{109} & \multicolumn{1}{r|}{10.0} & \multicolumn{1}{r|}{18.7} & \edit{589} \\ \hline
			Manhattan (2D) & 3500 & 5453 & 0.92 & 0.15 & \multicolumn{1}{r|}{1.75} & 0.47 & \multicolumn{1}{r|}{4} & \multicolumn{1}{r|}{5} & \edit{42} & \multicolumn{1}{r|}{72.8} & \multicolumn{1}{r|}{147.8} & \edit{982} \\ \hline
			KITTI 00 (2D) & 4541 & 4676 & 0.86 & 0.33 & \multicolumn{1}{r|}{0.46} & 0.64 & \multicolumn{1}{r|}{3} & \multicolumn{1}{r|}{2} & \edit{1000} & \multicolumn{1}{r|}{15.4} & \multicolumn{1}{r|}{19.8} & \edit{$1.1 \times 10^4$} \\ \hline
			City (2D) & 10000 & 20687 & 0.95 & 0.12 & \multicolumn{1}{r|}{0.63} & 0.18 & \multicolumn{1}{r|}{4} & \multicolumn{1}{r|}{4} & \edit{44} & \multicolumn{1}{r|}{611} & \multicolumn{1}{r|}{$1.1 \times 10^3$} & \edit{$8.4 \times 10^3$} \\ \hline
			Garage (3D) & 1661 & 6275 & 0.99 & 0.12 & \multicolumn{1}{r|}{0.43} & 0.33 & \multicolumn{1}{r|}{2} & \multicolumn{1}{r|}{2} & \edit{66} & \multicolumn{1}{r|}{161} & \multicolumn{1}{r|}{162} & \edit{$4.0 \times 10^3$} \\ \hline
			Sphere (3D) & 2500 & 4949 & 0.87 & 0.17 & \multicolumn{1}{r|}{1.39} & 0.38 & \multicolumn{1}{r|}{7} & \multicolumn{1}{r|}{7} & \edit{1} & \multicolumn{1}{r|}{185} & \multicolumn{1}{r|}{185} & \edit{32} \\ \hline
			Torus (3D) & 5000 & 9048 & 0.25 & 0.01 & \multicolumn{1}{r|}{2.15} & 0.07 & \multicolumn{1}{r|}{8} & \multicolumn{1}{r|}{6} & \edit{7} & \multicolumn{1}{r|}{394} & \multicolumn{1}{r|}{317} & \edit{305} \\ \hline
			Grid (3D) & 8000 & 22236 & 0.43 & 0.03 & \multicolumn{1}{r|}{1.22} & 0.06 & \multicolumn{1}{r|}{8} & \multicolumn{1}{r|}{8} & \edit{5} & \multicolumn{1}{r|}{$1.7 \times 10^3$} & \multicolumn{1}{r|}{$1.7 \times 10^3$} & \edit{$1.2 \times 10^3$} \\ \hline
			Cubicle (3D) & 5750 & 16869 & 0.86 & 0.18 & \multicolumn{1}{r|}{1.53} & 0.16 & \multicolumn{1}{r|}{7} & \multicolumn{1}{r|}{6} & \edit{101} & \multicolumn{1}{r|}{869} & \multicolumn{1}{r|}{722} & \edit{$9.3 \times 10^3$} \\ \hline
			Rim (3D) & 10195 & 29743 & 0.79 & 0.63 & \multicolumn{1}{r|}{4.95} & 0.78 & \multicolumn{1}{r|}{25} & \multicolumn{1}{r|}{6} & \edit{179} & \multicolumn{1}{r|}{$2.8 \times 10^3$} & \multicolumn{1}{r|}{748} & \edit{$1.3 \times 10^4$} \\ \hline
		\end{tabular}%
	}
\end{table*}

{\myParagraph{Initialization for PGO}
	Lastly, we evaluate the use of our methods to initialize PGO.
	Recall from \cref{sec:problem_definition:applications} that our initialization scheme involves two stages.
	First, we initialize rotations by solving the rotation averaging subproblem in PGO.
	\edit{In particular, we run \cref{alg:collaborative_rotation_averaging} using an initial guess computed from a spanning tree of the pose graph.}
	This also demonstrates that our method does not need to rely on distributed chordal initialization \cite{Choudhary17IJRR}, which is itself an iterative procedure.
	\edit{Then, fixing the rotation estimates in PGO,
	we initialize translations by solving the resulting translation estimation subproblem using \cref{alg:collaborative_translation_recovery}, where \cref{alg:collaborative_translation_recovery} is simply initialized at zero.}
	\cref{tab:pgo_initialization} reports the optimality gap and estimation RMSE of our initialization method compared to the optimal PGO solutions computed using \method{SE-Sync} \cite{Rosen19IJRR}.
	Our results show that the quality of initialization varies across datasets.
	In general,
	since our initialization method decouples the estimation of rotations and translations,
	we expect its performance to degrade when there is significant coupling between rotation and translation terms in the full PGO problem.
	To investigate this hypothesis, 
	we treat PGO as an inference problem over factor graphs \cite{Dellaert17FactorGraph}
	and consider the covariance $\SigmaPGO$ of the pose estimates at the optimal solution.
	We compare $\SigmaPGO$ with the corresponding covariance $\SigmaInit$ produced by our two-stage initialization,
	where the rotation and translation blocks of $\SigmaInit$ are extracted from rotation averaging and translation estimation, respectively.
	Since both covariance matrices are large and dense, we only compute their diagonal blocks $\SigmaPGO_i$ and $\SigmaInit_i$ that correspond to the marginal covariances of pose $i$.
	We quantify the error introduced by  decoupled rotation and translation estimation by computing the normalized error
	$e^{\Sigma}_i = \norm{\SigmaPGO_i - \SigmaInit_i }_F / \norm{\SigmaPGO_i}_F$.
	\cref{tab:pgo_initialization} reports $e^\Sigma$, which is the average of $e^{\Sigma}_i $  over all poses.
	We find that the results separate all datasets into two groups.
	\dataset{INTEL}, \dataset{CSAIL}, \dataset{Torus}, and \dataset{Grid} have small values of $e^\Sigma$, and our initialization achieves the best performance, especially in terms of optimality gap.
	On the remaining datasets with larger values of $e^\Sigma$, the two-stage initialization produces worse results.
	Lastly, we note that \dataset{Rim} is a special case due to the presence of outlier measurements.}

\begin{figure*}[t]
	\centering
	\begin{subfigure}[t]{0.23\linewidth}
		\includegraphics[width=\textwidth, trim=0 120 90 150, clip]{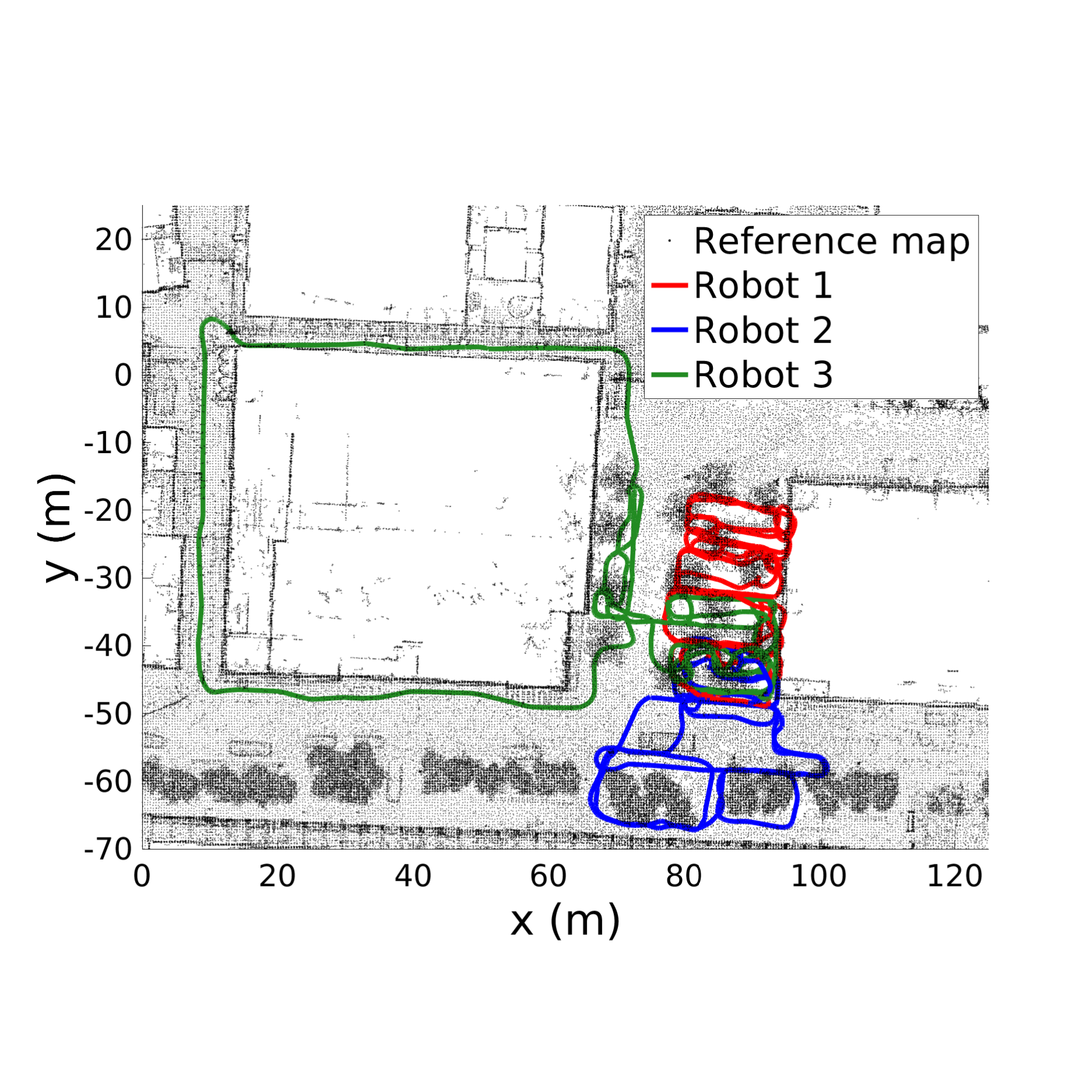}
		\caption{{Dataset visualization}}
		\label{fig:jackal_dataset:visualization}
	\end{subfigure}
	~
	\begin{subfigure}[t]{0.23\linewidth}
		\includegraphics[width=\textwidth, trim=0 0 50 0, clip]{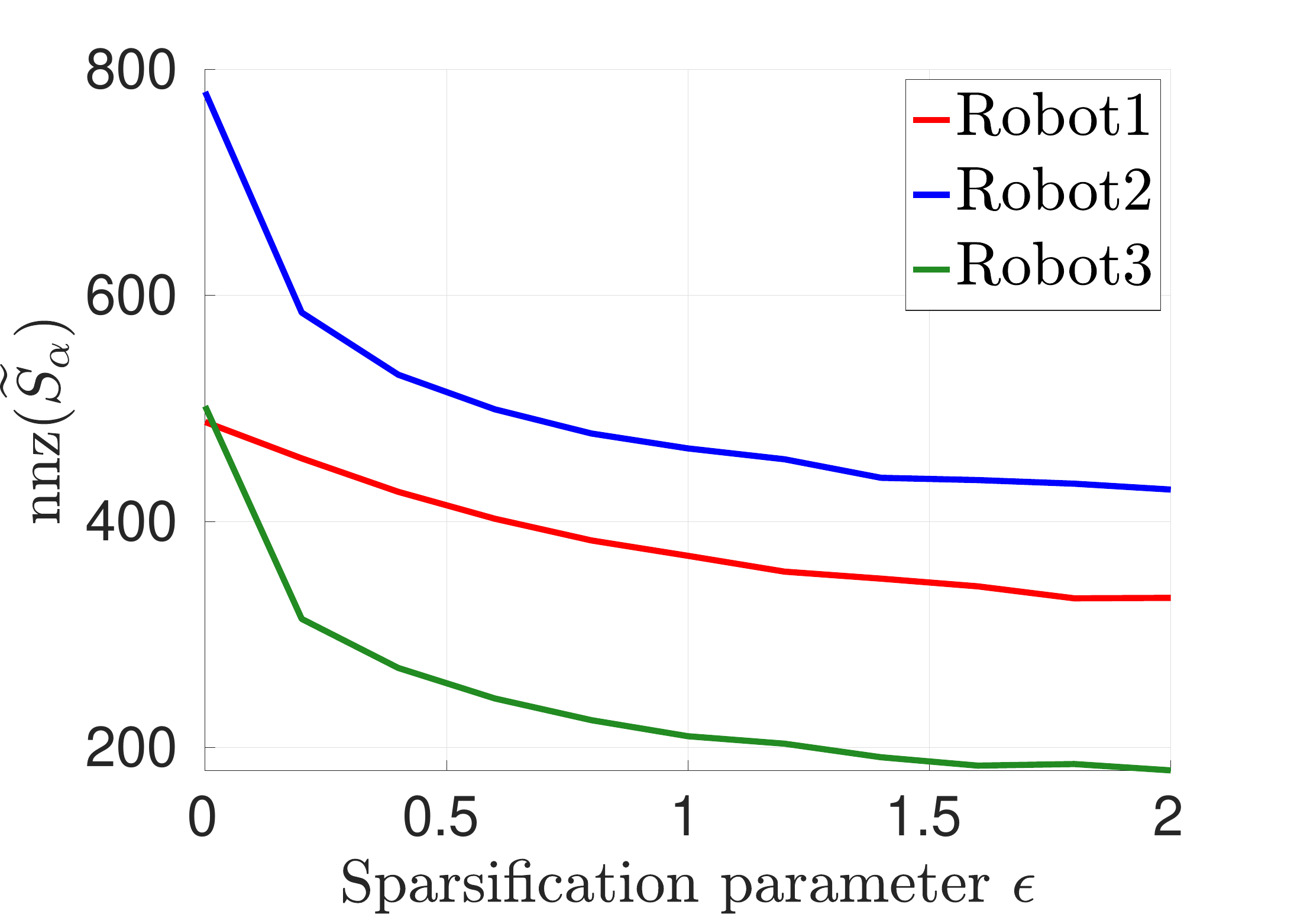}
		\caption{{Evaluation of sparsification}}
		\label{fig:jackal_dataset:sparsity}
	\end{subfigure}
	~
	\begin{subfigure}[t]{0.23\linewidth}
		\includegraphics[width=\textwidth, trim=0 0 0 0, clip]{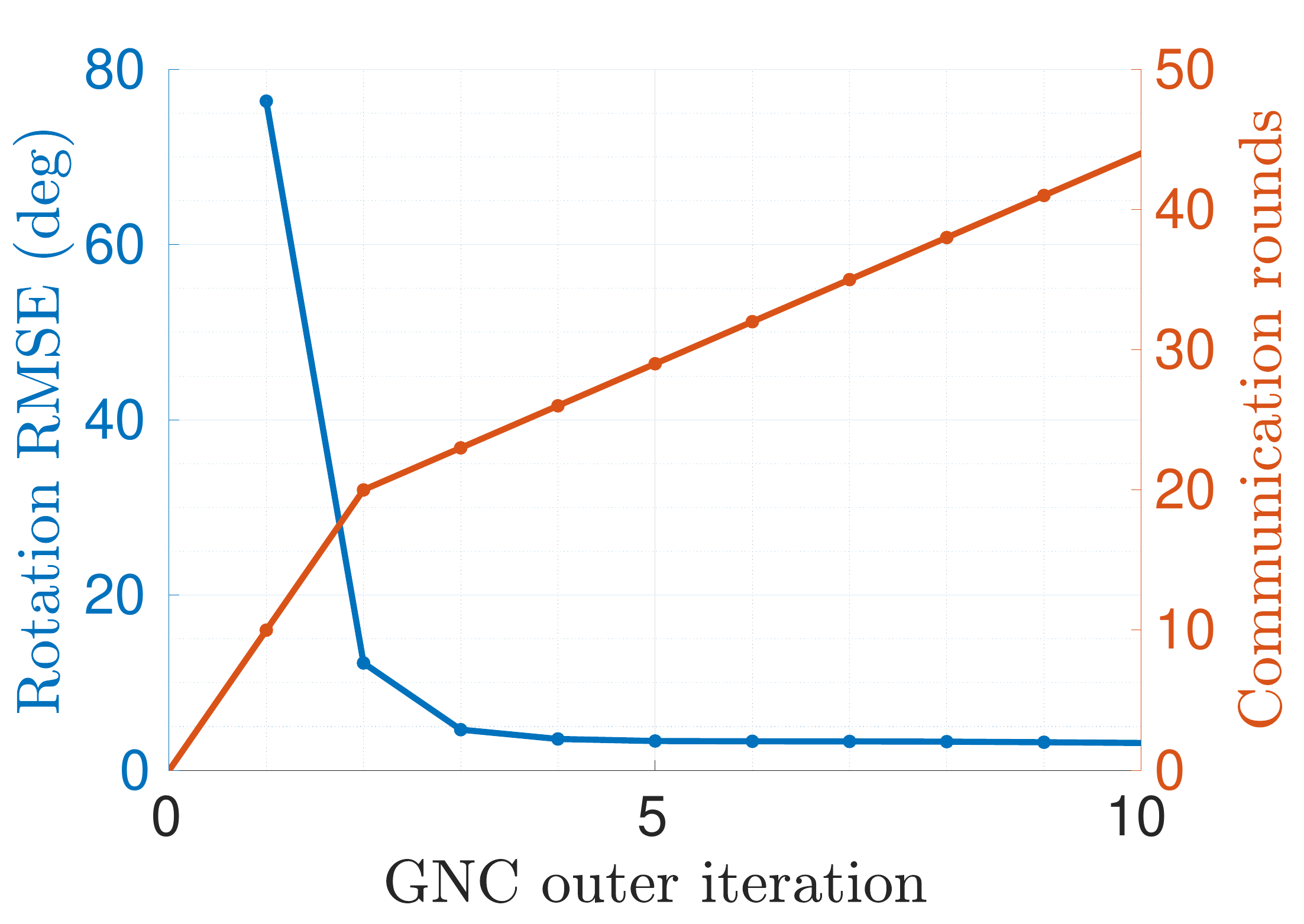}
		\caption{{Rotation estimation}}
		\label{fig:jackal_dataset:rotation_rmse}
	\end{subfigure}
	~
	\begin{subfigure}[t]{0.23\linewidth}
		\includegraphics[width=\textwidth, trim=0 0 0 0, clip]{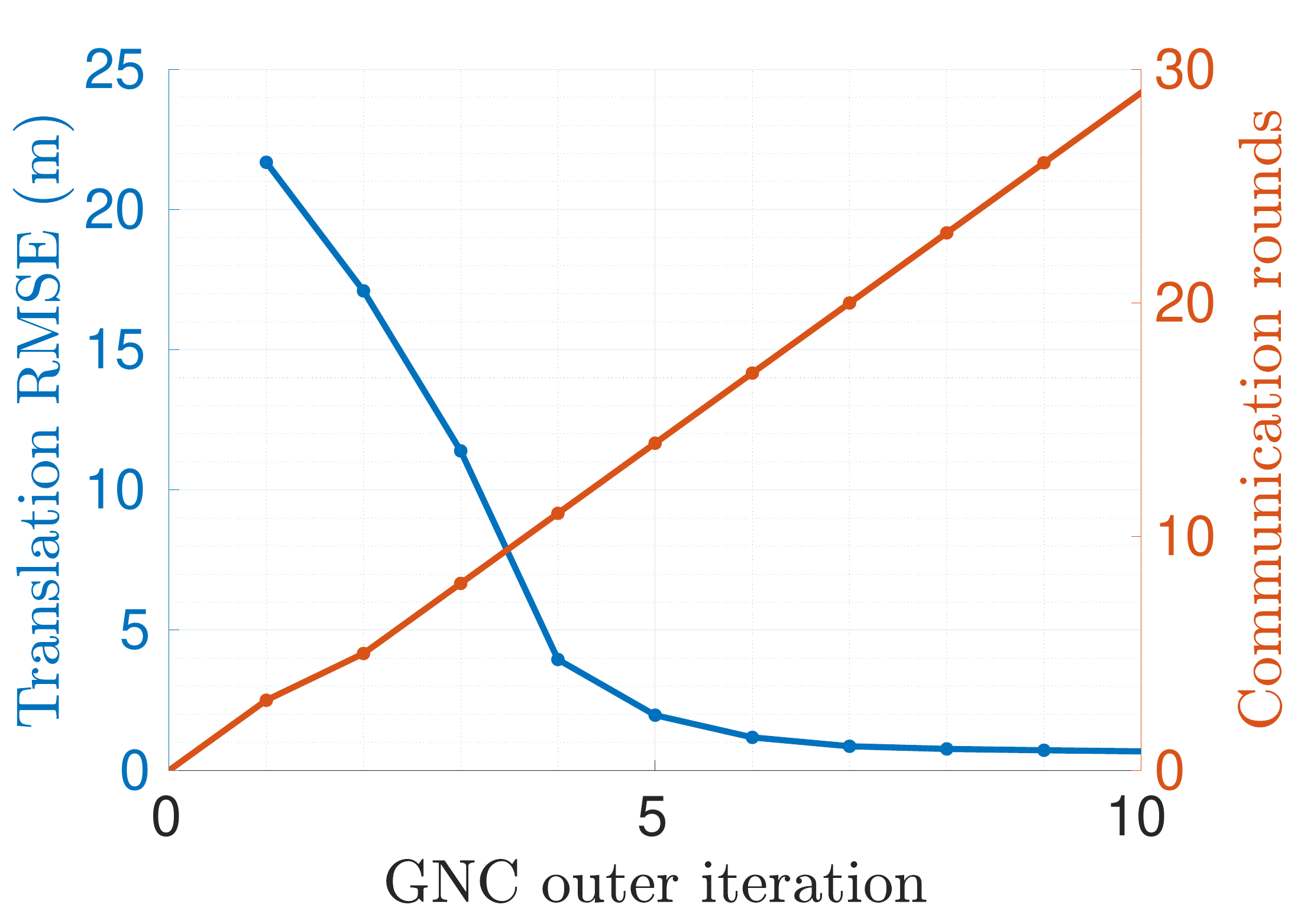}
		\caption{{Translation estimation}}
		\label{fig:jackal_dataset:translation_rmse}
	\end{subfigure}	
	\caption{
		{Robust PGO initialization on real-world collaborative SLAM dataset. 
			(a) Trajectory estimates produced by the proposed robust PGO initialization, which are
			qualitatively overlaid on top of a point cloud map of the experiment area. 
			The point cloud map was created at an earlier time (certain objects such as cars have changed) and is included only for visualization.
			(b) Evaluation of spectral sparsification. 
			(c) Rotation RMSE as a function of communication rounds during rotation estimation.
			(d) Translation RMSE as a function of communication rounds during translation estimation.}} 
	\label{fig:jackal_dataset}
        \vspace{-0.2cm}
\end{figure*}

{In addition, \cref{tab:pgo_initialization} also reports the number of iterations and total communication (both uploads and downloads) used by our initialization during rotation estimation and translation estimation.
	\edit{To provide additional context, we report corresponding results for the state-of-the-art \method{RBCD++} solver \cite{Tian21DC2PGO} to achieve the same optimality gap.}
	\mainpaper{\edit{We start \method{RBCD++} using an initial guess computed by aligning trajectory estimates obtained from robots' local PGO; see \cite[Appendix~VI]{Tian2022Sparsification} for details.}}
	\techreport{\edit{We start \method{RBCD++} using an initial guess computed by aligning trajectory estimates obtained from robots' local PGO; see Appendix~\ref{sec:additional_experiments} for details.}}
	We note that the \method{RBCD++} results are only included for reference since this method is fully distributed whereas our method assumes a server-client architecture.
	Furthermore, given more iterations, \method{RBCD++} will eventually achieve better accuracy because the method is solving the full PGO problem.
	However, our results still suggest that when a server-client architecture is available, 
	our method is favorable and provides high-quality initialization using only a few iterations.}

{\subsection{Robust PGO Initialization for Real-World Collaborative SLAM}
\label{sec:experiments:jackal_dataset}}

{In this section, we show that our approach can be used to achieve \emph{robust} PGO initialization in a real-world collaborative SLAM scenario with outlier measurements.
For this purpose, we collected three sets of trajectories using a Clearpath Jackal robot equipped with a front-facing RealSense D455 RGBD camera and IMU.
Each trajectory covers a different area outside a building on the MIT campus, with the robot making multiple loops within the designated area.
The three trajectories also overlap in a small region such that common features are observed (\cref{fig:jackal_dataset:visualization}).
We run \method{Kimera-Multi}~\cite{Tian21KimeraMulti} to process the dataset as a 3-robot collaborative visual SLAM mission.
The resulting multi-robot pose graph contains a 3D pose variable for each keyframe generated by visual-inertial odometry,
and each robot has a backbone of odometry measurements that are free of outliers.
However, there are many outlier loop closures (both within each robot's trajectory and between different robots), due to incorrect visual feature matching.}

{We demonstrate the two-stage PGO initialization as described in \cref{sec:problem_definition:applications}.
To account for outliers, we use the GNC-based robust optimization during both rotation estimation and translation estimation stages.
In our experiment, we observe that setting the TLS threshold to a smaller value of $0.5$~deg for rotation estimation leads to better performance.
The TLS threshold for translation estimation is set to $0.25$~m.
{With this setting, the two-stage initialization rejects $1090$ out of $1540$ loop closures ($71\%$).}
\techreport{In Appendix~\ref{sec:additional_experiments}, we provide additional evaluation on the sensitivity to the TLS thresholds.}
\mainpaper{In \cite[Appendix~VI]{Tian2022Sparsification}, we provide additional evaluation on the sensitivity to the TLS thresholds.}
\cref{tab:jackal_dataset} reports statistics and the accuracy achieved by our robust initialization for each robot.
As ground truth trajectories are not available, we compare against a reference solution computed by the GNC-based robust PGO solver implemented in GTSAM \cite{gtsam}.
While standard initialization (without GNC) has large errors, 
using GNC achieves robust initialization, and
the final rotation and translation RMSE over all robots are $3.0$~deg and $0.57$~m, respectively.
\cref{fig:jackal_dataset:sparsity} evaluates the effects of spectral sparsification on the real-world dataset.
We observe similar benefits as in previous experiments, where enabling sparsification ($\epsilon > 0$) significantly reduces the number of nonzero entries each robot needs to communicate.
Lastly, \cref{fig:jackal_dataset:rotation_rmse,fig:jackal_dataset:translation_rmse} evaluate the efficiency of our two-stage robust initialization, by visualizing the evolution of RMSE and number of communication rounds as a function of GNC outer iterations.
Recall that each communication round also corresponds to a single iteration of \cref{alg:collaborative_rotation_averaging} or \cref{alg:collaborative_translation_recovery}.
For this experiment, the sparsification parameter is fixed at $\epsilon=2$.
Overall, for both rotation estimation and translation estimation, 
the RMSE converges after a few GNC outer iterations, and consequently, only a small number of communication rounds is needed.}

\begin{table}[t]
	\centering
	\renewcommand{\arraystretch}{1.1}
	\caption{{Evaluation of robust PGO initialization on real-world collaborative SLAM dataset. }}
	\label{tab:jackal_dataset}
	\resizebox{\columnwidth}{!}{%
		\begin{tabular}{|l|r|r|rr|rr|}
			\hline
			\multicolumn{1}{|c|}{\multirow{2}{*}{Robot}} &
			\multicolumn{1}{c|}{\multirow{2}{*}{Length (m)}} &
			\multicolumn{1}{c|}{\multirow{2}{*}{Keyframes}} &
			\multicolumn{2}{c|}{RMSE without GNC} &
			\multicolumn{2}{c|}{RMSE with GNC} \\ \cline{4-7} 
			\multicolumn{1}{|c|}{} &
			\multicolumn{1}{c|}{} &
			\multicolumn{1}{c|}{} &
			\multicolumn{1}{c|}{Rot. (deg)} &
			\multicolumn{1}{c|}{Tran. (m)} &
			\multicolumn{1}{c|}{Rot. (deg)} &
			\multicolumn{1}{c|}{Tran. (m)} \\ \hline
			1 & 483 & 3192 & \multicolumn{1}{r|}{67.5}  & 11.9 & \multicolumn{1}{r|}{1.8} & 0.5 \\ \hline
			2 & 458 & 2518 & \multicolumn{1}{r|}{73.1} & 17.6 & \multicolumn{1}{r|}{3.3} & 0.7 \\ \hline
			3 & 524 & 3374 & \multicolumn{1}{r|}{86.8}  & 23.6 & \multicolumn{1}{r|}{3.7}  & 0.6 \\ \hline
		\end{tabular}%
	}
    \vspace{-0.3cm}
\end{table}

\begin{table*}[t]
\centering
\renewcommand{\arraystretch}{0.9}
\caption{Robust rotation averaging on real-world SfM datasets.
Each dataset is divided to simulate 5 robots.
$|\Vcal|$ and $|\Ecal|$ denote the total number of rotation variables and measurements, respectively.
Using the reference solution, we quantify the difficulty of each dataset by computing the percentage of high-quality inlier measurements (measurement error $< 5$~deg)	and gross outliers (measurement error $> 45$~deg).			For the proposed method, we show the sparsity achieved by sparsification (lower is better) and total communication.}
\label{tab:sfm}
\resizebox{\textwidth}{!}{%
\begin{tblr}{
  column{even} = {r},
  column{3} = {r},
  column{5} = {r},
  column{7} = {r},
  column{9} = {r},
  column{11} = {r},
  cell{1}{1} = {r=2}{},
  cell{1}{2} = {r=2}{c},
  cell{1}{3} = {r=2}{c},
  cell{1}{4} = {c=3}{c},
  cell{1}{7} = {c=4}{c},
  cell{1}{11} = {r=2}{c},
  cell{1}{12} = {c=2}{c},
  cell{2}{4} = {c},
  cell{2}{5} = {c},
  cell{2}{6} = {c},
  cell{2}{7} = {c},
  cell{2}{8} = {c},
  cell{2}{9} = {c},
  cell{2}{10} = {c},
  cell{2}{12} = {c},
  cell{3}{13} = {r},
  cell{4}{13} = {r},
  cell{5}{13} = {r},
  cell{6}{13} = {r},
  cell{7}{13} = {r},
  cell{8}{13} = {r},
  cell{9}{13} = {r},
  cell{10}{13} = {r},
  cell{11}{13} = {r},
  cell{12}{13} = {r},
  cell{13}{13} = {r},
  cell{14}{13} = {r},
  cell{15}{13} = {r},
  cell{16}{13} = {r},
  cell{17}{13} = {r},
  vlines,
  hline{1,3-18} = {-}{},
  hline{2} = {4-10,12-13}{},
  stretch = 0.8,
}
DATASETS            & $|\Vcal|$ & $|\Ecal|$ & Measurement Quality (\%) &         &       & Mean Error (deg) &        &          &       & {Achieved \\ sparsity (\%)} & Communication (MB) &        \\
                    &           &           & Inlier                   & Outlier & Other & Initial          & No GNC & With GNC & \edit{Theia} &                             & Download           & Upload \\
Montreal Notre Dame & 468       & 49705     & 81                       & 4       & 15    & 4.2              & 3.3    & 1.1      & \edit{1.1}   & 55.1                        & 2.09               & 1.21   \\
Ellis Island        & 241       & 19507     & 63                       & 1       & 26    & 7.0              & 5.6    & 2.3      & \edit{2.4}   & 66.0                        & 0.95               & 0.57   \\
NYC Library         & 355       & 17579     & 61                       & 6       & 33    & 4.3              & 4.3    & 2.3      & \edit{2.5}   & 75.8                        & 1.24               & 0.83   \\
Notre Dame          & 553       & 97764     & 70                       & 9       & 21    & 3.5              & 4.5    & 2.4      & \edit{2.6}   & 41.8                        & 2.85               & 1.54   \\
Roman Forum         & 1099      & 53989     & 74                       & 3       & 23    & 16.5             & 5.1    & 2.5      & \edit{2.5}   & 68.1                        & 4.43               & 2.73   \\
Alamo               & 606       & 87725     & 74                       & 3       & 23    & 8.0              & 4.5    & 2.9      & \edit{2.8}   & 44.7                        & 2.62               & 1.39   \\
Madrid Metropolis   & 379       & 18811     & 47                       & 20      & 33    & 7.7              & 8.0    & 3.4      & \edit{3.4}   & 70.3                        & 1.48               & 1.08   \\
Yorkminster         & 448       & 24416     & 73                       & 5       & 22    & 8.3              & 4.5    & 3.4      & \edit{3.3}   & 76.5                        & 1.69               & 1.1    \\
Tower of London     & 493       & 19798     & 76                       & 3       & 21    & 7.4              & 4.7    & 3.5      & \edit{3.4}   & 75.3                        & 2.00               & 1.35   \\
Trafalgar           & 5433      & 680012    & 63                       & 7       & 30    & 20.0             & 6.4    & 3.5      & \edit{3.3}   & 40.7                        & 35.89              & 24.97  \\
Piazza del Popolo   & 343       & 22342     & 82                       & 4       & 14    & 5.4              & 7.8    & 3.6      & \edit{3.5}   & 70.0                        & 1.29               & 0.8    \\
Piccadilly          & 2436      & 254175    & 58                       & 10      & 32    & 13.9             & 14.6   & 4.9      & \edit{5.0}   & 53.1                        & 14.23              & 9.79   \\
Union Square        & 930       & 25561     & 57                       & 6       & 37    & 11.9             & 10.9   & 6.0      & \edit{8.6}   & 82.1                        & 4.15               & 3.37   \\
Vienna Cathedral    & 900       & 96546     & 70                       & 6       & 24    & 13.9             & 9.6    & 8.9      & \edit{8.6}   & 51.9                        & 4.62               & 3.13   \\
Gendarmenmarkt      & 723       & 42980     & 36                       & 27      & 37    & 45.0             & 40.8   & 38.1     & \edit{38.1}  & 63.8                        & 4.13               & 3.03   
\end{tblr}
}
\end{table*}

{\subsection{Evaluation on Real-World Structure-from-Motion Datasets}
	\label{sec:experiments:sfm}
}

{Lastly, we evaluate our method on rotation averaging problems extracted from $15$ real-world structure-from-motion (SfM) datasets \cite{Wilson141dsfm}.
Each dataset is a collection of many internet images taken at a particular location.
We use Theia~\cite{TheiaSfM} to process each dataset and extract a rotation averaging problem with outliers (caused by incorrect feature matching).
As ground truth is not available, we follow \cite{TheiaSfM} and use 3D reconstructions produced by the incremental SfM pipeline \cite{Snavely06PhotoTourism} as reference solutions.
\cref{tab:sfm} reports full dataset statistics.}

\edit{Based on the image IDs, we equally partition each dataset to
simulate 5 robots and run our GNC-based rotation averaging solver, with the TLS threshold set to $5$~deg.
Compared to collaborative SLAM, in SfM there is a significantly larger number of inter-robot measurements.
Furthermore,}
each robot no longer has an outlier-free odometry backbone in SfM.
This means that we cannot use the approach of \cite{Tian21KimeraMulti} to compute an outlier-free initial guess for the variable update step in GNC (see \cref{rm:gnc_implementation}).
Instead, we use the initial guess from Theia that is computed using a spanning tree of the measurement graph.
\cref{tab:sfm} reports the mean estimation error. 
\edit{On $14$ out of the $15$ datasets, our GNC-based method produces accurate results with performance on par with the centralized Theia library.}
The only failure case, \Gendarmenmarkt, is known to be a very challenging case in which the underlying 3D scene is highly symmetric, leading to a significantly lower percentage of inlier measurements.
\edit{The initial guess has a large error, which both GNC and Theia are unable to recover from.}
\edit{This issue could potentially be addressed with a better initialization method (\eg, using pairwise consistency maximization \cite{Mangleson2018PCM}), which we leave for future work.}
In summary, we conclude that on most datasets, our proposed rotation averaging solver combined with GNC is able to achieve robust rotation estimation, despite outlier measurements and the increased noise level present in internet images.

\begin{figure*}[t]
	\centering
	\begin{subfigure}[t]{0.32\linewidth}
		\includegraphics[width=\textwidth, trim=100 0 80 0, clip]{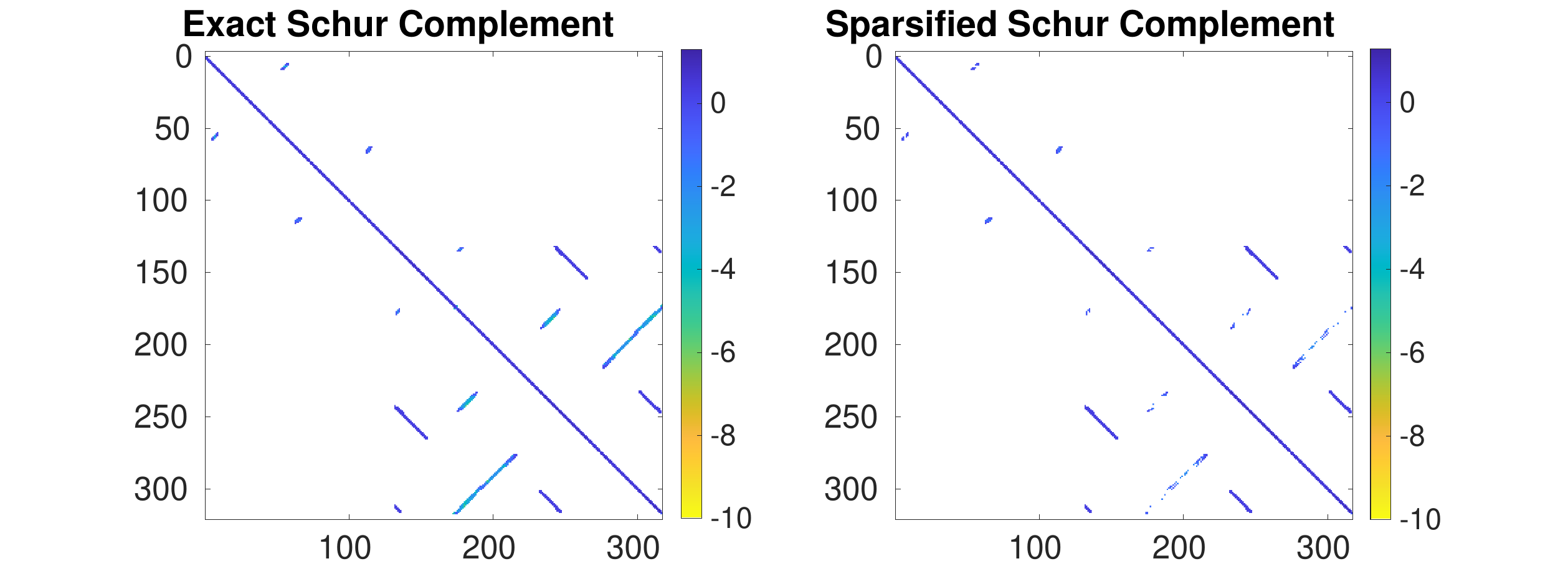}
		\caption{{\dataset{Garage} ($\epsilon=1.5$)}}
		\label{fig:discussion:garage}
	\end{subfigure}
	~
	\begin{subfigure}[t]{0.32\linewidth}
		\includegraphics[width=\textwidth, trim=100 0 80 0, clip]{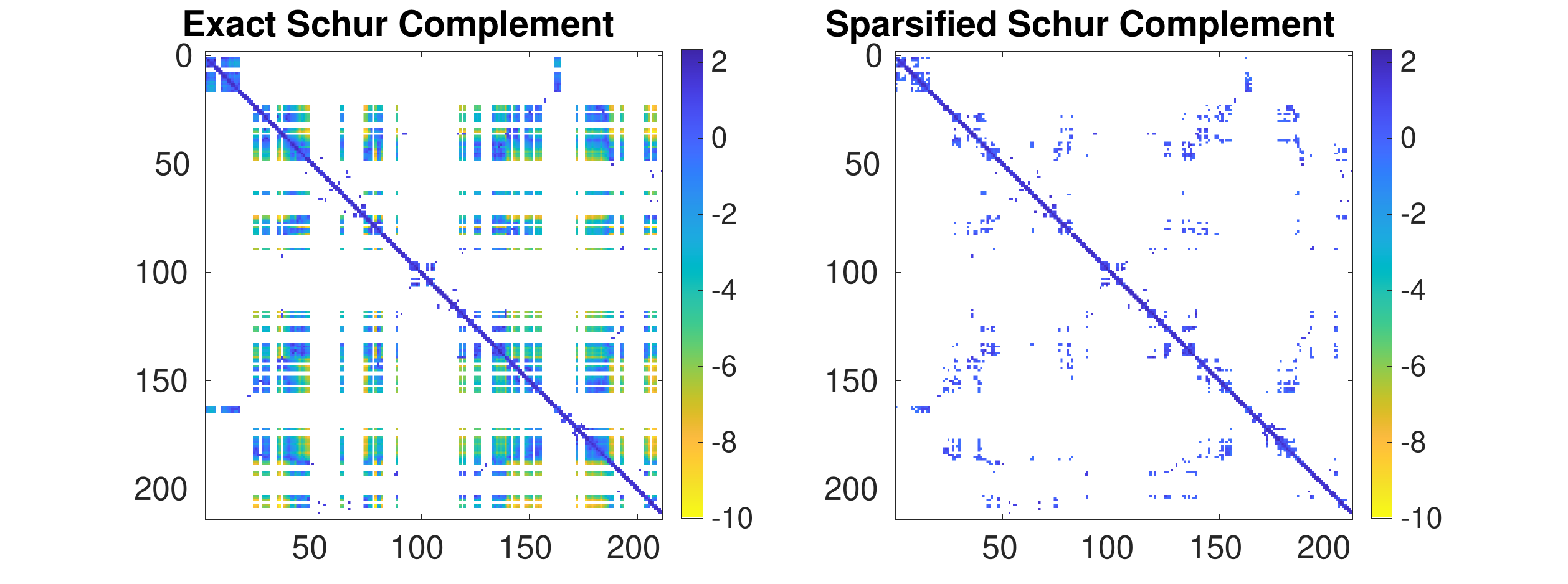}
		\caption{{\dataset{Manhattan} ($\epsilon=1.5$)}}
		\label{fig:discussion:manhattan}
	\end{subfigure}
	~
	\begin{subfigure}[t]{0.32\linewidth}
		\includegraphics[width=\textwidth, trim=100 0 100 0, clip]{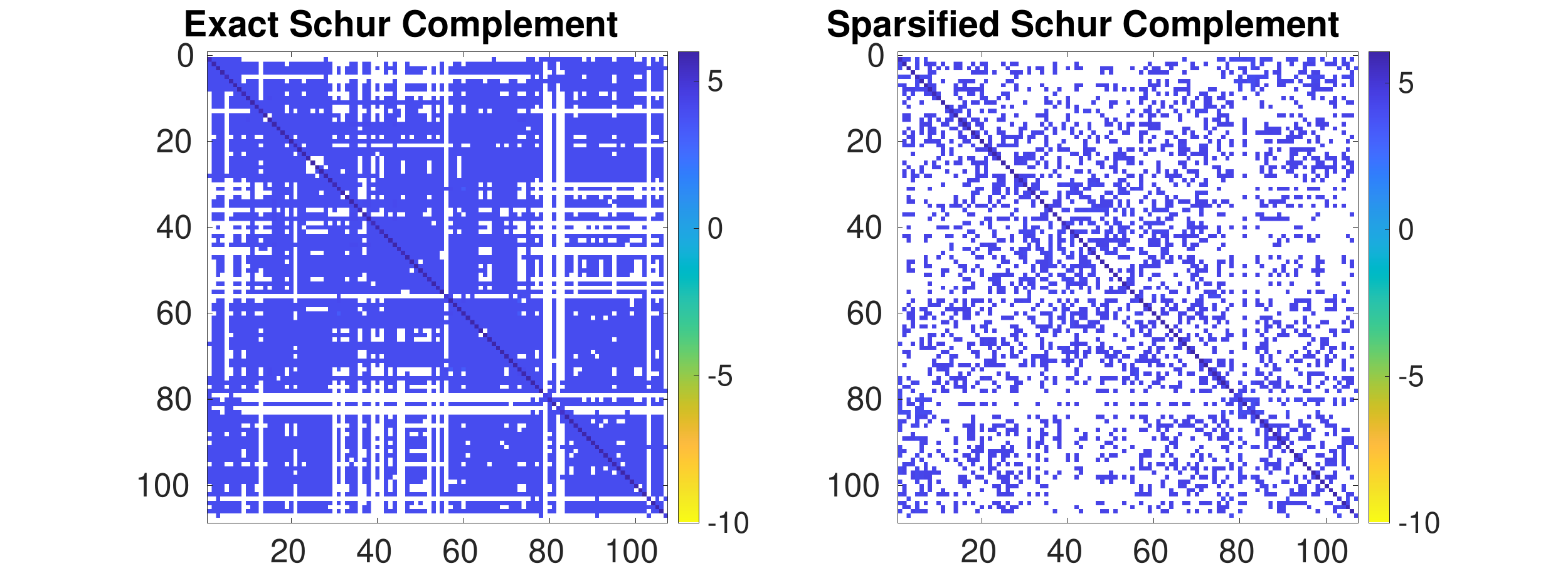}
		\caption{{\dataset{Notre Dame} ($\epsilon=5.0$)}}
		\label{fig:discussion:notre_dame}
	\end{subfigure}
	\caption{{Impact of the density of exact Schur complements on the performance of spectral sparsification. For each dataset, we select one robot and visualize the sparsity pattern of its exact Schur complement (corresponding to $S_\alpha$ in \cref{alg:sparsified_schur_complement}), and the result after spectral sparsification (corresponding to $\Stilde_\alpha$ in \cref{alg:sparsified_schur_complement}).
		Entries in the matrix are color-coded based on their magnitude in log scale.} } 
	\label{fig:discussion}
    \vspace{-0.5cm}
\end{figure*}

{We report the performance of spectral sparsification and the total communication costs of our method.
For our SfM experiment, we increase the sparsification parameter to $\epsilon=5$.
{In \cref{sec:experiments:discussion}, we explain the reasons behind using the increased value for $\epsilon$. }
\cref{tab:sfm} shows the achieved sparsity as the average ratio between the number of nonzero elements in the sparsified matrix and the input (dense) matrix.
On all datasets, spectral sparsification significantly improves sparsity to as low as $40.7\%$ on the largest \dataset{Trafalgar} dataset.
These results, together with the total amounts of uploads and downloads, demonstrate the effectiveness of our approach to achieve communication efficiency. }

{\subsection{Discussion} 
\label{sec:experiments:discussion}}

{We conclude our experimental evaluations by discussing the impacts of real-world problem properties on the performance of the proposed algorithms.}

{\myParagraph{Effectiveness of Laplacian Approximation in the Presence of Outliers}}
{Our rotation averaging method exploits the fact that under small measurement residuals, the Laplacian is an effective approximation of the Hessian (\cref{thm:Hessian_approximation}).
	When there are outlier measurements, we have seen that the approximation quality degrades, leading to increased number of iterations.
	An example is the \dataset{Rim} dataset in \cref{tab:rotation_averaging}, which is contaminated by outliers.
	Nonetheless, we note that this issue is mitigated when using a robust optimization framework such as GNC, since outliers will be gradually discounted and eventually rejected from the measurement graph.
	This is shown in \cref{fig:jackal_dataset:rotation_rmse}.
	During the first two GNC outer iterations, outliers have a substantial influence on the problem, causing our method (\cref{alg:collaborative_rotation_averaging}) to use more communication rounds.
	However, as GNC proceeds, outliers receive increasingly small weights, and our method recovers its fast convergence.
	In \cref{fig:jackal_dataset:rotation_rmse}, this is shown as the slower increase in communication rounds starting from the third GNC outer iteration.}

{\myParagraph{Impact of Problem Density on Sparsification Performance}}
{As we have seen (\eg, from \cref{tab:rotation_averaging}), spectral sparsification achieves different levels of sparsity improvement on the various SLAM and SfM datasets.
	This is because in our method, sparsification is applied to the Schur complements that the robots form after eliminating their interior variables (see \cref{alg:sparsified_schur_complement}).
	Thus, we expect the performance of sparsification to vary depending on the density of the Schur complements.
	To make the discussion more concrete, we identify three types of problems and \cref{fig:discussion} shows a representative sparsification result for each case.
	In the first case (\cref{fig:discussion:garage}), the multi-robot measurement graph is extremely sparse; consequently, the resulting Schur complements are already sparse and sparsification is not necessary.
	In the second case, the original measurement graph is still sparse, but the robots' Schur complements become dense due to \emph{fill-in} introduced during the elimination of interior variables.
	For the example in \cref{fig:discussion:manhattan}, the fill-in is visualized as patches of dense entries in the exact Schur complement, and our method is highly effective at sparsifying these dense blocks.
        Moreover, notice that the dense fill-ins have relatively smaller magnitudes (\eg, compared to the diagonal), and thus they can be sparsified with a smaller value of
        the sparsification parameter $\epsilon$.
	In the last case, the original measurement graph is already dense and so are the resulting Schur complements (\cref{fig:discussion:notre_dame}).
	All of the SfM datasets in \cref{tab:sfm} belong to this category because there are many images viewing a common landmark (\eg, the Notre Dame cathedral), albeit from different locations or angles.
	Consequently, a relative rotation can be estimated for many image pairs, which makes the input measurement graph dense.
        Since there is no significant difference in the magnitudes of different matrix entries, 
        a larger value of $\epsilon$ is needed.
	Similar to the second case, sparsification is highly effective at promoting sparsity in each robot's transmitted matrix in this case.}

\section{Conclusion}

We presented fast and communication-efficient methods for solving rotation averaging and translation estimation in multi-robot SLAM, {SfM}, and camera network localization applications.
Our algorithms leverage theoretical relations between the Hessians of the optimization problems and the Laplacians of the underlying graphs.
{At each iteration, robots coordinate with a central server to perform approximate second-order optimization, while using spectral sparsification to achieve communication efficiency.}
We performed rigorous analysis of our methods and proved that they achieve (local) linear rate of convergence.
{Furthermore, we proposed the combination of our solvers with GNC to achieve outlier-robust estimation.}
Extensive experiments {in real-world collaborative SLAM and SfM scenarios} validate our theoretical results and demonstrate the superior convergence rate and communication efficiency of our proposed methods.

\edit{While results are promising, this work also suggests several directions for future research. First,}
it remains an open problem whether a similar approach can be developed for the full PGO problem.
Our preliminary analysis shows that, unlike rotation averaging, 
the Hessian of PGO is no longer well approximated by the corresponding graph Laplacian due to the coupling between rotation and translation terms.
As a result, our current approach cannot be directly applied and additional techniques need to be considered.
\edit{Secondly, the proposed algorithms assume the availability of communication links that allow all robots to participate in collaborative optimization.
When some robots go offline, a practical remedy is to temporarily exclude them from optimization, 
but this could lead to decrease in the overall accuracy.
A principled extension to cope with communication failures and evaluations under more realistic scenarios (\eg, using real-world communication protocols) would be valuable.}
\edit{Lastly,} extending the algorithm to leverage the incremental nature of real-world SLAM problems is another interesting direction for future research.

{\small
	\bibliographystyle{unsrt}
	\bibliography{ref/main,ref/yulun}
}

\clearpage
\onecolumn
{
\appendices
\begin{table}[t]
	\centering
	\renewcommand{\arraystretch}{1.2}
	\caption{Summary of key notations used in this work (organized by sections).}
	\label{tab:notations}
	\resizebox{1.0\textwidth}{!}{%
		\begin{tabular}{cll}
			\hline
			{\bfseries Notation}                    & \multicolumn{1}{c}{ \bfseries Description} & \multicolumn{1}{c}{ \bfseries Reference} \\ \hline 
			\multicolumn{3}{c}{\bfseries Section~\ref{sec:problem_definition}}                                                             \\ \hline
			\multicolumn{1}{l}{$G=(\Vcal,\Ecal)$}       & Multi-robot measurement graph with vertex (variable) set $\Vcal$ and edge  (measurement) set $\Ecal$     &                              \\
			\multicolumn{1}{l}{$R_i$}                              & The $i$th rotation variable to be estimated in rotation averaging                                      &             \eqref{eq:rotation_averaging}                 \\
			\multicolumn{1}{l}{$\Rtilde_{ij}$}                 & Noisy relative rotation measurement in rotation averaging                               &             \eqref{eq:rotation_averaging}                 \\
			\multicolumn{1}{l}{$\varphi(\cdot, \cdot)$}      &  Squared geodesic or chordal distance  function between two rotations                                                      &             \eqref{eq:squared_geodesic_distance}-\eqref{eq:squared_chordal_distance}                 \\
			\multicolumn{1}{l}{$t_i$}                               & The $i$th translation variable to be estimated in translation estimation                          &             \eqref{eq:translation_recovery}                 \\
			\multicolumn{1}{l}{$\that_{ij}$}                    & Noisy relative translation measurement in translation estimation                               &     \eqref{eq:translation_recovery}                 \\
			\multicolumn{1}{l}{$\ttilde_{ij}$}                  & Noisy relative translation measurement in PGO                                                    &     \eqref{eq:pose_graph_optimization}                 \\
			\multicolumn{1}{l}{$\kappa_{ij}$}                & Weight (precision) associated with the relative rotation measurement between vertex $i$ and $j$                     &     \eqref{eq:rotation_averaging}, \eqref{eq:pose_graph_optimization}                 \\
			\multicolumn{1}{l}{$\tau_{ij}$}                    & Weight (precision) associated with the relative translation measurement between vertex $i$ and $j$                     &     \eqref{eq:translation_recovery}, \eqref{eq:pose_graph_optimization}                 \\ \hline
			\multicolumn{3}{c}{\bfseries Section~\ref{sec:laplacian_systems}}           \\ \hline
			\multicolumn{1}{l}{$p$}                     & $p \triangleq \dim \SOd(d)$ is the intrinsic dimension of the rotation group $\SOd(d)$                              &        \\ 
			\multicolumn{1}{l}{$[R]$}                    & The equivalent class corresponding to $n$ rotations $R = (R_1, \hdots, R_n) \in \SOd(d)^n$                              &     \eqref{eq:rotation_averaging_equivalent_class}     \\  
			\multicolumn{1}{l}{$v_i$}                    & $v_i \in \Real^{p}$ is the correction vector (to be optimized) for rotation variable $R_i$                          &       \\ 
			\multicolumn{1}{l}{$v$}                    & $v \in \Real^{pn}$ is formed by concatenating the $v_i$ vectors of all $n$ rotation variables                            &     \eqref{eq:RA_full_pullback_global_def}     \\ 
			\multicolumn{1}{l}{$V$}                    & $V \in \Real^{n \times p}$ is the matrix representation of $v$                             &     \eqref{eq:rotation_averaging_matrix_variables}     \\ 
			\multicolumn{1}{l}{$\Ncal, \Hcal$}         & Subspaces of $\Real^{pn}$ corresponding to the vertical space and horizontal space  in rotation averaging  ($\Hcal \triangleq \Vcal^\perp$)            &     \eqref{eq:vertical_space_RA}     \\ 
			\multicolumn{1}{l}{$P_H$}         & Orthogonal projection onto the horizontal space $\Hcal$                            &     \eqref{eq:Newton_step_quotient}     \\ 
			\multicolumn{1}{l}{$\gbar(R)$}       & $\gbar(R) \in \Real^{pn}$ is the vector corresponding to the Riemannian  gradient of rotation averaging in the total space        &     \eqref{eq:gradient_and_hessian_def}     \\  
			\multicolumn{1}{l}{$\Hbar(R)$}       & $\Hbar(R) \in \Sym^{pn}$ is the matrix corresponding to the Riemannian  Hessian of rotation averaging in the total space        &     \eqref{eq:gradient_and_hessian_def}     \\  
			\multicolumn{1}{l}{$H(R)$}        & $H(R) \in \Sym^{pn}$ is the matrix corresponding to the Riemannian Hessian in the quotient space      &     \eqref{eq:Newton_step_quotient}     \\  
			\multicolumn{1}{l}{$w$}             & $w: \Ecal \to \Real_{>0}$ is the edge weight function that appears in \cref{thm:Hessian_approximation}   &   \eqref{thm:Hessian_approximation:spectral_approximation}     \\          
			\multicolumn{1}{l}{$\delta$}       & The approximation constant in \cref{thm:Hessian_approximation} between $H(R)$ and the Laplacian $L(G; w) \otimes I_p$      &   \eqref{thm:Hessian_approximation:spectral_approximation}     \\                               
			\multicolumn{1}{l}{$\mu_H$}       & Lower bound of $H(R)$  in  \cref{cor:Hessian_bound} &   \eqref{eq:Hessian_bound}     \\   
			\multicolumn{1}{l}{$L_H$}            & Upper bound of $H(R)$  in  \cref{cor:Hessian_bound} &   \eqref{eq:Hessian_bound}     \\     
			\multicolumn{1}{l}{$\kappa_H$}       & Condition number of $H(R)$  as defined by $\kappa_H = L_H / \mu_H$         &  Cor.~\ref{cor:Hessian_bound}     \\                                  
			\multicolumn{1}{l}{$B(R)$}         & $B(R) \in \Real^{n \times p}$ is the matrix representation of the negative Riemannian gradient $-\gbar(R)$                      &     \eqref{eq:rotation_averaging_matrix_variables}     \\  \hline
			\multicolumn{3}{c}{\bfseries Section~\ref{sec:algorithm}}           \\ \hline
			\multicolumn{1}{l}{$\Vcal_\alpha$}     & The set of vertices (variables) of robot $\alpha \in [m]$, and  $\Vcal_\alpha = \Fcal_\alpha \uplus \Ccal_\alpha$       &  \eqref{eq:vertex_partition}        \\   
			\multicolumn{1}{l}{$\Fcal_\alpha$}     &  $\Fcal_\alpha \subseteq \Vcal_\alpha$ is the set of interior vertices of robot $\alpha$   that does not have inter-robot measurement    &     \\   
			\multicolumn{1}{l}{$\Ccal_\alpha$}     &  $\Ccal_\alpha \subseteq \Vcal_\alpha$ is the set of separator vertices of robot $\alpha$  hat have inter-robot measurement    &     \\  
			\multicolumn{1}{l}{$\Ccal$}     &  $\Ccal = \Ccal_1 \uplus \hdots \uplus \Ccal_m$ is the union of separator vertices of all $m$ robots  &     \\  
			\multicolumn{1}{l}{$\Ecal_\alpha$}       & The set of local edges (measurements) of robot $\alpha \in [m]$     &  \eqref{eq:edge_partition}      \\   
			\multicolumn{1}{l}{$\Ecal_c$}       & The set of inter-robot edges (measurements)      &  \eqref{eq:edge_partition}      \\   
			\multicolumn{1}{l}{$S_\alpha$}       & $S_\alpha \in \PSD^{|\Ccal|}$ is the exact Schur complement of robot $\alpha$'s local graph $G_\alpha$  &     \\   
			\multicolumn{1}{l}{$S$}       & $S \in \PSD^{|\Ccal|}$ is the exact Schur complement of the multi-robot measurement graph &  \eqref{eq:schur_complement_sum}    \\   
			\multicolumn{1}{l}{$\Stilde_\alpha$}   & The sparsified version of $S_\alpha$ robot $\alpha$ transmits to the server in \cref{alg:sparsified_schur_complement}  &  Alg.~\ref{alg:sparsified_schur_complement}, line~\ref{alg:sparsified_schur_complement:upload}   \\   
			\multicolumn{1}{l}{$\Stilde$}   & The sparsified version of $S$ computed by the server in \cref{alg:sparsified_schur_complement}  &  Alg.~\ref{alg:sparsified_schur_complement}, line~\ref{alg:sparsified_schur_complement:server_sum}   \\  
			\multicolumn{1}{l}{$U_\alpha$}       & $U_\alpha \in \Real^{|\Ccal| \times p}$ is the block vector robot $\alpha$ transmits to the server in \cref{alg:sparsified_laplacian_solver}  &   Alg.~\ref{alg:sparsified_laplacian_solver}, line~\ref{alg:sparsified_laplacian_solver:upload}  \\   
			\multicolumn{1}{l}{$\epsilon$}   & The spectral sparsification parameter that is used in the algorithm and appears in \cref{thm:approximation_guarantees} &   \\  
			\multicolumn{1}{l}{$\TLS$}   & The truncated least squares (TLS) cost function for outlier-robust estimation &  \eqref{eq:robust_estimation}  \\  
			\multicolumn{1}{l}{$e_{ij}$}   & Measurement error corresponding to the measurement $(i,j) \in \Ecal$ in the measurement graph \\  
			\multicolumn{1}{l}{$\ebar$}   & Threshold that specifies the maximum error of inlier measurement in TLS &    \\  
			\multicolumn{1}{l}{$\mu$}   & Control parameter of graduated non-convexity  (GNC) &    \\  
			\multicolumn{1}{l}{$\wgnc_{ij}$}   & GNC weight for measurement $(i,j) \in \Ecal$ in the measurement graph &  \eqref{eq:gnc_surrogate_func}   \\  
			\hline
		\end{tabular}%
	}
\end{table}

\clearpage
\section{Details of Spectral Sparsification Algorithm}
\label{sec:spectral_sparsification}

\begin{algorithm}[t]
	\caption{\textsc{Spectral Sparsification by Effective Resistance Sampling}}
	\label{alg:effective_resistance_sampling}
	\begin{algorithmic}[1]
		\small
		\For{each edge $(i,j) \in \Ecal$ in the graph $G$ corresponding to the input Laplacian $L$ (in parallel)}
		\State Compute leverage score $\ell_{ij} = w_{ij} (\indvector_i - \indvector_j)^\top L\pinv (\indvector_i - \indvector_j).$
		\State Select this edge with probability $p_{ij} = \min(1, \; {3.5 \log n}  \ell_{ij} / {\epsilon_l^2} )$.
		\State If this edge is selected, add it to the sparsified graph $\Gtilde$ with increased edge weight $w_{ij} / p_{ij}$.
		\EndFor
		\State \Return The Laplacian $\Ltilde$ of the sparse graph $\Gtilde$.
	\end{algorithmic}
\end{algorithm}

In this appendix,  we provide details of the spectral sparsification algorithm used in this work.
Given the Laplacian matrix $L$ of a dense graph $G$, recall that the goal of spectral sparsification is to find a \emph{sparse} Laplacian $\Ltilde$ such that,
\begin{equation}
	e^{-\epsilon} L \preceq \Ltilde \preceq e^\epsilon L,
	\label{eq:spectral_approximation_def_copy}
\end{equation}
where $\epsilon > 0$ is the desired sparsification parameter. 
Note that \eqref{eq:spectral_approximation_def_copy} is the same definition as \eqref{eq:spectral_approximation_def} in the main paper.
In graph terms, this is the same as finding a sparse graph $\Gtilde$ whose Laplacian approximates that of the dense graph $G$.

In this work, we use the random sampling approach developed by Spielman and Srivastava \cite{Spielman2011Sampling}.
In particular, we implement the improved version presented in \cite[Chapter~32]{Spielman2019Notes}, which is more suitable for batch computation since it avoids sampling with replacement and the decision to keep or remove each edge can be made in parallel.
Given a constant $\epsilon_l \in (0,1)$, this method produces a sparse $\Ltilde$ with $O(n\log n)$ entries (where $n$ is the number of vertices in the graph) such that with high probability,
\begin{equation}
	(1-\epsilon_l) L \preceq \Ltilde \preceq (1+\epsilon_l) L.
	\label{eq:spectral_approximation_def_linear_scale}
\end{equation}
Note that \eqref{eq:spectral_approximation_def_linear_scale} can be used to ensure that \eqref{eq:spectral_approximation_def_copy} holds:
in our implementation, given $\epsilon$, we find the smallest $\epsilon_l$ such that \eqref{eq:spectral_approximation_def_copy} holds, which is given by,
\begin{equation}
	\epsilon_l = \min(e^\epsilon - 1, 1 - e^{-\epsilon}).
\end{equation}
The sparsification algorithm works by selecting edges in the input dense graph $G$ based on their \emph{leverage scores}.
Recall that each edge corresponds to a non-zero off-diagonal term in the Laplacian $L$, and thus selecting a small subset of edges leads to a sparse output Laplacian $\Ltilde$.
For each edge $(i,j) \in \Ecal$, its leverage score is defined as,
\begin{equation}
	\ell_{ij} \triangleq w_{ij} (\indvector_i - \indvector_j)^\top L\pinv (\indvector_i - \indvector_j),
	\label{eq:leverage_score}
\end{equation}
\begin{wrapfigure}{r}{0.3\textwidth}
	\begin{center}
		\centering
		\includegraphics[width=0.25\textwidth, trim=0 0 0 0, clip]{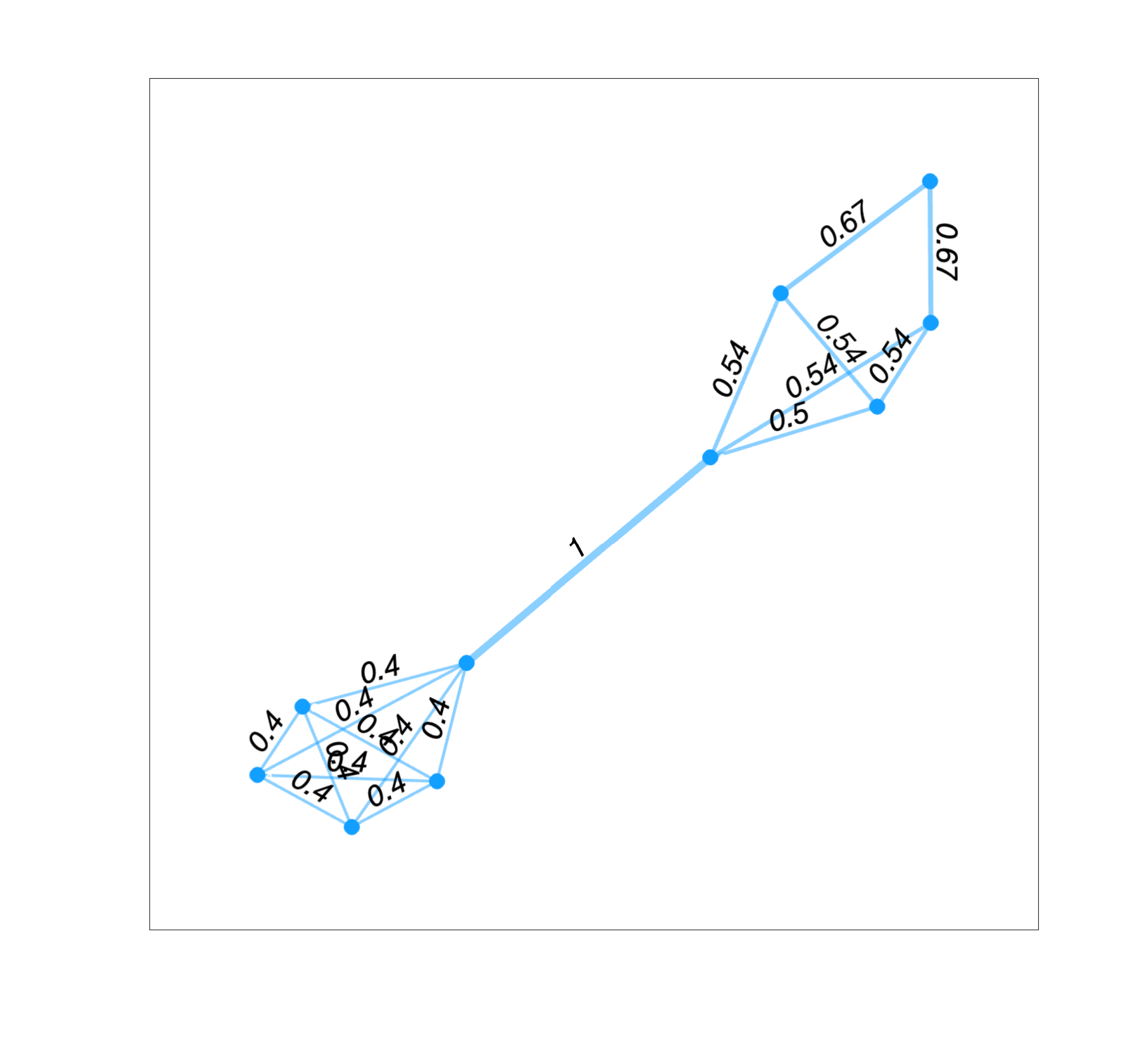}
		\caption{{Leverage scores on toy graph.}}
		\label{fig:leverage_score_example}
	\end{center}
\end{wrapfigure}
where $\indvector_i \in \Real^n$ is the $i$-th basis vector with a one in coordinate $i$, and $w_{ij} > 0$ is the edge weight.
The quantity $\ell_{ij} / w_{ij}$ is also known as the \emph{effective resistance}.
Intuitively, the leverage score measures the importance of each edge to the connectivity of the overall graph. 
\cref{fig:leverage_score_example} shows an illustration on a toy graph consisting of two clusters of vertices connected by a single edge.
All edges have unit weights, and each edge is labeled by its leverage score computed according to \eqref{eq:leverage_score}.
Notice that the middle edge has the highest leverage score, since it is critical to keep the overall graph connected.
In comparison, the remaining edges have lower leverage scores, due to the redundancy of edges in each cluster.
The actual sparsification algorithm is a remarkably simple method, which independently selects each edge with a probability proportional to its leverage score:
\begin{equation}
	p_{ij} = \min\left(1,  \; \frac{3.5 \log n}{\epsilon_l^2}   \ell_{ij}\right).
\end{equation}
If edge $(i,j) \in \Ecal$ is selected, we add it to the sparsified graph $\Gtilde$ with an \emph{increased edge weight} $w_{ij} / p_{ij}$.
The reason behind increasing the edge weight is to ensure that we can recover $L$ in expectation, \ie, $\mathbb{E}(\Ltilde) = L$.
\cref{alg:effective_resistance_sampling} shows the pseudocode.
The majority of computation is spent on factorizing $L$ to compute the leverage scores in \eqref{eq:leverage_score}.
The approximation guarantee \eqref{eq:spectral_approximation_def_linear_scale} of the resulting $\Ltilde$ is proved using certain concentration bounds of random matrices.
The interested reader is referred to \cite[Chapter~32]{Spielman2019Notes} for the complete proof.

Lastly, we refer the reader to \cref{fig:discussion} in the main paper, which demonstrates spectral sparsification on Laplacian matrices encountered in our application.
Recall that in our case, we apply sparsification on the Schur complement $S_\alpha$ of robot $\alpha$'s local Laplacian.
Since Laplacians are closed under Schur complements~\cite[Fact~4.2]{lee2015sparsified}, $S_\alpha$ is still a Laplacian matrix and thus the sparsified result $\Stilde_\alpha$ retains all the theoretical guarantees.

\section{Analysis of Riemannian Hessian of Rotation Averaging}
\label{sec:rotation_averaging_hessian_analysis}

In this appendix, we use $\distang \equiv \dist$ to denote the geodesic distance on the rotation group, 
and use $\distchr$ to denote the chordal distance.
Recall the definition of $\varphi$ in \cref{prob:rotation_averaging} as either 
the squared geodesic distance  or the squared chordal distance,
\begin{subnumcases}{\varphi(R_1, R_2) =}
	\frac{1}{2} \distang(R_1, R_2)^2 =
	\frac{1}{2} \norm{\Log(R_{1}^\top R_2)}^2, \quad &\text{squared geodesic distance}, 
	\label{eq:squared_geodesic_distance_apx}
	\\
	\frac{1}{2} \distchr(R_1, R_2)^2 =
	\frac{1}{2} \norm{R_1 - R_2}^2_F, &\text{squared chordal distance}.
	\label{eq:squared_chordal_distance_apx}
\end{subnumcases}
Using the notion of ``reshaped distance'' introduced in \cite{Tron2012Thesis},
we may express both cases as a function of the geodesic distance as follows,
\begin{equation}
	\varphi(R_1, R_2) = \rho(\distang(R_1,R_2)).
	\label{eq:reshaped_distance}
\end{equation}
{Note that the notation for reshaped distance $\rho$ is not to be confused with $\TLS$ in \cref{sec:gnc}, which instead denotes the truncated least squares function.}
It can be shown that the scalar function $\rho(\cdot)$ is defined as,
\begin{subnumcases}{\rho(\theta) = }
	\theta^2 / 2, \quad &\text{for squared geodesic distance \eqref{eq:squared_geodesic_distance_apx}}, 
	\label{eq:reshape_func_angular}
	\\
	2 - 2\cos(\theta), &\text{for squared chordal distance \eqref{eq:squared_chordal_distance_apx}}.
	\label{eq:reshape_func_chordal}
\end{subnumcases}
The first case \eqref{eq:reshape_func_angular} is readily verified.
The second case \eqref{eq:reshape_func_chordal} makes use of the relation between chordal and geodesic distances; see \cite[Table~2]{Hartley2013RA}.
To analyze the Hessian of rotation averaging \eqref{eq:rotation_averaging}, 
we start by considering the cost associated with a single relative rotation measurement,
\begin{equation}
	f_{ij}(R_i, R_j) = \varphi(R_i \Rtilde_{ij}, R_j).
	\label{eq:RA_term}
\end{equation}
In addition to \eqref{eq:RA_term}, we also consider its approximation defined on the Lie algebra,
\begin{equation}
	h_{ij}(v_i, v_j) \triangleq \varphi(\Exp(v_i) R_i \Rtilde_{ij}, \Exp(v_j) R_j).
	\label{eq:RA_term_pullback_global}
\end{equation}
Note that \eqref{eq:RA_term_pullback_global} corresponds to a single term in the overall approximation 
defined in \eqref{eq:RA_full_pullback_global_def}. 
Similar to \eqref{eq:RA_full_pullback_global_def}, $h_{ij}$ depends on the current rotation estimates $R$, but we omit this from our notation for simplicity. 
Define the gradient and Hessian of $h_{ij}$ as follows,
\begin{align}
	\gbar_{ij}  &\triangleq \nabla h_{ij}(v_i, v_j) \rvert_{v_i = v_j = 0}, \\
	\Hbar_{ij} &\triangleq \nabla^2 h_{ij}(v_i, v_j) \rvert_{v_i = v_j = 0}.
\end{align}

In the following, we first derive auxiliary results that characterize $\gbar_{ij}$ and $\Hbar_{ij}$.
Once we understand the properties of $\gbar_{ij}$ and $\Hbar_{ij}$, 
understanding the full rotation averaging problem becomes straightforward thanks to the additive structure in the cost function \eqref{eq:rotation_averaging}.
\edit{Note that we prove the main theoretical results \cref{thm:Hessian_approximation} and \cref{cor:Hessian_bound} for 3D rotation averaging.
The case of $d=2$ can be proved using the exact same arguments, and some steps would simplify due to the fact that 2D rotations commute.}

\subsection{Auxiliary Results for 3D Rotation Averaging}
\label{sec:rotation_averaging_hessian_analysis:auxliary}

\begin{lemma}
	Consider a 3D rotation averaging problem.
	Let 
	\begin{align}
	\theta_{ij} &= \norm{\Log(\Rtilde_{ij}^\top R_i^\top R_j)}, \\
	u_{ij} &= \Log(\Rtilde_{ij}^\top R_i^\top R_j) / \theta_{ij},
	\end{align}
	denote the angle-axis representation of the current rotation error.
	Then the gradient is given by,
	\begin{equation}
		\gbar_{ij} = 
		\dot{\rho}(\theta_{ij}) 
		\begin{bmatrix}
		R_i \Rtilde_{ij}  & 0 \\
		0 & R_j
		\end{bmatrix}
		\begin{bmatrix}
		-u_{ij} \\
		u_{ij}
		\end{bmatrix}.
		\label{eq:RA_term_gradient_3d}
	\end{equation}
	The Hessian is given by,
	\begin{equation}
		\Hbar_{ij} = 
		\begin{bmatrix}
		R_i \Rtilde_{ij}  & 0 \\
		0 & R_j
		\end{bmatrix}
		\begin{bmatrix}
		\Sym(\Htilde_{ij})  & -\Htilde_{ij} \\
		-\Htilde_{ij}^\top  & \Sym(\Htilde_{ij})
		\end{bmatrix}
		\begin{bmatrix}
		R_i \Rtilde_{ij}  & 0 \\
		0 & R_j
		\end{bmatrix}^\top,
		\label{eq:RA_term_hessian_3d}
	\end{equation}
	where $\Htilde_{ij} = \alpha(\theta_{ij}) I
	+ \gamma(\theta_{ij})  u_{ij}u_{ij}^\top 
	+ \beta(\theta_{ij}) \hatop{u_{ij}}$ with
	\begin{align}
		\alpha(\theta_{ij}) &= \frac{\dot{\rho}(\theta_{ij})\cot(\theta_{ij}/2)}{2}, 
		\label{eq:alpha_func} \\
		\gamma(\theta_{ij}) &= \ddot{\rho}(\theta_{ij}) - \alpha(\theta_{ij}), 
		\label{eq:gamma_func}\\
		\beta(\theta_{ij}) &= \frac{\dot{\rho}(\theta_{ij})}{2}.
		\label{eq:beta_func}
	\end{align}
	\label{lem:RA_term_3d}
\end{lemma}
\begin{proof}
	Introduce new rotation variables $S_i, S_j \in \SOd(3)$ and consider the following function,
	\begin{equation}
		\hhat_{ij}(S_i, S_j) = \varphi(S_i R_i \Rtilde_{ij},  S_j R_j).
		\label{eq:RA_term_2}
	\end{equation}
	Note that $\gbar_{ij}$ and $\Hbar_{ij}$ correspond to the Riemannian gradient and Riemannian Hessian of \eqref{eq:RA_term_2} evaluated at $S_i = S_j = I$. 
	Define $F: \SOd(3) \times \SOd(3) \to \SOd(3) \times \SOd(3)$ be the mapping such that,
	\begin{equation}
		F(S_i, S_j) = (S_i R_i \Rtilde_{ij},  S_j R_j) \triangleq (\Shat_i, \Shat_j).
	\end{equation}
	Then we have 
	$\hhat_{ij}(S_i, S_j) = \varphi(F(S_i, S_j))$.
	By chain rule, 
	\begin{align}
		\rgrad \hhat_{ij}(S_i, S_j) 
		&= DF(S_i, S_j)^\top [ \rgrad \varphi(\Shat_i, \Shat_j)]
		\label{eq:RA_term_gradient_3d_chain_rule_init} \\
		&= DF(S_i, S_j)^\top [ \dot{\rho}(\theta_{ij})  \rgrad \distang(\Shat_i, \Shat_j)] \\
		&= \dot{\rho}(\theta_{ij}) DF(S_i, S_j)^\top [  \rgrad \distang(\Shat_i, \Shat_j)].
		\label{eq:RA_term_gradient_3d_chain_rule}
	\end{align}
	In \eqref{eq:RA_term_gradient_3d_chain_rule}, $DF(S_i, S_j)^\top$ stands for the adjoint operator (transpose in matrix form) of the differential
	$DF(S_i, S_j)$.
	Using the standard basis for the Lie algebra $\sod(3)$,
	Tron~\cite{Tron2012Thesis} showed that the Riemannian gradient of the geodesic distance is,
	\begin{equation}
		\rgrad \distang(\Shat_i, \Shat_j) = \begin{bmatrix}
		 -u_{ij} \\
		 u_{ij}
		\end{bmatrix},
		\label{eq:angular_distance_3d_gradient}
	\end{equation}
	\cf \cite[Equation~(E.31)]{Tron2012Thesis}.
	Substituting \eqref{eq:angular_distance_3d_gradient} into \eqref{eq:RA_term_gradient_3d_chain_rule} and furthermore using the matrix form of the differential $DF(S_i, S_j)$ in the standard basis, we have
	\begin{align}
		\gbar_{ij}=
		\rgrad \hhat_{ij}(S_i, S_j) 
		&= \dot{\rho}(\theta_{ij}) 
		\begin{bmatrix}
		R_i \Rtilde_{ij}  & 0 \\
		0 & R_j
		\end{bmatrix}
		\begin{bmatrix}
		-u_{ij} \\
		u_{ij}
		\end{bmatrix}.
	\end{align}
	For the Riemannian Hessian, differentiating \eqref{eq:RA_term_gradient_3d_chain_rule_init} again yields,
	\begin{align}
		\Hess \hhat_{ij}(S_i, S_j) 
		&= DF(S_i, S_j)^\top \circ \Hess \varphi(\Shat_i, \Shat_j) \circ DF(S_i, S_j).
		\label{eq:RA_term_hessian_3d_chain_rule}
	\end{align}
	For the Hessian term in the middle of \eqref{eq:RA_term_hessian_3d_chain_rule}, we once again leverage existing results from \cite[Proposition~E.3.1]{Tron2012Thesis}:
	\begin{align}
		\Hess \varphi(\Shat_i, \Shat_j) &= 
		\begin{bmatrix}
		\Sym(\Htilde_{ij})  & -\Htilde_{ij} \\
		-\Htilde_{ij}^\top  & \Sym(\Htilde_{ij})
		\end{bmatrix},
	\end{align}
	where the inner matrix $\Htilde_{ij}$ is defined as, 
	\begin{align}
		\Htilde_{ij} &=
		\ddot{\rho}(\theta_{ij}) u_{ij}u_{ij}^\top 
		+ \frac{\dot{\rho}(\theta_{ij})}{\theta_{ij}} (D\Log(\Rtilde_{ij}^\top R_i^\top R_j) - u_{ij}u_{ij}^\top ).
	\end{align}
	Using the expression for the differential of the logarithm map \cite[Proposition~E.2.1]{Tron2012Thesis},
	the previous expression further simplifies to,
	\begin{equation}
		\begin{aligned}
		\Htilde_{ij} 
		&=
		\ddot{\rho}(\theta_{ij}) u_{ij}u_{ij}^\top 
		+ \frac{\dot{\rho}(\theta_{ij})}{\theta_{ij}} 
		\left(u_{ij}u_{ij}^\top + \frac{\theta_{ij}}{2} 
		\left ( \hatop{u_{ij}} - \cot(\theta_{ij}/2) \hatop{u_{ij}}^2 \right )
		- u_{ij}u_{ij}^\top 
		\right ) \\
		&=
		\ddot{\rho}(\theta_{ij}) u_{ij}u_{ij}^\top 
		+ \frac{\dot{\rho}(\theta_{ij})}{2}
		\left(
		\hatop{u_{ij}} - \cot(\theta_{ij}/2) \hatop{u_{ij}}^2
		\right) \\
		&=
		\ddot{\rho}(\theta_{ij}) u_{ij}u_{ij}^\top 
		+ \frac{\dot{\rho}(\theta_{ij})}{2} \hatop{u_{ij}} 
		- \frac{\dot{\rho}(\theta_{ij})\cot(\theta_{ij}/2)}{2} \left(-I + u_{ij}u_{ij}^\top \right) \\
		&=
		\frac{\dot{\rho}(\theta_{ij})\cot(\theta_{ij}/2)}{2} I 
		+ \left(\ddot{\rho}(\theta_{ij}) - \frac{\dot{\rho}(\theta_{ij})\cot(\theta_{ij}/2)}{2} \right) u_{ij}u_{ij}^\top 
		+ \frac{\dot{\rho}(\theta_{ij})}{2} \hatop{u_{ij}} \\
		&=
		\alpha(\theta_{ij}) I
		+ \gamma(\theta_{ij})  u_{ij}u_{ij}^\top 
		+ \beta(\theta_{ij}) \hatop{u_{ij}}.
		\end{aligned}
	\end{equation}
	To conclude, the Hessian is obtained by substituting the above results into \eqref{eq:RA_term_hessian_3d_chain_rule}:
	\begin{equation}
	\Hbar_{ij} =
	\Hess \hhat_{ij}(S_i, S_j) 
	= 
	\begin{bmatrix}
	R_i \Rtilde_{ij}  & 0 \\
	0 & R_j
	\end{bmatrix}
	\begin{bmatrix}
	\Sym(\Htilde_{ij})  & -\Htilde_{ij} \\
	-\Htilde_{ij}^\top  & \Sym(\Htilde_{ij})
	\end{bmatrix}
	\begin{bmatrix}
	R_i \Rtilde_{ij}  & 0 \\
	0 & R_j
	\end{bmatrix}^\top.
	\end{equation}
\end{proof}

The Hessian expression in \cref{lem:RA_term_3d} is complicated in general.
However, we will show that as the angular error $\theta_{ij}$ tends to zero, 
the Hessian $\Hbar_{ij}$ converges to a particular simple form. 
\edit{We note that the case under geodesic distance (\cref{lem:Hij_limit_angular} below) can also be derived as a special case of \cite[Theorem~1]{Wilson2016Convexity}.}

\begin{lemma}[Limit of $\Hbar_{ij}$ under geodesic distance] 
	For rotation averaging under the geodesic distance, it holds that, 
	\begin{equation}
		\lim_{\theta_{ij} \to 0} \Hbar_{ij}(\theta_{ij}) = 
		\begin{bmatrix}
			I_3 & -I_3 \\
			-I_3 & I_3 \\
		\end{bmatrix}.
	\end{equation}
	\label{lem:Hij_limit_angular}
\end{lemma}
\begin{proof}
	We first compute limits of $\alpha(\theta), \gamma(\theta)$, and $\beta(\theta)$ 
	that appear in the definition of $\Hbar_{ij}$.
	For rotation averaging under the geodesic distance, the scalar function $\rho(\theta)$ is defined as in \eqref{eq:reshape_func_angular}.
	In this case, we have
	\begin{equation}
		\dot{\rho}(\theta) = \theta, \;\; \ddot{\rho}(\theta) = 1.
	\end{equation}
	Substituting into \eqref{eq:alpha_func}-\eqref{eq:beta_func},
	\begin{align}
		\alpha(\theta) &= \frac{1}{2} \theta \cot(\theta/2) 
		= \frac{1}{2} \frac{\theta}{\sin(\theta/2)} \cos(\theta/2), \\
		\gamma(\theta) &= 1 - \alpha(\theta), \\
		\beta(\theta) &= \theta / 2.
	\end{align}
	Take the limit as $\theta$ tends to zero,
	\begin{align}
		\lim_{\theta \to 0} \alpha(\theta)
		& = \frac{1}{2} \cdot \lim_{\theta \to 0} \frac{\theta}{\sin(\theta/2)} \cdot 
		\lim_{\theta \to 0} \cos(\theta/2)
		= 1, \\
		\lim_{\theta \to 0} \gamma(\theta)
		&= 1 - \lim_{\theta \to 0} \alpha(\theta) = 0, \\
		\lim_{\theta \to 0} \beta(\theta)
		&= 0.
	\end{align}
	Define the following matrix,
	\begin{equation}
		P = 
		\begin{bmatrix}
			R_i \Rtilde_{ij}  & 0 \\
			0 & R_j
		\end{bmatrix}.
	\end{equation}
	From the definition of $\Hbar_{ij}$ in \eqref{eq:RA_term_hessian_3d},
	\begin{equation}
		\begin{aligned}
			\Hbar_{ij} &= 
			\alpha(\theta_{ij}) P 
			\begin{bmatrix}
				I_3 & -I_3 \\
				-I_3 & I_3
			\end{bmatrix}
			P^\top \\
			&+
			\gamma(\theta_{ij}) P
			\begin{bmatrix}
				u_{ij} u_{ij}^\top & -u_{ij} u_{ij}^\top \\
				-u_{ij} u_{ij}^\top & u_{ij} u_{ij}^\top
			\end{bmatrix}
			P^\top \\
			&+ 
			\beta(\theta_{ij}) P
			\begin{bmatrix}
				0_3 & - \hatop{u_{ij}} \\
				- \hatop{u_{ij}}^\top & 0_3
			\end{bmatrix}
			P^\top.
		\end{aligned}
		\label{eq:Hij_limit_angular_step_1}
	\end{equation}
	Since $\lim_{\theta \to 0} \gamma(\theta) = \lim_{\theta \to 0} \beta(\theta) = 0$ and all matrices involved in \eqref{eq:Hij_limit_angular_step_1} are bounded, 
	we conclude that the last two terms in \eqref{eq:Hij_limit_angular_step_1} vanish as $\theta_{ij}$ converges to zero.
	For the first term in \eqref{eq:Hij_limit_angular_step_1},
	notice that,
	\begin{equation}
		\alpha(\theta_{ij}) P 
		\begin{bmatrix}
			I_3 & -I_3 \\
			-I_3 & I_3
		\end{bmatrix}
		P^\top
		= 
		\alpha(\theta_{ij})
		\begin{bmatrix}
			I_3 & -R_i \Rtilde_{ij} R_j^\top \\
			-R_j \Rtilde_{ij}^\top R_i^\top & I_3
		\end{bmatrix}.
		\label{eq:Hij_limit_angular_step_2}
	\end{equation}
	As $\theta_{ij}$ tends to zero, 
	$\alpha(\theta_{ij})$ converges to $1$ and 
	the off-diagonal blocks in \eqref{eq:Hij_limit_angular_step_2} tend to $-I_3$.
	Hence the proof is completed.
\end{proof}

\begin{lemma}[Limit of $\Hbar_{ij}$ under chordal distance]
	For rotation averaging under the chordal distance, it holds that, 
	\begin{equation}
		\lim_{\theta_{ij} \to 0} \Hbar_{ij}(\theta_{ij}) = 
		2
		\begin{bmatrix}
			I_3 & -I_3 \\
			-I_3 & I_3 \\
		\end{bmatrix}.
	\end{equation}
	\label{lem:Hij_limit_chordal}
\end{lemma}
\begin{proof}
	For rotation averaging under the chordal distance, the scalar function $\rho(\theta)$ is defined as in \eqref{eq:reshape_func_chordal}.
	In this case, we have
	\begin{equation}
		\dot{\rho}(\theta) = 2 \sin(\theta), \;\; \ddot{\rho}(\theta) = 2\cos(\theta).
	\end{equation}
	Substituting into \eqref{eq:alpha_func}-\eqref{eq:beta_func},
	\begin{align}
		\alpha(\theta) &=  
		\sin(\theta) \cot(\theta/2) = 2\cos(\theta/2)^2, \\
		\gamma(\theta) &= 2\cos(\theta) - \alpha(\theta), \\
		\beta(\theta) &= \sin(\theta).
	\end{align}
	Take the limit as $\theta$ tends to zero,
	\begin{align}
		\lim_{\theta \to 0} \alpha(\theta)
		& = 2, \\
		\lim_{\theta \to 0} \gamma(\theta)
		&= 2 - \lim_{\theta \to 0} \alpha(\theta) = 0, \\
		\lim_{\theta \to 0} \beta(\theta)
		&= 0.
	\end{align}
	The remaining proof is similar to that of \cref{lem:Hij_limit_angular} and is omitted.
\end{proof}

\subsection{Proof of Theorem~\ref{thm:Hessian_approximation}}

\begin{proof}

We will use \cref{lem:Hij_limit_angular} to prove the theorem for the case of squared geodesic cost.
The case of squared chordal cost is analogous: instead of \cref{lem:Hij_limit_angular}, we will use \cref{lem:Hij_limit_chordal} and the remaining steps are the same. 
Recall the approximation of the overall cost function defined in \eqref{eq:RA_full_pullback_global_def}:
\begin{equation}
	h(v; R) 
	= \sum_{(i,j) \in \Ecal} \kappa_{ij} h_{ij}(v_i, v_j)
	= \sum_{(i,j) \in \Ecal} 
	\kappa_{ij}
	\varphi(\Exp(v_i) R_i \Rtilde_{ij}, \Exp(v_j) R_j),
	\label{eq:RA_full_pullback_global}
\end{equation}
Observe that the Hessian of $h(v; R)$ is simply given by the sum of the Hessian matrices of $h_{ij}(v_i, v_j)$, after ``lifting'' the latter to the dimension of the full optimization problem, \ie, 
\begin{equation}
	\begin{aligned}
		\Hbar(R) = \sum_{(i,j) \in \Ecal} \kappa_{ij} W_{ij}, \;\;
		W_{ij} =
		\renewcommand{\kbldelim}{[}
		\renewcommand{\kbrdelim}{]}
		\kbordermatrix{ & & i & & j & \cr
			& & \vdots & & \vdots & \cr
			i &\ldots & \Hbar_{ij}^{(ii)} & \ldots & \Hbar_{ij}^{(ij)} & \ldots \cr
			& & \vdots & & \vdots & \cr
			j & \ldots & \Hbar_{ij}^{(ji)} & \ldots & \Hbar_{ij}^{(jj)} & \ldots \cr
			& & \vdots & & \vdots & \cr
		}.
	\end{aligned}
	\label{eq:RA_full_hessian_3d}
\end{equation}
In \eqref{eq:RA_full_hessian_3d}, $W_{ij}$ is formed by placing the $p$-by-$p$ blocks of $\Hbar_{ij}$ defined in \eqref{eq:RA_term_hessian_3d} in corresponding locations of the full matrix.
{For instance, $\Hbar_{ij}^{(ii)}$ is the block of $\Hbar_{ij}$ that corresponds to vertex $i$.}

	In the following,  let $\theta_{ij}(R)$ denote the residual of edge $(i,j) \in \Ecal$ evaluated at $R \in \SOd(d)^n$.
	From \eqref{eq:RA_full_hessian_3d} and \cref{lem:Hij_limit_angular}, 
    we see that $\Hbar(R)$ has the following limit point as all edge residuals tend to zero, 
	\begin{equation}
		\lim_{
		\substack{
							\theta_{ij}(R) \to 0, \\
							\forall (i,j) \in \Ecal}
		} \Hbar(R)
		= M
		\triangleq
		L(G; \kappa)  \otimes I_p. 
	\end{equation}
	Recall from \cref{rm:approximate_newton_step_feasibility} that $M \succeq 0$ and
	$\ker(M) = \Ncal$ where $\Ncal$ is the vertical space defined in \eqref{eq:vertical_space_RA}. 
	In addition, recall the definition of $H(R)$ in \eqref{eq:Newton_step_quotient}: 
	\begin{equation}
		H(R) = P_H \Hbar(R) P_H.
	\end{equation}
	Since $P_H$ is the (constant) orthogonal projection matrix onto the horizontal space $\Hcal = \Ncal^\perp$,
	it holds that,
	\begin{equation}
		\lim_{
			\substack{
				\theta_{ij}(R) \to 0, \\
				\forall (i,j) \in \Ecal}
		} H(R)
		= P_H \, M \, P_H
		= M.
		\label{eq:H_limit_angular}
	\end{equation}
	
	Note that for singular symmetric matrices $A$ and $B$, $A\approx_\delta B$ necessarily means that $\ker(A) = \ker(B)$.
	Therefore, to prove the theorem, we must first show that $\ker(H(R)) = \ker(M) = \Ncal$ under our assumptions.
	Let $\lambda_1(A), \lambda_2(A), \hdots$ denote the eigenvalues of a symmetric matrix $A$ sorted in increasing order.
	By construction, $\Ncal$ is always contained in $\ker(H(R))$, and thus $\dim(\Ncal) = p$ eigenvalues of $H(R)$ are always zero. 
	Next, we will show that if all measurement residuals are sufficiently small, then the remaining 
	$pn-p$ eigenvalues of $H(R)$ will be strictly positive. 
	Define $E(R) \triangleq H(R) - M$.
	Let $x$ be any unit vector such that $x \perp \Ncal$. Note that,
	\begin{equation}
		\begin{aligned}
			x^\top H(R) x 
			&= x^\top M x + x^\top E(R) x \\
			&\geq \lambda_{p+1} (M) - \norm{E(R)}_2.
		\end{aligned}
	\label{eq:H_kernel_analysis}
	\end{equation}
	Since $M = L(G; \kappa) \otimes I_p$, it holds that $\lambda_{p+1} (M) = \lambda_2(L(G; \kappa))$.
	The latter is known as the \emph{algebraic connectivity} which is always positive for a connected graph $G$.
	Thus $\lambda_{p+1} (M) > 0$ and by \eqref{eq:H_limit_angular}, we also have $\lim_{\theta_{ij}(R) \to 0} E(R) = 0$.
	Consequently, when all $\theta_{ij}(R)$ are sufficiently small, 
	the right-hand side of \eqref{eq:H_kernel_analysis} is strictly positive, \ie,
	there exists $\bar{\theta}_1 > 0$ such that if $\theta_{ij}(R) \leq \bar{\theta}_1$ for all $(i,j) \in \Ecal$, we have,
	\begin{equation}
		\ker(H(R)) = \ker(M) = \Ncal.
		\label{eq:H_kernel_low_noise}
	\end{equation}
	
	Under \eqref{eq:H_kernel_low_noise}, the desired approximation $H(R) \approx_\delta M$ is equivalent to,
	\begin{equation}
		e^{-\delta} P_H \preceq M^{\frac{+}{2}} H(R) M^{\frac{+}{2}}  \preceq e^{\delta} P_H,
	\end{equation}
	where $M^{\frac{+}{2}}$ denotes the square root of the pseudoinverse of $M$, and
	$P_H$ is the orthogonal projection onto the horizontal space $\Hcal$.
	This condition is true if and only if the nontrivial eigenvalues are bounded as follows,
	\begin{align}
		\lambda_{p+1}(M^{\frac{+}{2}} H(R) M^{\frac{+}{2}}) \geq e^{-\delta},  \;
		\lambda_{pn}(M^{\frac{+}{2}} H(R) M^{\frac{+}{2}})   \leq e^\delta.
		\label{eq:eigenvalue_bounds}
	\end{align}
	Using the convergence result \eqref{eq:H_limit_angular}
	and the eigenvalue perturbation bounds in \cite[Corollary~6.3.8]{Horn2012MatrixAnalysis},
	we conclude that there exists $\thetabar_0 \in (0, \thetabar_1]$ such that 
	if $R \in \SOd(d)^n$ satisfies 
	\begin{equation}
		\theta_{ij}(R) \leq \bar{\theta}_0, \; \forall (i,j) \in \Ecal,
		\label{eq:edge_residual_bound}
	\end{equation}
	then \eqref{eq:eigenvalue_bounds} holds, \ie, we have the desired approximation,
	\begin{align}
		H(R) \approx_\delta M.
	\end{align}

	To conclude the proof, we need to show that there exist $\thetabar, r > 0$ such that 
	condition \eqref{eq:edge_residual_bound} holds for all $R \in B_r(\Rstar)$.
	Let us first consider residuals at the global minimizer $\Rstar$.
	Using assumption \eqref{thm:Hessian_approximation:noise_bound} and the cost function,
	we obtain the following simple bound:
	\begin{equation}
		\max_{(i,j) \in \Ecal} \frac{\kappa_{ij} \theta_{ij}(\Rstar)^2}{2} 
		\leq 
		f(\Rstar) \leq f(\Runder) \leq 
		\frac{\sum_{(i,j) \in \Ecal} \kappa_{ij} \thetabar^2}{2}.
		\label{eq:cost_function_bound}
	\end{equation}
	It can be verified that if,
	\begin{equation}
		\thetabar \leq \frac{\bar{\theta}_0}{2} \sqrt{\frac{\min_{(i,j) \in \Ecal} \kappa_{ij}}{\sum_{(i,j) \in \Ecal} \kappa_{ij}}},
		\label{eq:delta_bound} 
	\end{equation}
	then \eqref{eq:cost_function_bound} yields $\theta_{ij}(\Rstar) \leq \bar{\theta}_0/2$ for all edges $(i,j) \in \Ecal$.
	Finally, let us select $r \in (0, \; \bar{\theta}_0/4)$.
	For $i \in [n]$, let $E_i \in \SOd(d)$ such that $R_i = E_i \Rstar_i$.
	Using the triangle inequality and the fact that the geodesic distance $\distang(\cdot, \cdot)$ is bi-invariant,
	we can show that for any $R \in B_r(\Rstar)$,
	\begin{align}
		\theta_{ij}(R) 
		&= \distang(R_i \Rtilde_{ij}, R_j) \\
		&= \distang(E_i \Rstar_i \Rtilde_{ij}, E_j \Rstar_j) \\
		&= \distang(\Rstar_i \Rtilde_{ij} (\Rstar_j)^\top, E_i^\top E_j) \\
		&\leq \distang(\Rstar_i \Rtilde_{ij} (\Rstar_j)^\top, I) + \distang(E_i, I) + \distang(E_j, I) \\
		&\leq \theta_{ij}(\Rstar) + 2r \\
		&\leq \bar{\theta}_0.
	\end{align}
	In summary, we have shown that if $\thetabar$ satisfies \eqref{eq:delta_bound}
	and furthermore $0 < r < \bar{\theta_0} / 4$,
	then the desired approximation $M \approx_\delta H(R)$ holds for all $R \in B_r(\Rstar)$.	
	This concludes the proof.
\end{proof}

\subsection{Proof of Corollary~\ref{cor:Hessian_bound}}
\begin{proof}
	To simplify notation, we use $L$ to denote $L(G;w)$.
	By \eqref{thm:Hessian_approximation:spectral_approximation}, it holds that,
	\begin{equation}
		e^{-\delta} (L \otimes I_p) \preceq H(R) \preceq e^{\delta} (L \otimes I_p).
	\end{equation}
	Note that the eigenvalues of $L \otimes I_p$ are given by the eigenvalues of $L$, repeated $p$ times.
	Therefore, the desired result follows by noting that,
	\begin{equation}
		\lambda_2(L) P_H  \preceq L \otimes I_p \preceq \lambda_n(L) P_H.
	\end{equation}
\end{proof}

\edit{
\subsection{Connections between Approximate Cost (\ref{eq:RA_full_pullback_global_def}) and the Pullback Function}
\label{sec:rotation_averaging_hessian_analysis:pullback_discussion}

In this subsection, we show that the approximate cost function in \eqref{eq:RA_full_pullback_global_def} is related to the standard \emph{pullback function} \cite{Absil2009Book,Boumal20Book} of rotation averaging in the total space via a \emph{change of variable}.
As before, we focus on the case of $d=3$ and the derivations can be simplified for $d=2$.
Given $R = (R_1, \hdots, R_n) \in \SOd(3)^n$, let $\fbar(R)$ denote the cost function in rotation averaging (\cref{prob:rotation_averaging}).
The standard pullback function, using the Riemannian exponential map as retraction, is defined as,
\begin{equation}
\fhat(\eta) \triangleq \fbar (\Exp_R(\eta)),
\label{eq:pullback_eta}
\end{equation}
where $\eta = (\eta_1, \hdots, \eta_n) \in T_R \SOd(3)^n$.
For every $i \in [n]$, $\eta_i$ is a tangent vector $\eta_i \in T_{R_i} \SOd(3)$ where
$\eta_i = R_i \hatop{u_i}$ for some $u_i \in \Real^3$ \cite[Example~3.5.3]{Absil2009Book}.
The Riemannian exponential map is given by 
$\Exp_R(\eta) = (\Exp_{R_1} (\eta_1), \hdots, \Exp_{R_n} (\eta_n) )$,
and for every $i \in [n]$,
\begin{equation}
	\Exp_{R_i} (\eta_i) = R_i \Exp(u_i),
\end{equation}
where $\Exp(\cdot)$ is defined in \eqref{eq:Exp_def}; see also \cite[Equation~2.23]{Tron2012Thesis}.
Next, let us define $v_i \triangleq R_i u_i$.
Using the adjoint operator on $\SOd(3)$, we see that,
\begin{equation}
	\Exp_{R_i} (\eta_i) = R_i \Exp(u_i) = \Exp(v_i) R_i.
	\label{eq:eta_vs_u_vs_v}
\end{equation}
Combining the fact that $\eta_i = R_i \hatop{u_i}$ and $v_i = R_i u_i$,
we see that $\eta_i \in T_{R_i} \SOd(3)$ is related to $v_i \in \Real^3$ via a linear and invertible mapping:
\begin{equation}
	\eta_i = R_i \hatop{R_i^\top v_i}.
	\label{eq:change_of_var}
\end{equation}
Using \eqref{eq:change_of_var}, we can write the pullback \eqref{eq:pullback_eta}
using the new variable $v \triangleq \begin{bmatrix} v_1^\top & \hdots & v_n^\top \end{bmatrix}^\top \in \Real^{3n}$,
\begin{equation}
	h(v; R) \triangleq \fhat(\eta(v)) = f(\Exp(v_1) R_1, \hdots, \Exp(v_n) R_n),
	\label{eq:pullback_v}
\end{equation}
where we use $\eta(v)$ to denote the change of variable from $v \in \Real^{3n}$ to $\eta \in T_R \SOd(3)^n$.
Note that $h(v;R)$ defined in \eqref{eq:pullback_v} corresponds exactly with \eqref{eq:RA_full_pullback_global_def} after expanding the cost function $f(\cdot)$ using its definition.

In Riemannian Newton's method, one computes each iteration by minimizing a \emph{quadratic model} of the pullback function \eqref{eq:pullback_eta}.
On the total space, the quadratic model is defined as follows,
\begin{equation}
	\fhat(\eta) \approx \mhat(\eta) \triangleq \fhat(0) + \inner{\nabla \fhat(0)}{\eta} + 
	\frac{1}{2}\inner{\nabla^2 \fhat(0)[\eta]}{\eta}.
	\label{eq:pullback_model_eta}
\end{equation}
By Proposition~3.59 and Proposition~5.45 in\cite{Boumal20Book}, it holds that,
\begin{equation}
	\nabla \fhat(0) = \rgrad \fbar(R), \;
	\nabla^2 \fhat(0) = \Hess \fbar(R),
\end{equation}
where $\rgrad \fbar(R) \in T_R \SOd(3)^n$ 
and $\Hess \fbar(R):  T_R \SOd(3)^n \to  T_R \SOd(3)^n$ 
are the conventional Riemannian gradient and Hessian on the total space.
With the change of variable from $\eta \in T_R \SOd(3)$ to $v \in \Real^{3n}$,
we note that minimizing \eqref{eq:pullback_model_eta} is equivalent to minimizing the following quadratic model defined using $v$:
\begin{equation}
	h(v;R) \approx m(v;R) \triangleq h(0; R) + \inner{\gbar(R)}{v} + \frac{1}{2} \inner{\Hbar(R)v}{v}, 
	\label{eq:pullback_model_v}
\end{equation}
where $\gbar(R)$ and $\Hbar(R)$ are defined in \eqref{eq:gradient_and_hessian_def} in the main paper. 
}

\section{Performance Guarantees for Collaborative Laplacian Solver}
\label{sec:laplacian_appendix}

\subsection{Proof of Lemma~\ref{lem:schur_complement_sum}}
\label{sec:laplacian_appendix:schur_complement_sum}
\begin{proof}
	By definition in \eqref{eq:laplacian_reduced_system}, 
	\begin{equation}
		S = L_{cc} - \sum_{\alpha \in [m]} L_{c\alpha} L_{\alpha \alpha}^{-1} L_{\alpha c}.
		\label{eq:schur_complement_sum_proof_1}
	\end{equation}
	Above, $L_{cc}$ is the block of the full Laplacian $L$ that corresponds to the separators, 
	denoted as $L_{cc} \equiv L(G)_{cc}$.
	Note that $L_{cc}$ can be decomposed as the sum, 
	\begin{equation}
		L_{cc} = L(G_c) + \sum_{\alpha \in [m]} L(G_\alpha)_{cc}.
		\label{eq:schur_complement_sum_proof_2}
	\end{equation}
	Intuitively, the first term in \eqref{eq:schur_complement_sum_proof_2} accounts for inter-robot edges $\Ecal_c$,
	and the second group of terms accounts for robots' local edges $\Ecal_\alpha$;
	see \cref{fig:example:original_graph}.
	Substitute \eqref{eq:schur_complement_sum_proof_2} into \eqref{eq:schur_complement_sum_proof_1},
	\begin{equation}
		\begin{aligned}
			S &= L(G_c) + \sum_{\alpha \in [m]} \left (
			L(G_\alpha)_{cc} - L_{c\alpha} L_{\alpha \alpha}^{-1} L_{\alpha c} \right ) \\
			&= L(G_c) + \sum_{\alpha \in [m]} \Sc(L(G_\alpha), \Fcal_\alpha).
		\end{aligned}
	\end{equation}
\end{proof}

\subsection{Proof of Theorem~\ref{thm:approximation_guarantees}}
\label{sec:laplacian_appendix:approximation_guarantees}
\begin{proof}
	Let us simplify the notations in the Laplacian system \eqref{eq:laplacian_arrowhead_pattern} by considering interior nodes from all robots as a single block:
	\begin{equation}
		\begin{bmatrix}
			L_{ff}  & L_{fc} \\
			L_{cf}  & L_{cc}
		\end{bmatrix}
		\begin{bmatrix}
			X_f \\
			X_c
		\end{bmatrix}
		= 
		\begin{bmatrix}
			B_f \\
			B_c
		\end{bmatrix}
		\label{eq:laplacian_block_form}
	\end{equation}
	where 
	\begin{align}
		L_{ff} &= \Diag(L_{11}, \hdots, L_{mm}), \\
		L_{cf} &= L_{fc}^\top
		= 
		\begin{bmatrix}
			L_{c1}   & \hdots  & L_{cm}
		\end{bmatrix}, \\
		X_f &= 
		\begin{bmatrix}
			X_1^\top & \hdots & X_m^\top
		\end{bmatrix}^\top, \\
		B_f &= 
		\begin{bmatrix}
			B_1^\top & \hdots & B_m^\top
		\end{bmatrix}^\top.
	\end{align}
	By applying the Schur complement to \eqref{eq:laplacian_block_form},
	we obtain the following factorization for the input Laplacian system $LX = B$,
	\begin{equation}
		\underbrace{
			\bmat
			I & 0 \\
			L_{cf} L_{ff}\inv & I
			\emat
			\bmat
			L_{ff}  & 0 \\
			0  & S
			\emat
			\bmat
			I & L_{ff}\inv L_{fc} \\
			0 & I
			\emat
		}_{L}
		\bmat
		X_f \\
		X_c
		\emat
		= 
		\bmat
		B_f \\
		B_c
		\emat,
		\label{eq:laplacian_schur_factorization}
	\end{equation}
	where $S = \Sc(L, \Fcal)$ is the Schur complement that appears in \eqref{eq:laplacian_reduced_system}.
	It can be verified that \cref{alg:sparsified_laplacian_solver}
	returns a solution to the following system,
	\begin{equation}
		\underbrace{
			\bmat
			I & 0 \\
			L_{cf} L_{ff}\inv & I
			\emat
			\bmat
			L_{ff}  & 0 \\
			0  & \Stilde
			\emat
			\bmat
			I & L_{ff}\inv L_{fc} \\
			0 & I
			\emat
		}_{\Ltilde}
		\bmat
		X_f \\
		X_c
		\emat
		= 
		\bmat
		B_f \\
		B_c
		\emat.
		\label{eq:laplacian_approximate_factorization}
	\end{equation}
	Recall from \cref{lem:schur_complement_sum} that,
	\begin{equation}
		S = L(G_c) + \sum_{\alpha \in [m]} S_\alpha.
	\end{equation}
	Meanwhile, by construction, $\Stilde$ is given by, 
	\begin{equation}
		\Stilde = L(G_c) + \sum_{\alpha \in [m]} \Stilde_\alpha, 
	\end{equation}
	where $\Stilde_\alpha \approx_\epsilon S_\alpha$ for all $\alpha \in [m]$.
	Since spectral approximation is preserved under addition, 
	it holds that $\Stilde \approx_\epsilon S$.
	Furthermore, by comparing $L$ defined in \eqref{eq:laplacian_schur_factorization} and $\Ltilde$ defined in \eqref{eq:laplacian_approximate_factorization} and using \cite[Fact~3.2]{lee2015sparsified}, we conclude that $\Ltilde \approx_\epsilon L$ and thus \eqref{thm:approximation_guarantees:spectrum} is true.
	Lastly, \eqref{thm:approximation_guarantees:solution} follows from \cref{lem:apx_laplacian_solution_guarantee}.
\end{proof}
\section{Convergence Analysis}
\label{sec:linear_convergence_proof}
In this section, we establish convergence guarantees for the collaborative rotation averaging (\cref{alg:collaborative_rotation_averaging}) and translation estimation (\cref{alg:collaborative_translation_recovery}) methods developed in \cref{sec:algorithm}.
Between the two, analyzing \cref{alg:collaborative_rotation_averaging} is more complicated owing to the fact that rotation averaging is an optimization problem defined on a Riemannian manifold.
To establish its convergence, in \cref{sec:linear_convergence_general} we first prove a more general result that holds for generic approximate Newton methods on manifolds.
Then, in \cref{sec:linear_convergence_RA}, we invoke this result for the special case of rotation averaging and show that \cref{alg:collaborative_rotation_averaging} enjoys a local \emph{linear} convergence rate.
Lastly, in \cref{sec:linear_convergence_translation}, we prove the linear convergence of translation estimation (\cref{alg:collaborative_translation_recovery}).

\subsection{Analysis of General Approximate Newton Method}
\label{sec:linear_convergence_general}

\begin{algorithm}[t]
	\caption{\textsc{Approximate Newton Method}}
	\label{alg:apx_newton}
	\begin{algorithmic}[1]
		\small
		\For{iteration $k = 0, 1, \hdots$}
			\State  $\eta^k = -M(x^k)\inv \rgrad f(x^k)$. \label{alg:apx_newton:newton_step}
			\State Update iterate by $x^{k+1} = \Retr_{x^k}(\eta^k)$.
		\EndFor
	\end{algorithmic}
\end{algorithm}

In this subsection, we consider a generic optimization problem on a smooth Riemannian manifold $\Mcal$:
\begin{equation}
	\min_{x \in \Mcal} f(x).
\end{equation}
We consider solving the above problem using an \emph{approximate Newton method} described in \cref{alg:apx_newton}.
At each iteration, the Riemannian Hessian $\Hess f(x)$ is replaced with an approximation $M(x)$, 
and the update is computed by solving a linear system in $M(x)$; see line~\ref{alg:apx_newton:newton_step}.
We will show that under the following assumptions (in particular, $M(x)$ is a sufficiently good approximation of $\Hess f(x)$),
\cref{alg:apx_newton} achieves a local linear rate of convergence.

\begin{assumption}
	Let $\xstar$ denote a strict second-order critical point.
	There exist $\mu_H, L_H, \beta, \epsilon > 0$ such that 
	for all $x$ in a neighborhood $\Ucal$ of $\xstar$,
	\begin{enumerate}[label=(A{{\arabic*}})]
		\item $\mu_H I \preceq \Hess f(x) \preceq L_H I$.
		\label{as:linear_convergence:hessian_bound}
		\item $M(x)$ is invertible and $\norm{M(x)\inv } \leq \beta$.
		\label{as:linear_convergence:M_bound}
		\item $M(x) \approx_\epsilon \Hess f(x)$ and $\epsilon$ satisfies 
		\begin{equation}
			\gamma(\epsilon) \triangleq 2  \sqrt{\kappa_H} c(\epsilon) < 1,
		\end{equation}
		where $c(\epsilon)$ is defined in \eqref{eq:c_epsilon_def} and 
		$\kappa_H = L_H / \mu_H $ is the condition number.
		\label{as:linear_convergence:spectral_approximation}
	\end{enumerate}
	\label{as:linear_convergence}
\end{assumption}

\begin{theorem}
	Under \cref{as:linear_convergence}, 
	there exists a neighborhood $\Ucal' \subseteq \Ucal$ such that for all $x_0 \in \Ucal'$,
	\cref{alg:apx_newton} generates an infinite sequence $x^k$ converging linearly to $\xstar$.
	Furthermore, the linear convergence factor is given by,
	\begin{equation}
		\underset{k \to \infty}{\lim \sup}
		\frac{\dist(x^{k+1}, \xstar)}{\dist(x^k, \xstar)} = 
		\gamma(\epsilon).
		\label{eq:linear_convergence:factor}
	\end{equation}
	\label{thm:linear_convergence}
\end{theorem}
\begin{proof}
	We prove the theorem by adapting the local convergence analysis of the Riemannian Newton method presented in \cite[Theorem~6.3.2]{Absil2009Book}.
	Let $(\Ucal', \varphi)$ be a local coordinate chart defined by the normal coordinates around $\xstar$.
	Similar to the original proof, we will use the $\hat{\cdot}$ notation to denote coordinate expressions in this chart.
	In particular, let us define,
	\begin{equation}
		\xhat = \varphi(x) = \Exp_{\xstar}\inv(x).
	\end{equation}
	Note that under the normal coordinates, we have
	$\xstarhat = 0$,
	$\Exp_{\xstar}(\xhat) = x$,
	and $\dist(x, \xstar) = \norm{\xhat}$.
	In addition, let us define,
	\begin{align}
		\etahat &= \Diff \varphi(x)[\eta], \; \eta \in T_x \Mcal, \\
		\ghat(\xhat^k) &= \Diff \varphi(x^k)[\rgrad f], \\
		\Hhat(\xhat^k) &= \Diff \varphi(x^k) \circ \Hess f(x^k) \circ (\Diff \varphi(x^k))\inv, \\
		\Rhat_{\xhat}(\etahat) &= \varphi(\Retr_{x}(\eta)), \; \eta \in T_x \Mcal,
	\end{align}
	to be the coordinate expressions of vector fields, gradient, Hessian, and the retraction, respectively.
	Finally, let $\Mhat$ denote the coordinate expression of the linear map $M$ used in \cref{alg:apx_newton}:
	\begin{equation}
		\Mhat(\xhat^k) = \Diff \varphi(x_k) \circ M(x^k) \circ (\Diff \varphi(x_k))\inv.
		\label{eq:Mhat_def}
	\end{equation}
	Let us express each iteration of \cref{alg:apx_newton} in the chart,
	\begin{equation}
		\xhat^{k+1} = \Rhat_{\xhat^k} (- \Mhat(\xhat^k)\inv \ghat(\xhat^k)).
	\end{equation}
	Using the triangle inequality, we can bound the distance between $x^{k+1}$ and $\xstar$,
	\begin{equation}
	\begin{aligned}
		\dist(x^{k+1}, \xstar) 
		= &\norm{\xhat^{k+1} - \xstarhat} \\
		= &\norm{\Rhat_{\xhat^k} (- \Mhat(\xhat^k)\inv \ghat(\xhat^k)) - \xstarhat} \\
		\leq &
		\underbrace{\norm{\Rhat_{\xhat^k} (- \Mhat(\xhat^k)\inv \ghat(\xhat^k)) - (\xhat^k - \Mhat(\xhat^k)\inv \ghat(\xhat^k))}}_{A}
		+ \\
		&
		\underbrace{\norm{\xhat^k - \Mhat(\xhat^k)\inv \ghat(\xhat^k) - \xstarhat}}_{B}
	\end{aligned}
	\label{eq:linear_convergence_proof:triangle_inequality}
	\end{equation}
	In the following, we will derive upper bounds for $A$ and $B$ as a function of $\norm{\xhat^k - \xstarhat} = \dist(x^{k}, \xstar)$.
	\paragraph{\underline{Bounding $A$}:}
	Note that at $\xstar$, we have $\Diff \varphi(\xstar) = I$.
	Since $\varphi$ is smooth, there exists $r_1 > 0$ such that for all 
	$x \in B_{r_1}(\xstar) = \{x \in \Mcal, \; \dist(x, \xstar) < r_1 \}$,
	\begin{equation}
		\norm{\Diff \varphi(x)} \leq \sqrt{2}, \; \norm{(\Diff \varphi(x))\inv} \leq \sqrt{2}.
		\label{eq:Dvarphi_bound}
	\end{equation}
	It follows from \ref{as:linear_convergence:M_bound} and \eqref{eq:Mhat_def} that,
	\begin{equation}
		\norm{\Mhat(\xhat)\inv} \leq 2\beta, \; \forall x \in B_{r_1}(\xstar).
		\label{eq:Mhat_bound}
	\end{equation}
	Furthermore, Assumption \ref{as:linear_convergence:hessian_bound} implies that the gradient is Lipschitz continuous.
	Using the fact that $\ghat(\xstarhat) = 0$ (since $\xstar$ is a critical point),
	we have the following upper bound for the norm of the Newton step,
	\begin{equation}
	\begin{aligned}
		\norm{\Mhat(\xhat^k)\inv \ghat(\xhat^k)} 
		&= \norm{\Mhat(\xhat^k)\inv (\ghat(\xhat^k) - \ghat(\xstarhat))} \\
		&\leq 2\beta L' \norm{\xhat^k - \xstarhat},
	\end{aligned}
	\label{eq:newton_step_bound}
	\end{equation}
	where $L' > 0$ is a fixed constant.
	Using the local rigidity property of the retraction (\eg, see \cite[Definition~4.1.1]{Absil2009Book}), we have that,
	\begin{equation}
		\norm{\Rhat_{\xhat}(\etahat) - (\xhat + \etahat)} = O(\norm{\etahat}^2), 
		\label{eq:retraction_local_bound}
	\end{equation}
	for all $x$ in a neighborhood of $\xstar$ and all $\eta$ sufficiently small; see also the discussions in \cite[p.~115]{Absil2009Book}.
	It follows from \eqref{eq:retraction_local_bound} and \eqref{eq:newton_step_bound}
	that there exists $r_2, C_2 > 0$ such that 
	\begin{equation}
		A \leq C_2 \norm{\xhat^k - \xstarhat}^2,
		\label{eq:A_bound}
	\end{equation}
	for all $x^k \in B_{r_2}(\xstar)$.
	\paragraph{\underline{Bounding $B$}:}
	To begin, we derive the following upper bound for $B$ using triangle inequality,
	\begin{equation}
	\begin{aligned}
		B =& \norm{\xhat^k - \Mhat(\xhat^k)\inv \ghat(\xhat^k) - \xstarhat} \\
		=& \norm{\Mhat(\xhat^k)\inv \left [ \ghat(\xstarhat) -\ghat(\xhat^k) - \Mhat(\xhat^k) (  \xstarhat - \xhat^k  ) \right]} \\
		=& \norm{\Mhat(\xhat^k)\inv \left [ \ghat(\xstarhat) -\ghat(\xhat^k) - \Hhat(\xhat^k) (  \xstarhat - \xhat^k  ) - (\Mhat(\xhat^k) - \Hhat(\xhat^k))(  \xstarhat - \xhat^k  ) \right ]} \\
		\leq &
		\underbrace{\norm{\Mhat(\xhat^k)\inv \left [ \ghat(\xstarhat) -\ghat(\xhat^k) - \Hhat(\xhat^k) (  \xstarhat - \xhat^k  ) \right]}}_{B_1} + \\
		&
		\underbrace{\norm{ \left [I - \Mhat(\xhat^k)\inv \Hhat(\xhat^k) \right ] (  \xstarhat - \xhat^k  ) }}_{B_2}
	\end{aligned}
	\label{eq:B_bound}
	\end{equation}
	To bound $B_1$, it follows from \eqref{eq:Mhat_bound} that,
	\begin{equation}
		B_1 \leq 2\beta \norm{\ghat(\xstarhat) -\ghat(\xhat^k) - \Hhat(\xhat^k) (  \xstarhat - \xhat^k  )}.
		\label{eq:B1_bound_1}
	\end{equation}
	Furthermore, in the proof of \cite[Theorem~6.3.2]{Absil2009Book}, it is shown that there exist $r_3, C_3 > 0$ such that,
	\begin{equation}
		\norm{\ghat(\xstarhat) -\ghat(\xhat^k) - \Hhat(\xhat^k) (  \xstarhat - \xhat^k  )} \leq C_3 \norm{\xhat^k - \xstarhat}^2,
	\end{equation}
	for $x^k \in B_{r_3}(\xstar)$. Combining this result with \eqref{eq:B1_bound_1}, it holds that, 
	\begin{equation}
		B_1 \leq 2\beta C_3 \norm{\xhat^k - \xstarhat}^2.
		\label{eq:B1_bound}
	\end{equation}
	It remains to establish an upper bound for the matrix that appears in $B_2$:
	\begin{equation}
	\begin{aligned}
		I - \Mhat(\xhat^k)\inv \Hhat(\xhat^k) 
		=& \Diff \varphi(x) \circ (I - M(x^k)\inv \Hess f(x^k) ) \circ (\Diff \varphi(x))\inv.
	\end{aligned}
	\label{eq:Minv_H_bound_0}
	\end{equation}
	In the following, we first bound the norm of $I - M(x^k)\inv \Hess f(x^k)$.
	For any $\eta \in T_{x^k} \Mcal$, let us consider the following quantity,
	\begin{equation}
		\begin{aligned}
			  &\norm{\left (I - M(x^k)\inv \Hess f(x^k) \right ) \eta}_{H^k} \\
			= &\norm{\left (\Hess f(x^k)\inv - M(x^k)\inv \right ) \Hess f(x^k)\eta}_{H^k},
		\end{aligned}
		\label{eq:Minv_H_bound_1}
	\end{equation}
	where $\norm{\varepsilon}_{H^k} = \sqrt{\inner{\varepsilon}{\Hess f(x^k) [\varepsilon]}}$ denotes the norm induced by the Riemannian Hessian.
	Since $M(x^k) \approx_\epsilon \Hess f(x^k)$ by Assumption~\ref{as:linear_convergence:spectral_approximation}, 
	we can use \cref{lem:apx_laplacian_solution_guarantee} to obtain an upper bound of \eqref{eq:Minv_H_bound_1},
	\begin{equation}
		\norm{\left (\Hess f(x^k)\inv - M(x^k)\inv \right ) \Hess f(x^k)\eta}_{H^k} \leq c(\epsilon) \norm{\eta}_{H^k}.
		\label{eq:Minv_H_bound_2}
	\end{equation}
	In addition, using Assumption~\ref{as:linear_convergence:hessian_bound}, it holds that, 
	\begin{equation}
		\sqrt{\mu_H} \norm{\varepsilon} \leq \norm{\varepsilon}_{H^k} \leq \sqrt{L_H} \norm{\varepsilon}, \; \forall \varepsilon \in T_{x^k} \Mcal.
		\label{eq:Minv_H_bound_3}
	\end{equation}
	Combining \eqref{eq:Minv_H_bound_1}-\eqref{eq:Minv_H_bound_3} yields,
	\begin{equation}
	\begin{aligned}
		\sqrt{\mu_H}	\norm{\left (I - M(x^k)\inv \Hess f(x^k) \right ) \eta} \leq \sqrt{L_H} c(\epsilon) \norm{\eta} \\
		\implies 
		\norm{\left (I - M(x^k)\inv \Hess f(x^k) \right )} \leq \sqrt{\kappa_H}c(\epsilon).
	\end{aligned}
	\label{eq:Minv_H_bound_4}
	\end{equation}
	Combining \eqref{eq:Minv_H_bound_4} and \eqref{eq:Dvarphi_bound} in \eqref{eq:Minv_H_bound_0}, we conclude that,
	\begin{equation}
		B_2 \leq 2 \sqrt{\kappa_H}c(\epsilon) \norm{ \xhat^k - \xstarhat}.
		\label{eq:B2_bound}
	\end{equation}
	\paragraph{\underline{Finishing the proof}:}
	We conclude the proof by combining the upper bounds \eqref{eq:A_bound}, \eqref{eq:B1_bound}, and \eqref{eq:B2_bound} in \eqref{eq:linear_convergence_proof:triangle_inequality}, 
	and using the fact that $\norm{\xhat^k - \xstarhat} = \dist(x^k, \xstar)$,
	\begin{equation}
		\dist(x^{k+1}, \xstar) \leq 2 \sqrt{\kappa_H}c(\epsilon) \dist(x^k, \xstar) + (C_2 + 2\beta C_3) \dist(x^k, \xstar)^2.
	\end{equation}
	The linear convergence factor in \eqref{eq:linear_convergence:factor} is obtained by noting that the second term on the right-hand side vanishes at a \emph{quadratic rate}.
\end{proof}

\subsection{Proof of Theorem~\ref{thm:rotation_averaging_convergence}}
\label{sec:linear_convergence_RA}

\begin{proof}
We will use the general linear convergence result established in \cref{thm:linear_convergence}.
However, in order to properly account for the gauge symmetry in rotation averaging, we need to invoke \cref{thm:linear_convergence} on the quotient manifold that underlies our optimization problem.
In the following, we break the proof into three main parts (highlighted in bold).
The proof makes heavy use of results regarding Riemannian quotient manifolds.
The reader is referred to \cite[Chapter~9]{Boumal20Book} for a comprehensive review.

\paragraph{Rotation Averaging and Optimization on Quotient Manifolds.}
Following standard references, we denote a Riemannian quotient manifold as $\Mcal = \Mcalbar / \sim$.
For rotation averaging (\cref{prob:rotation_averaging}), the \emph{total space} is given by $\Mcalbar = \SOd(d)^n$. 
Let $R = \{R_1, \hdots, R_n\}$ and $R' = \{R'_1, \hdots, R'_n\}$ be two points on $\Mcalbar$. 
We say that $R$ and $R'$ are equivalent if they are related via a left group action:
\begin{equation}
	R \sim R' \iff \exists S \in \SOd(d), \; SR_i = R'_i, \forall i = 1, \hdots, n.
\end{equation}
The equivalence class represented by $R$ is defined as,
\begin{equation}
	[R] = \{(SR_1, \hdots, SR_n), \; S \in \SOd(d)\}.
\end{equation}
Note that the cost function of \cref{prob:rotation_averaging} is invariant within an equivalence class.
In the following, we use $T_R \Mcalbar$ to denote the usual tangent space at $R \in \Mcalbar$, 
and $T_{[R]} \Mcal$ to denote the corresponding tangent space on the quotient manifold.

Given a retraction $\Retr$ on $\Mcal$ (see \cite[Chapter~9.6]{Boumal20Book}), 
we can execute the Riemannian Newton's method on $\Mcal$:
\begin{equation}
	[R^{k+1}] = \Retr_{[R^k]}(\xi^k),
	\label{eq:quotient_newton_retraction}
\end{equation}
Above, $\xi^k \in T_{[R^k]} \Mcal$ is the solution of the following linear equation,
\begin{equation}
	\Hess f([R^k])[\xi^k] = - \rgrad f([R^k]),
	\label{eq:quotient_newton_system}
\end{equation}
where $\rgrad f([R^k])$ and $\Hess f([R^k])$ denote the Riemannian gradient and Hessian on the quotient manifold, respectively.

\paragraph{Expressing the iterates of \cref{alg:collaborative_rotation_averaging} on the quotient manifold.}
Recall that \cref{alg:collaborative_rotation_averaging} generates a sequence of iterates $\{R^k\}$ on the total space $\Mcalbar$. 
To prove convergence, we need to analyze the corresponding sequence of equivalence classes $\{[R^k]\}$ on the quotient manifold $\Mcal$.
Once we understand how $[R^k]$ evolves, we can prove the desired result by invoking \cref{thm:linear_convergence} on the quotient manifold $\Mcal$.
To begin with, we write the update step in \cref{alg:collaborative_rotation_averaging} in the following general form:
\begin{equation}
R^{k+1} = \Retrbar_{R^k}(v^k),
\label{eq:retraction_total_space}
\end{equation}
where $v^k$ is the vector corresponding to the matrix $V^k$ in line~\ref{alg:collaborative_rotation_averaging:solve}; see \eqref{eq:rotation_averaging_matrix_variables} for how $v^k$ and $V^k$ are related.
Together, $\Retrbar_{R^k}(v^k)$ is the concise notation for the update steps in line~\ref{alg:collaborative_rotation_averaging:retraction_start}-\ref{alg:collaborative_rotation_averaging:retraction_end},
where we update each rotation $R_i^k$ to $\Exp(v_i^k)R_i^k$.
Our notation $\Retrbar$ serves to emphasize that the retraction is performed on the total space $\Mcalbar$.
Recall from \eqref{eq:rotation_laplacian_system}-\eqref{eq:rotation_laplacian_system_matrix_form} and \cref{thm:approximation_guarantees} that $v^k$ is a solution to the linear system, 
\begin{equation}
(\Ltilde \otimes I_p) v^k = -\gbar(R^k),
\label{eq:apx_system_total_space}
\end{equation}
where $\Ltilde \approx_\epsilon L \equiv L(G; w)$ is defined in \eqref{thm:approximation_guarantees:spectrum}
and $\gbar(R^k) \in T_R \Mcalbar$ is the Riemannian gradient in the total space defined in \eqref{eq:gradient_and_hessian_def}. 
In this proof, we will use the notations $\Ncal_R$ and $\Hcal_R$ to denote the vertical and horizontal spaces at $R \in \Mcalbar$.\footnote{
For rotation averaging, it turns out that definitions of vertical and horizontal spaces do not depend on the point $R$; \eg, see \eqref{eq:vertical_space_RA} for the definition of the vertical space.
However, in this proof we will still use the more general notations $\Ncal_R$ and $\Hcal_R$, which help us to emphasize that $\Ncal_R$ and $\Hcal_R$ are subspaces of the tangent space $T_R \Mcalbar$.
}
By assumption, $v^k$ is the unique solution to \eqref{eq:apx_system_total_space} that satisfies $v^k \perp \Ncal_{R^k}$, \ie, $v^k \in \Hcal_{R^k}$.
This implies that there is a unique tangent vector $\eta^k \in T_{[R^k]} \Mcal$ on the tangent space of the quotient manifold
such that $v^k$ is the \emph{horizontal lift} \cite[Definition~9.25]{Boumal20Book} of $\eta^k$ at $R^k$:
\begin{equation}
	v^k = \Lift_{R^k}(\eta^k).
	\label{eq:horizontal_lift}
\end{equation}
At any $R \in \Mcalbar$, 
define $\Mbar(R): \Hcal_{R}  \to \Hcal_{R} $ to be the linear map,
\begin{equation}
	\Mbar(R): v \mapsto (\Ltilde \otimes I_p) v,
	\label{eq:M_map_total_space}
\end{equation}
where $v \in \Hcal_{R}$ is any vector from the horizontal space.
Since $\ker(\Ltilde \otimes I_p) = \Ncal_{R}$, it holds that
$\Mbar(R)$ is invertible on $\Hcal_{R}$.
Define $M([R]): T_{[R]} \Mcal \to T_{[R]} \Mcal $ 
to be the corresponding linear map on the quotient manifold,
\begin{equation}
	M([R]) =  \Lift_{R}\inv \circ \Mbar(R) \circ \Lift_{R}.
	\label{eq:M_map_quotient_space_def}
\end{equation}
Note that $M([R])$ is indeed linear because when considered as a mapping $\Lift_{R}: T_{[R]} \Mcal \to \Hcal_{R}$,
the horizontal lift is linear and invertible (see \cite[Definition~9.25]{Boumal20Book}).
Furthermore, applying $\Lift_{R}$ from the left on both sides of \eqref{eq:M_map_quotient_space_def} shows that for any
tangent vector $\eta \in T_{[R]} \Mcal$,
\begin{equation}
	\Lift_{R} \left ( 
	M([R])[\eta]
	\right)
	=
	\Mbar(R) [\Lift_{R}(\eta)].
	\label{eq:M_map_quotient_space}
\end{equation}
By \cite[Proposition~9.39]{Boumal20Book}, at iteration $k$, the Riemannian gradients on the total space and quotient space are related via,
\begin{equation}
	\gbar(R^k) = \Lift_{R^k}(\rgrad f([R^k])).
	\label{eq:gradient_quotient_vs_total}
\end{equation}
Combining \eqref{eq:horizontal_lift}-\eqref{eq:gradient_quotient_vs_total},  
we see that \eqref{eq:apx_system_total_space} is equivalent to, 
\begin{equation}
	\Lift_{R^k} \left ( 
	M([R^k])[\eta^k]
	\right)
	=
	-\Lift_{R^k}(\rgrad f([R^k])). 
	\label{eq:apx_system_total_space_lift_form}
\end{equation}
Applying $\Lift_{R^k}\inv$ to both sides of \eqref{eq:apx_system_total_space_lift_form}, 
\begin{equation}
	M([R^k])[\eta^k] = -\rgrad f([R^k]).
	\label{eq:apx_system_quotient_space}
\end{equation}
In \cite[Chapter~9.6]{Boumal20Book}, it is shown that
$\Retr_{[R]}(\eta) = [\Retrbar_{R}(\Lift_{R}(\eta))]$.
Using this result, we see that the update equation on the total space \eqref{eq:retraction_total_space}
can be converted to the following update equation, defined on the quotient space:
\begin{equation}
	[R^{k+1}] = [\Retrbar_{R^k}(v^k)] = \Retr_{[R^k]} (\eta^k).
	\label{eq:retraction_quotient_space}
\end{equation}
In summary, let $\{R_k\}$ denotes the iterates generated by \cref{alg:collaborative_rotation_averaging} on the total space $\Mcalbar$.
We have shown that $\{R_k\}$ corresponds to a sequence $\{[R^k]\}$ on the quotient space $\Mcal$ that evolves according to
\eqref{eq:apx_system_quotient_space}-\eqref{eq:retraction_quotient_space}.

\paragraph{Invoking \cref{thm:linear_convergence} on the quotient manifold.}
To finish the proof, we will invoke \cref{thm:linear_convergence} to show that the sequence of iterates $[R^k]$ generated by \eqref{eq:apx_system_quotient_space}-\eqref{eq:retraction_quotient_space} converges linearly to $[\Rstar]$.
This amounts to verifying that each condition in \cref{as:linear_convergence} holds on the quotient manifold $\Mcal$.
To start, note that there exists $r' > 0$ such that for any $R \in \SOd(d)^n$,
the condition $\dist([R], [\Rstar]) < r'$ implies that $R \in B_r(\Rstar)$ for some global minimizer $\Rstar$ in the total space, where $B_r(\Rstar)$ is the neighborhood within which
\cref{thm:Hessian_approximation} and \cref{cor:Hessian_bound} hold.

\textbf{Verifying \ref{as:linear_convergence:hessian_bound}.} 
We need to derive lower and upper bounds for the Hessian of the quotient optimization problem $\Hess f([R])$.
For any $\eta \in T_{[R]} \Mcal$, let $v = \Lift_{R}(\eta)$.
By \cite[Proposition~9.45]{Boumal20Book}, we have,
\begin{equation}
	\inner{\eta}{\Hess f([R]) \eta} = \inner{v}{H(R)v},
	\label{eq:Hessian_quotient_vs_total}
\end{equation}
where $H(R)$ is defined in \eqref{eq:Newton_step_quotient}.
Since $v \in \Hcal_{R} \equiv \Hcal$ belongs to the horizontal space,
we conclude using \cref{cor:Hessian_bound} that 
$\inner{v}{H(R)v} \geq \mu_H \norm{v}^2$ where $\mu_H$ is the constant defined in \cref{cor:Hessian_bound}.
Furthermore, since the quotient manifold $\Mcal$ inherits the Riemannian metric from the total space $\Mcalbar$, 
we have $\norm{\eta} = \norm{v}$.
We thus conclude that,
\begin{equation}
	\inner{\eta}{\Hess f([R]) \eta} = \inner{v}{H(R)v} \geq \mu_H \norm{v}^2 = \mu_H \norm{\eta}^2.
\end{equation}
Similarly, we can show that $\inner{\eta}{\Hess f([R]) \eta} \leq L_H \norm{\eta}^2$ where $L_H$ is also defined in \cref{cor:Hessian_bound}.
Therefore,
\begin{equation}
	\mu_H I \preceq \Hess f([R]) \preceq L_H I.
\end{equation}

\textbf{Verifying \ref{as:linear_convergence:M_bound}.}
We need to show that $M([R])$ defined in \eqref{eq:M_map_quotient_space_def} is invertible and 
the operator norm of its inverse can be upper bounded.
The invertibility follows from \eqref{eq:M_map_quotient_space_def} and the fact that both $\Lift_{R}$ and $\Mbar(R)$ are invertible on the horizontal space.
To upper bound $M([R])\inv$, it is equivalent to derive a lower bound on $M([R])$.
Let $\eta \in T_{[R]} \Mcal$ and $v = \Lift_{R} (\eta)$. 
We have,
\begin{equation}
	\inner{\eta}{M([R]) \eta} = \inner{v}{\Mbar(R) v} = v^\top (\Ltilde \otimes I_p) v \geq \lambda_2(\Ltilde) \norm{v}^2.
\end{equation}
The last inequality holds because $v \perp \Ncal$.
Thus, we conclude that, 
\begin{equation}
	\norm{M([R])\inv} \leq 1/{\lambda_2(\Ltilde)}.
\end{equation}

\textbf{Verifying \ref{as:linear_convergence:spectral_approximation}.}
Lastly, we need to show that the linear map $M([R])$ is a spectral approximation of the Riemannian Hessian $\Hess f([R])$ on the quotient manifold.
From \cref{thm:Hessian_approximation}, it holds that,
\begin{equation}
	H(R) \approx_\delta L \otimes I_p.
\end{equation}
In addition, from \cref{thm:approximation_guarantees}, we have
\begin{equation}
	L \approx_\epsilon \Ltilde \; \implies \; L \otimes I_p \approx_\epsilon \Ltilde \otimes I_p.
\end{equation}
Composing the two approximations yields,
\begin{equation}
	H(R) \approx_{\delta + \epsilon} \Ltilde \otimes I_p.
	\label{eq:Ltilde_approximate_H}
\end{equation}
Note that the above result directly implies the following approximation relation on the quotient manifold,
\begin{equation}
	\Hess f([R]) \approx_{\delta + \epsilon} M([R]).
\end{equation}
To see this, note that for any $\eta \in T_{[R]} \Mcal$ and $v = \Lift_{R}(\eta)$,
\begin{equation}
\begin{aligned}
\inner{\eta}{\Hess f([R]) \eta} 
&= \inner{v}{H(R) v} \\
&\leq e^{\delta + \epsilon} \inner{v}{\Mbar(R) v} \\
&= e^{\delta + \epsilon} \inner{\eta}{M([R]) \eta}.
\end{aligned}
\end{equation}
The same argument leads to,
\begin{equation}
	\inner{\eta}{\Hess f([R]) \eta} \geq e^{-\delta - \epsilon} \inner{\eta}{M([R]) \eta}.
\end{equation}

\end{proof}

\subsection{Proof of Theorem~\ref{thm:translation_recovery_convergence}}
\label{sec:linear_convergence_translation}

\begin{proof}
	We prove this result using induction.
	The base case of $k=1$ (first iteration) is true by \cref{thm:approximation_guarantees}.
	Now suppose \eqref{eq:translation_recovery_convergence_rate} holds at iteration $k \geq 1$.
	Define $D^\star = M_t^\star - M_t^k$.
	By \cref{thm:approximation_guarantees}, the approximate refinement $D^k$ computed at line~\ref{alg:collaborative_translation_recovery:solve} of \cref{alg:collaborative_translation_recovery} satisfies,
	\begin{equation}
		\begin{aligned}
			\norm{D^k - D^\star}_L \leq c(\epsilon) \norm{D^\star}_L 
			\implies &\norm{D^k + M_t^k - M_t^\star}_L \leq c(\epsilon) \norm{M_t^k - M_t^\star}_L \\
			\implies &\norm{M_t^{k+1} - M_t^\star}_L \leq c(\epsilon)^{k+1} \norm{M_t^\star}_L.
		\end{aligned}
	\end{equation}
	The second step above holds by the inductive hypothesis.
\end{proof}

\section{Auxiliary Lemmas}

\begin{lemma}
	Let $L, \Ltilde \in \PSD^n$ such that
	$L \approx_\epsilon \Ltilde$.
	Let $B \in \Real^{n \times p}$ be a matrix where each column of $B$ lives in the image of $L$.
	Let $\Xstar, \Xtilde \in \Real^{n \times p}$ be matrices such that
	$L \Xstar = B$ and $\Ltilde \Xtilde = B$. Then,
	\begin{equation}
	\norm{\Xstar - \Xtilde}_L \leq c(\epsilon) \norm{\Xstar}_L,
	\end{equation} 
	where $c(\epsilon) = \sqrt{1 + e^{2\epsilon} - 2e^{-\epsilon}}$.
	\label{lem:apx_laplacian_solution_guarantee}
\end{lemma}
\begin{proof}
	We note that the proof of a similar result can be found at \cite[Claim~2.4]{kyng2016sparsified}.
	In the following, we provide the proof for the case where $L$ and $\Ltilde$ are singular.
	The non-singular case can be proved in the same way by replacing matrix psuedoinverse with the inverse.
	Observe that
	\begin{equation}
	\norm{\Xstar - \Xtilde}_L^2 = \sum_{i=1}^p \norm{\Xstar_{[:,i]} - \Xtilde_{[:,i]}}_L^2,
	\label{eq:matrix_squared_norm_as_a_sum}
	\end{equation}
	where $\Xstar_{[:,i]}$ denotes the $i$-th column of $\Xstar$.
	Therefore, we can first obtain an upper bound for the squared norm on a single column. 
	To simplify notation, let $\xstar$ be a column of $\Xstar$, 
	and let $\xtilde$ and $b$ be the corresponding columns of $\Xtilde$ and $B$.
	Let us expand the squared norm,
	\begin{equation}
	\norm{\xstar - \xtilde}^2_L = 
	{\xstar}^\top L \xstar 
	- 2 {\xstar}^\top L \xtilde
	+ \xtilde^\top L \xtilde.
	\label{eq:laplacian_solution_squared_distance}
	\end{equation}
	Note that since $L \approx_\epsilon \Ltilde$, we have $\ker(L) = \ker(\Ltilde)$.
	In addition, any $\xtilde$ where $\Ltilde \xtilde = b$ can be written as 
	$\xtilde = \Ltilde\pinv b + \xtilde_\perp$ for some $\xtilde_\perp \in \ker(L)$.
	Now, let us consider the middle term in \eqref{eq:laplacian_solution_squared_distance},
	\begin{equation}
	{\xstar}^\top L \xtilde 
	= {\xstar}^\top L \Ltilde\pinv b
	= {\xstar}^\top L \Ltilde\pinv L \xstar
	= (L^{1/2} \xstar)^\top
	(L^{1/2} \Ltilde\pinv L^{1/2})
	(L^{1/2} \xstar).
	\end{equation}
	The relation 
	$\Ltilde \approx_\epsilon L$ implies 
	$\Ltilde\pinv \approx_\epsilon L\pinv$,
	which is equivalent to,
	\begin{equation}
	e^{-\epsilon} \Pi \preceq L^{1/2} \Ltilde\pinv L^{1/2} \preceq e^{\epsilon} \Pi,
	\label{eq:transformed_spectral_approximation}
	\end{equation}
	where $\Pi$ denotes the orthogonal projection onto $\image(L^{1/2} \Ltilde\pinv L^{1/2}) = \image(L)$.
	By construction, it holds that $L^{1/2}\xstar \in \image(L)$.
	Therefore,
	\begin{equation}
	{\xstar}^\top L \xtilde 
	= (L^{1/2} \xstar)^\top
	(L^{1/2} \Ltilde\pinv L^{1/2})
	(L^{1/2} \xstar)
	\geq e^{-\epsilon} \norm{\xstar}^2_L.
	\label{eq:laplacian_solution_distance_bound_1}
	\end{equation}
	Next, expand the last term in \eqref{eq:laplacian_solution_squared_distance}:
	\begin{equation}
	\begin{aligned}
	\xtilde^\top L \xtilde
	&= b^\top \Ltilde\pinv L \Ltilde\pinv b
	= {\xstar}^\top L \Ltilde\pinv L \Ltilde\pinv L \xstar \\
	&= (L^{1/2} \xstar)^\top 
	(L^{1/2} \Ltilde\pinv L^{1/2})
	(L^{1/2} \Ltilde\pinv L^{1/2})
	(L^{1/2} \xstar) \\
	&= \norm{(L^{1/2} \Ltilde\pinv L^{1/2})(L^{1/2} \xstar)}_2^2.
	\end{aligned}
	\end{equation}
	Using \eqref{eq:transformed_spectral_approximation}, we conclude that,
	\begin{equation}
	\xtilde^\top L \xtilde \leq e^{2\epsilon} \norm{\xstar}^2_L.
	\label{eq:laplacian_solution_distance_bound_2}
	\end{equation}
	Combining \eqref{eq:laplacian_solution_distance_bound_1} and \eqref{eq:laplacian_solution_distance_bound_2} in \eqref{eq:laplacian_solution_squared_distance} yields,
	\begin{equation}
	\norm{\xstar - \xtilde}^2_L 
	\leq (1 + e^{2\epsilon} - 2e^{-\epsilon}) \norm{\xstar}^2_L
	= c(\epsilon)^2 \norm{\xstar}^2_L.
	\end{equation}
	Finally, using this upper bound on \eqref{eq:matrix_squared_norm_as_a_sum} yields the desired result,
	\begin{equation}
	\begin{aligned}
	\norm{\Xstar - \Xtilde}_L^2 
	&= \sum_{i=1}^p \norm{\Xstar_{[:,i]} - \Xtilde_{[:,i]}}_L^2 \\
	&\leq \sum_{i=1}^p c(\epsilon)^2 \norm{\Xstar_{[:,i]}}_L^2 \\
	&= c(\epsilon)^2 \norm{\Xstar}_L^2.
	\end{aligned}
	\end{equation}
\end{proof}
\section{Experiment Details and Additional Results}
\label{sec:additional_experiments}

\begin{table}[H]
	\centering
	\renewcommand{\arraystretch}{1.1}
	\caption{Rotation averaging on benchmark SLAM datasets with 5 robots under the \textbf{squared geodesic distance cost function}.
		$|\Vcal|$ and $|\Ecal|$ denote the total number of rotation variables and measurements, respectively.
		We run the baseline \method{Newton} method and the proposed method (\cref{alg:collaborative_rotation_averaging}) with sparsification parameter $\epsilon=1.5$,
		and compare the number of iterations, uploads, and downloads to reach a Riemannian gradient norm of $10^{-5}$.
		{For the proposed method, we also show the sparsity achieved by sparsification (lower is better).}
		Results averaged across 5 runs.}
	\label{tab:rotation_averaging_geodesic}
	\resizebox{\textwidth}{!}{%
		\begin{tabular}{|l|r|r|rr|rr|rr|r|}
			\hline
			\multicolumn{1}{|c|}{\multirow{2}{*}{Datasets}} &
			\multicolumn{1}{c|}{\multirow{2}{*}{$|\Vcal|$}} &
			\multicolumn{1}{c|}{\multirow{2}{*}{$|\Ecal|$}} &
			\multicolumn{2}{c|}{Iterations} &
			\multicolumn{2}{c|}{Uploads (kB)} &
			\multicolumn{2}{c|}{Downloads (kB)} &
			\multicolumn{1}{c|}{\multirow{2}{*}{\begin{tabular}[c]{@{}c@{}}Achieved\\ sparsity (\%)\end{tabular}}} \\ \cline{4-9}
			\multicolumn{1}{|c|}{} &
			\multicolumn{1}{c|}{} &
			\multicolumn{1}{c|}{} &
			\multicolumn{1}{l|}{Newton} &
			\multicolumn{1}{l|}{Proposed} &
			\multicolumn{1}{l|}{Newton} &
			\multicolumn{1}{l|}{Proposed} &
			\multicolumn{1}{l|}{Newton} &
			\multicolumn{1}{l|}{Proposed} &
			\multicolumn{1}{c|}{} \\ \hline
			Killian Court (2D) & 808   & 827   & \multicolumn{1}{r|}{1} & 1   & \multicolumn{1}{r|}{0.8}      & 0.5   & \multicolumn{1}{r|}{0.3}   & 0.3   & 100  \\ \hline
			CSAIL (2D)         & 1045  & 1171  & \multicolumn{1}{r|}{1} & 4   & \multicolumn{1}{r|}{3.6}      & 5.9   & \multicolumn{1}{r|}{1.2}   & 4.6   & 97.3 \\ \hline
			INTEL (2D)         & 1228  & 1483  & \multicolumn{1}{r|}{1} & 3.4 & \multicolumn{1}{r|}{3.5}      & 5     & \multicolumn{1}{r|}{1.1}   & 3.7   & 96.1 \\ \hline
			Manhattan (2D)     & 3500  & 5453  & \multicolumn{1}{r|}{1} & 4.8 & \multicolumn{1}{r|}{59.5}     & 48.3  & \multicolumn{1}{r|}{6.3}   & 30.1  & 38.8 \\ \hline
			KITTI 00 (2D)      & 4541  & 4676  & \multicolumn{1}{r|}{1} & 1   & \multicolumn{1}{r|}{6.6}      & 4.4   & \multicolumn{1}{r|}{2.2}   & 2.2   & 100  \\ \hline
			City (2D)          & 10000 & 20687 & \multicolumn{1}{r|}{1} & 4   & \multicolumn{1}{r|}{225.1}    & 351.5 & \multicolumn{1}{r|}{64.5}  & 258.1 & 97.2 \\ \hline
			Garage (3D)        & 1661  & 6275  & \multicolumn{1}{r|}{1} & 2   & \multicolumn{1}{r|}{274.4}    & 88.9  & \multicolumn{1}{r|}{35.8}  & 71.6  & 93.4 \\ \hline
			Sphere (3D)        & 2500  & 4949  & \multicolumn{1}{r|}{2} & 8.4 & \multicolumn{1}{r|}{2548.8}   & 104   & \multicolumn{1}{r|}{19.2}  & 80.6  & 16.9 \\ \hline
			Torus (3D)         & 5000  & 9048  & \multicolumn{1}{r|}{3} & 9   & \multicolumn{1}{r|}{10423.7}  & 218.1 & \multicolumn{1}{r|}{57}    & 170.9 & 12.4 \\ \hline
			Grid (3D)          & 8000  & 22236 & \multicolumn{1}{r|}{3} & 8.8 & \multicolumn{1}{r|}{206871.6} & 857.9 & \multicolumn{1}{r|}{220.8} & 647.5 & 2.8  \\ \hline
			Cubicle (3D)       & 5750  & 16869 & \multicolumn{1}{r|}{2} & 6.2 & \multicolumn{1}{r|}{7015}     & 407.9 & \multicolumn{1}{r|}{107.7} & 333.9 & 19.9 \\ \hline
			Rim (3D)           & 10195 & 29743 & \multicolumn{1}{r|}{2} & 9   & \multicolumn{1}{r|}{26828.9}  & 568.4 & \multicolumn{1}{r|}{104.5} & 470.4 & 6.6  \\ \hline
		\end{tabular}%
	}
\end{table}

\begin{figure}[H]
	\centering
	\begin{subfigure}[t]{0.33\linewidth}
		\includegraphics[width=\textwidth, trim=0 0 0 0, clip]{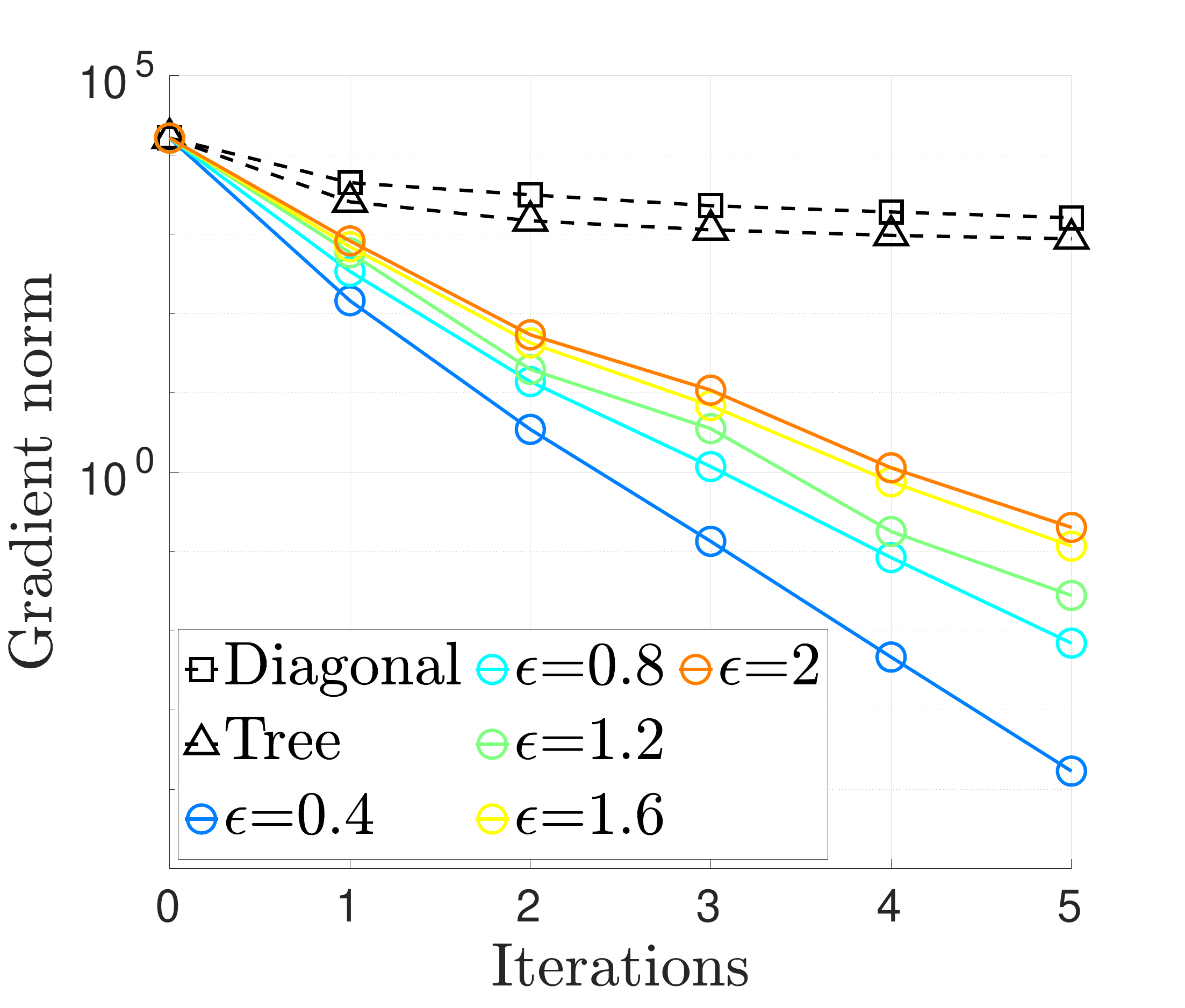}
		\caption{Gradient norm vs. iterations}
		\label{fig:translation_eval:iteration}
	\end{subfigure}
	\hspace{2cm}
	\begin{subfigure}[t]{0.33\linewidth}
		\includegraphics[width=\textwidth, trim=0 0 0 0, clip]{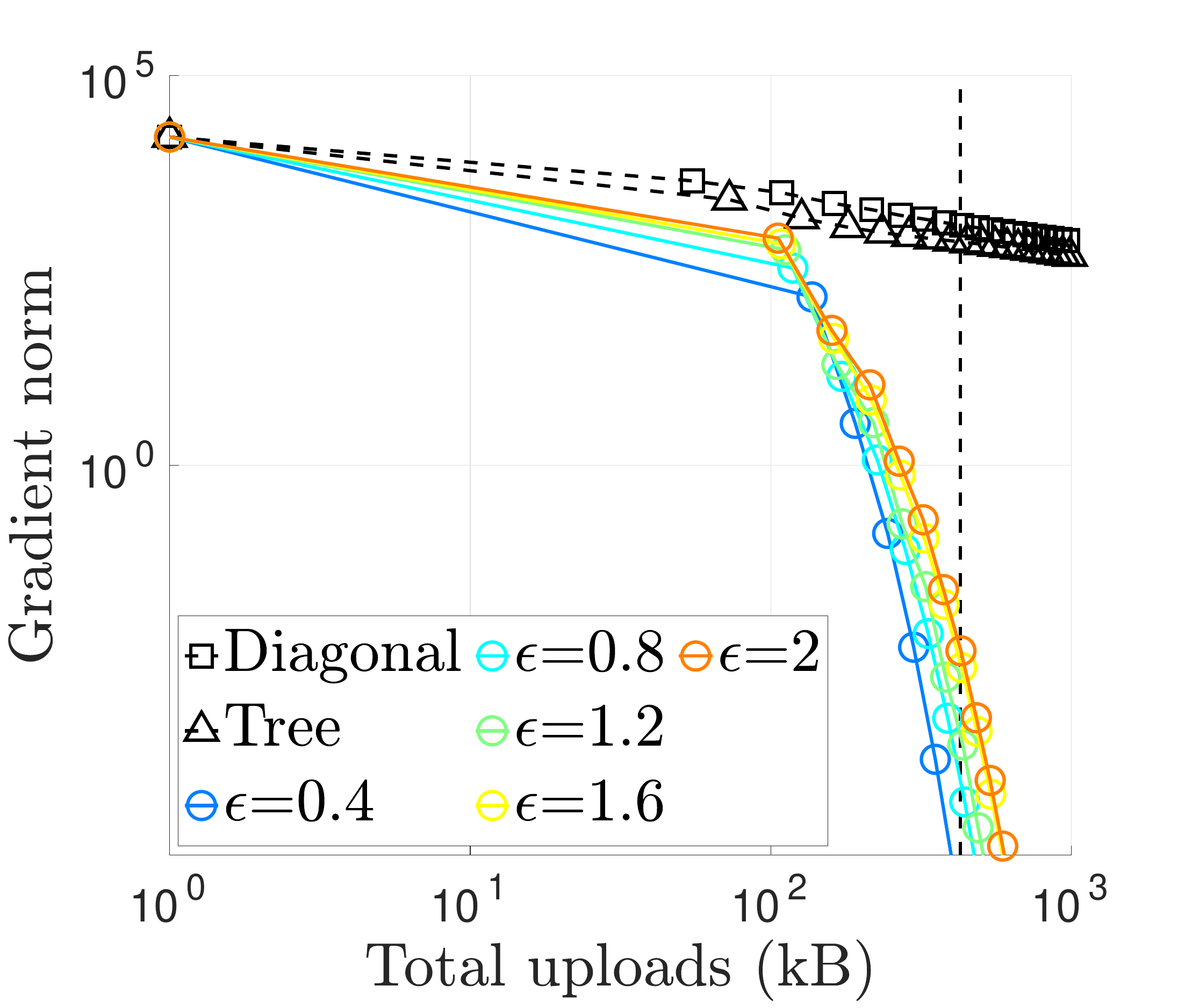}
		\caption{Gradient norm vs. uploads}
		\label{fig:translation_eval:communication}
	\end{subfigure}	
	\caption{
		Evaluation of \cref{alg:collaborative_translation_recovery} on the 5-robot translation estimation problem from the \dataset{Cubicle} dataset.
		(a) Evolution of gradient norm as a function of iterations.
		(b) Evolution of gradient norm as a function of total uploads to the server. 
	} 
	\label{fig:translation_eval}
\end{figure}

\begin{figure}[H]
	\centering
	\begin{subfigure}[t]{0.33\linewidth}
		\includegraphics[width=\textwidth, trim=0 0 0 0, clip]{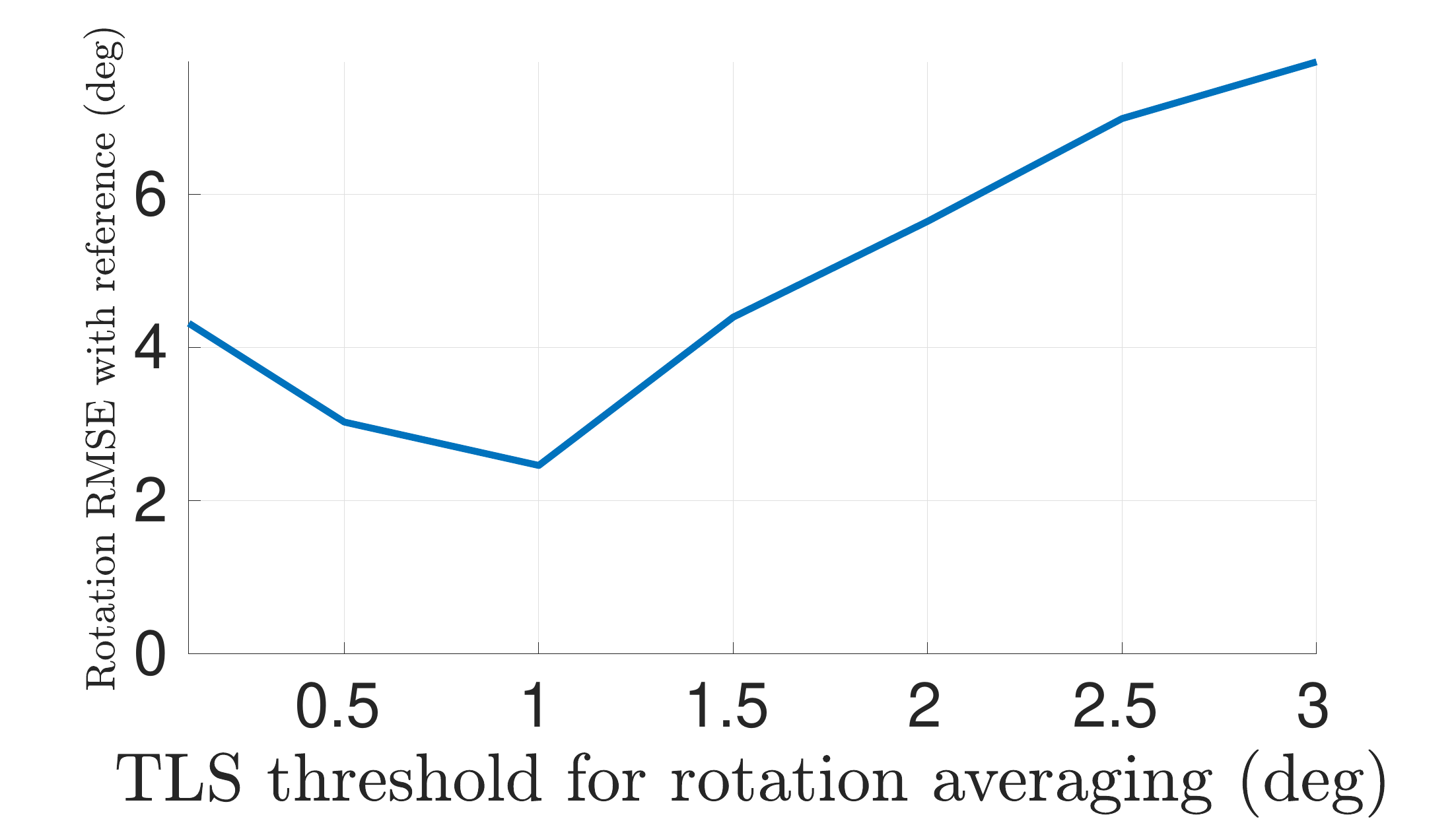}
		\caption{{Rotation estimation}}
		\label{fig:jackal_dataset_sensitivity:rotation}
	\end{subfigure}
	\hspace{2cm}
	\begin{subfigure}[t]{0.33\linewidth}
		\includegraphics[width=\textwidth, trim=0 0 0 0, clip]{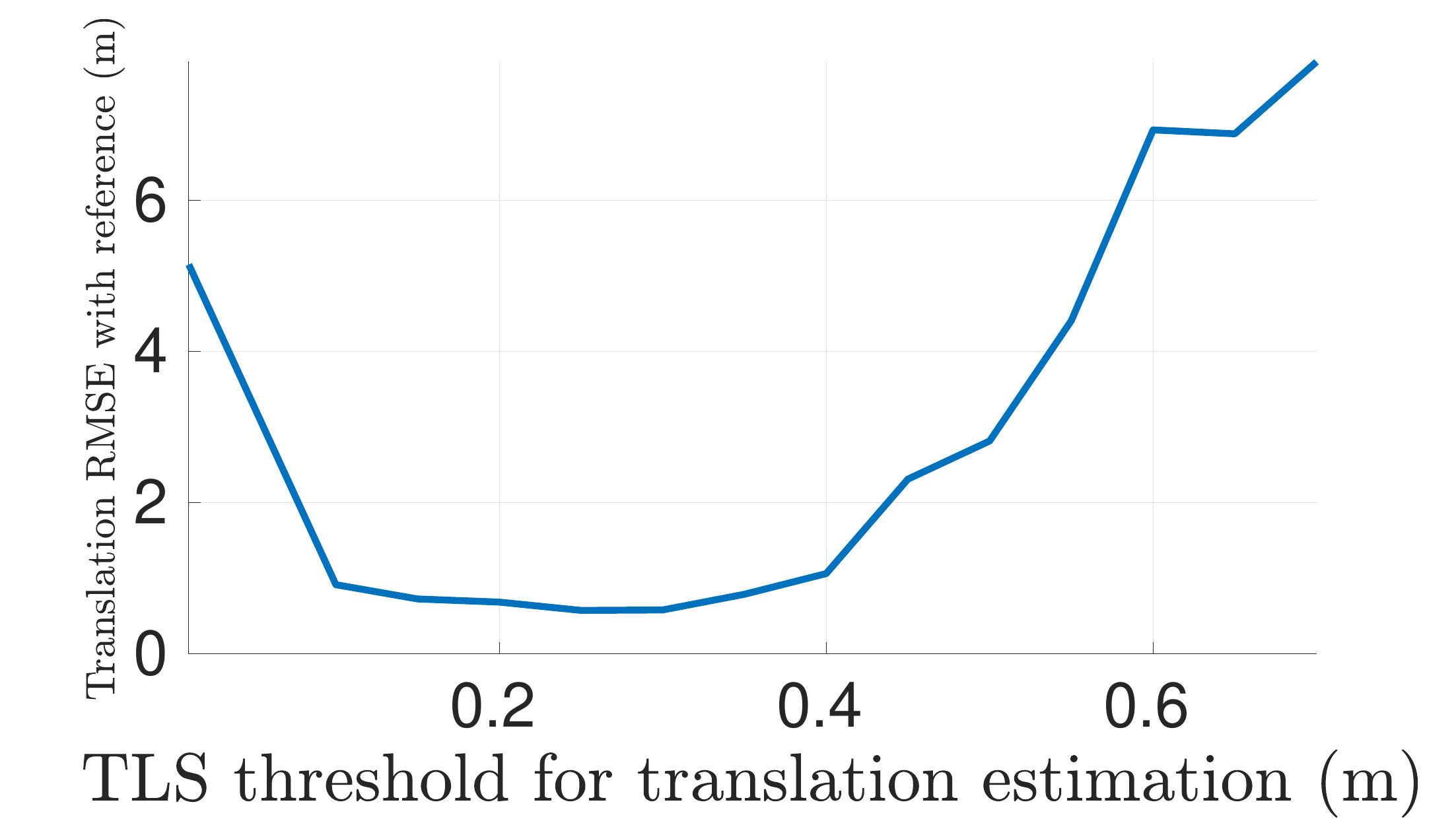}
		\caption{{Translation} estimation}
		\label{fig:jackal_dataset_sensitivity:translation}
	\end{subfigure}	
	\caption{
		Sensitivity of robust PGO initialization to TLS thresholds.
	} 
	\label{fig:jackal_dataset_sensitivity}
\end{figure}

\edit{
\myParagraph{Initialization of \method{RBCD++}}
In \cref{tab:pgo_initialization}, we initialize \method{RBCD++} using a procedure that requires a single round of communication between the server and all robots, and furthermore is expected to be more accurate than the spanning tree initialization.
In the following, let $\alpha \in \Natural$ denote the robot index.
For each pose variable $i \in \Vcal$ in the multi-robot pose graph, let $\alpha(i)$ denote the robot that owns this pose.
At a high level, the initialization approach we implement for \method{RBCD++} first lets robots solve local PGO problems in parallel, 
and then lets the server align the local trajectory estimates by solving another PGO problem defined using all inter-robot loop closures.
We outline the steps in detail below:
\begin{enumerate}[leftmargin=1.0cm]
	\item In parallel, each robot solves a local PGO problem using \method{SE-Sync} \cite{Rosen19IJRR} by considering all local odometry and intra-robot loop closures.
	For each pose variable $i \in \Vcal$, this step returns an estimate in the corresponding robot's local frame, denoted as $\That^{\alpha(i)}_{i} \in \SE(d)$.
	
	\item The server jointly estimates the transformation of each robot $\alpha$ to the global frame, denoted as $T_\alpha = (R_\alpha, t_\alpha) \in \SE(d)$,
	by solving the following PGO problem using \method{SE-Sync},
	\begin{equation}
	\label{eq:inter_robot_pgo}
		\begin{aligned}
			\minimize_{
				\substack{
					R_\alpha \in \SOd(d), 
					t_\alpha \in \Real^d}} \;
			\sum_{(i,j) \in \Ecal_\text{inter}} {\kappa_{ij}} \norm{R_{\alpha(i)} \Rhat^{\alpha(i)}_{\alpha(j)} - R_{\alpha(j)}}^2_F 
			+ {\tau_{ij}} \norm{t_{\alpha(j)} - t_{\alpha(i)} - R_{\alpha(i)} \that^{\alpha(i)}_{\alpha(j)}}^2_2.
		\end{aligned}
	\end{equation}
	In \eqref{eq:inter_robot_pgo}, $\Ecal_\text{inter}$ denotes the set of inter-robot loop closures.
	Constants $\kappa_{ij}, \tau_{ij} > 0$ are the rotation and translation measurement weights associated with the inter-robot loop closure $(i,j) \in \Ecal_\text{inter}$.
	Each $\That^{\alpha(i)}_{\alpha(j)} \triangleq (\Rhat^{\alpha(i)}_{\alpha(j)}, \that^{\alpha(i)}_{\alpha(j)} ) \in \SE(d)$ is an estimate of the relative transformation between robots $\alpha(i)$ and $\alpha(j)$ where $\alpha(i) \neq \alpha(j)$, computed as,
	\begin{equation}
	\label{eq:inter_robot_transform}
		\That^{\alpha(i)}_{\alpha(j)} = \That^{\alpha(i)}_{i} \Ttilde_{ij} \left (\That^{\alpha(j)}_{j} \right)^{-1}.
	\end{equation}
	In \eqref{eq:inter_robot_transform}, 
	$\Ttilde_{ij} \in \SE(d)$ is the original inter-robot loop closure, and
	$\That^{\alpha(i)}_{i}, \That^{\alpha(j)}_{j} \in \SE(d)$
	are the pose estimates in robots' local frames computed in step 1. 

	\item For each pose variable $i \in \Vcal$ in the pose graph, its final initialization (input to \method{RBCD++}) is computed as,
	\begin{equation}
		\That_i = \That_{\alpha(i)}  \That^{\alpha(i)}_{i},
	\end{equation}
	where $\That_{\alpha(i)} \in \SE(d)$ is robot $\alpha(i)$'s coordinate frame estimated in step 2.
\end{enumerate}
}

\myParagraph{Impact of Spectral Sparsification on Translation Estimation}
In  \cref{fig:translation_eval}, we evaluate the impact of spectral sparsification
on \cref{alg:collaborative_translation_recovery} for collaborative translation estimation.
For conciseness, we only evaluate convergence as a function of iteration (\cref{fig:translation_eval:iteration}) and uploads (\cref{fig:translation_eval:communication}).
Using the \dataset{Cubicle} dataset, we generate the translation estimation problem by fixing the rotation estimates in PGO \eqref{eq:pose_graph_optimization} to the solution produced by \cref{alg:collaborative_rotation_averaging}; see our discussions in \cref{sec:problem_definition:applications}.
{Note that for translation estimation, setting $\epsilon = 0$ (\ie, no sparsification) in \cref{alg:collaborative_translation_recovery}
	corresponds to the exact Newton method and recovers the exact solution in a single iteration (since \cref{prob:translation_recovery} is a linear least squares problem).}
In \cref{fig:translation_eval:communication}, the vertical line shows the total uploads incurred with this setting.
We run \cref{alg:collaborative_translation_recovery} under varying sparsification parameter $\epsilon$.
Similar to the rotation averaging experiment, 
we introduce two baselines in which each robot heuristically sparsifies its Schur complement matrix $S_\alpha$
according to a diagonal sparsity pattern (\method{Diagonal}) or a tree sparsity pattern (\method{Tree}).
Our results in \cref{fig:translation_eval} show that the proposed method dominates the baselines in terms of both iteration and communication complexity.
Furthermore, spectral sparsification provides a way to trade off accuracy with communication efficiency.
For instance, to obtain an approximate solution, we can run \cref{alg:collaborative_translation_recovery} with $\epsilon = 0.4$ and terminate when the gradient norm reaches $10^{-1}$, which would incur less communication compared to recovering the exact solution with $\epsilon = 0$.

\myParagraph{Sensitivity of robust PGO initialization to TLS thresholds}
In \cref{fig:jackal_dataset_sensitivity}, we evaluate the sensitivity of our two-stage robust PGO initialization to the choice of TLS thresholds on the real-world collaborative SLAM dataset (\cref{sec:experiments:jackal_dataset}).
First, for rotation estimation,
\cref{fig:jackal_dataset_sensitivity:rotation} shows the final rotation RMSE achieved by GNC as a function of the TLS threshold.
We observe that setting the threshold between $0.5$~deg to $1.5$~deg produces the best performance.
For values too small, GNC starts to reject correct (\ie inlier) measurements, while for large values the final solution is negatively influenced by outliers.
\cref{fig:jackal_dataset_sensitivity:translation} repeats the analysis for the translation estimation stage.
While a similar trend is observed, the result shows that GNC is less sensitive to the TLS threshold in translation estimation.

}

\end{document}